\documentclass[twoside]{article}

\PassOptionsToPackage{square, numbers}{natbib}

\usepackage{adjustbox}
\usepackage{microtype}
\usepackage{graphicx}
\usepackage{subcaption}
\usepackage{booktabs}
\usepackage{multirow}
\usepackage{arydshln} 
\usepackage{hyperref}
\usepackage{enumitem}
\usepackage{xcolor}
\definecolor{entcolor}{rgb}{0.13, 0.55, 0.13}
\newcommand{\ent}[1]{\textcolor{entcolor}{#1}}

\usepackage{amsmath}
\usepackage{amssymb}
\usepackage{mathtools}
\usepackage{amsthm}
\usepackage{xspace}
\usepackage{float}
\usepackage{algorithm}
\usepackage{algorithmic}

\usepackage[capitalize,noabbrev]{cleveref}

\usepackage[preprint]{aistats2027}
\usepackage{notation}
\usepackage[round]{natbib}

\usepackage[utf8]{inputenc} % allow utf-8 input
\usepackage[T1]{fontenc}    % use 8-bit T1 fonts
\usepackage{hyperref}       % hyperlinks
\usepackage{url}            % simple URL typesetting
\usepackage{booktabs}       % professional-quality tables
\usepackage{amsfonts}       % blackboard math symbols
\usepackage{nicefrac}       % compact symbols for 1/2, etc.
\usepackage{microtype}      % microtypography
\usepackage{xcolor}         % colors

\usepackage{aliascnt}
\newtheorem{lemma}{Lemma}
\newaliascnt{proposition}{lemma}
\newtheorem{proposition}[proposition]{Proposition}
\aliascntresetthe{proposition}
\crefname{proposition}{Proposition}{Propositions}

\usepackage[textsize=tiny]{todonotes}

\begin{document}



\twocolumn[

\aistatstitle{Proximal Policy Optimization for Amortized Discrete Sampling}

\aistatsauthor{
  Anna Zykova-Myzina$^{1,*}$
  \quad
  Timofei Gritsaev$^{2,*}$
  \quad
  Daniil Tiapkin$^{3,\dagger}$
  \quad
  Nikita Morozov$^{1,\dagger}$
}

\aistatsaddress{
  $^{1}$\,HSE University
  \qquad
  $^{2}$\,Constructor University
  \\
  $^{3}$\,CMAP, CNRS, \'Ecole polytechnique, IPP
} ]

\begingroup
\renewcommand{\thefootnote}{}
\footnotetext{%
  $^{*}$Equal contribution.
  $^{\dagger}$Equal senior authorship.
}
\endgroup
% \aistatsauthor{ Author 1 \And Author 2 \And  Author 3 }

% \aistatsaddress{ Institution 1 \And  Institution 2 \And Institution 3 } ]

\begin{abstract}
  This paper explores policy gradient algorithms for training stochastic policies to sample from structured discrete probability distributions under the Generative Flow Network (GFlowNet) framework. Building on extensive theoretical connections between GFlowNets and entropy-regularized reinforcement learning, we derive equivalents of standard policy gradient algorithms for training GFlowNets, as well as experimentally explore their various methodological aspects, including baseline training and advantage estimation. Most importantly, our work is the first to derive and successfully apply proximal policy optimization to GFlowNets, showing its improved convergence speed and data efficiency compared to standard GFlowNet training objectives on benchmarks ranging from synthetic energies to molecular-graph generation.
\end{abstract}

%\vspace{-0.1cm}
\section{Introduction}
%\vspace{-0.1cm}

In this paper, we consider a task of sampling from a probability distribution over a finite discrete space $\cX$ with probability mass function $p(x) = \cR(x) / \rmZ$, where $\cR\colon \cX \to \mathbb{R}_{>0}$ is a positive reward function and $\rmZ = \sum_{x \in \cX} \cR(x)$ is an unknown normalizing constant. The task arises in many domains, including statistical physics~\cite{potts1952some, wolff1989collective}, bayesian statistics~\cite{madigan1995bayesian, green1995reversible}, computational biology~\cite{yang1997bayesian, ding2003statistical} and natural language processing~\cite{griffiths2002probabilistic, miao2019cgmh}. Generative Flow Networks (GFlowNets, \cite{bengio2021flow}) are designed to solve this task when the space of interest $\cX$ has a compositional structure: any object $x$ can be sequentially constructed from "building blocks". This construction process induces a directed acyclic graph (DAG) environment, and GFlowNets train a stochastic policy to navigate it in order to match the distribution of interest $p(x)$ over the terminal states.

A line of work~\cite{tiapkin2024generative, deleu2024discrete} uncovered deep theoretical connections between GFlowNets and entropy-regularized reinforcement learning~\cite{haarnoja2017reinforcement, schulman2017equivalence}, showing that in essence, GFlowNet framework utilizes reinforcement learning to train the policy for sampling. These works proved that, if intermediate rewards are appropriately defined in the DAG environment, the optimal policy in the entropy-regularized RL task coincides with the sampling policy GFlowNets aim to recover. Moreover, it was shown that existing GFlowNet training algorithms like Trajectory Balance~\citep[TB]{malkin2022trajectory}, Detailed Balance~\citep[DB]{bengio2023gflownet} and Subtrajectory Balance~\citep[SubTB]{madan2023learning} can be viewed as variations of value-based RL algorithms SoftDQN~\cite{haarnoja2017reinforcement} and PCL~\cite{nachum2017bridging}. Most notably, this perspective showed that a large variety of existing RL algorithms can be directly applied to train policies for solving sampling tasks.

GFlowNet training has been dominated by value-based objectives~\citep{bengio2023gflownet,malkin2022trajectory}. Policy gradient methods such as Proximal Policy Optimization (PPO,~\citealt{schulman2017proximal}) remain comparatively underexplored for sampling, despite being the dominant approach to large-scale RL post-training of language models~\citep{ouyang2022training,guo2025deepseek,yu2025dapo} and diffusion models~\citep{black2023training,fan2023dpok}, where they have proven to scale to large models and difficult problems.

Beyond this empirical track record, policy gradient methods differ from value-based GFlowNet objectives fundamentally: rather than learning a globally consistent flow or partition function whose fixed point matches the target distribution, they iteratively improve the current policy using estimates of its own value. This distinction matters in practice. TB must jointly estimate the log-partition function $\log \rmZ$, and training is inefficient as long as this estimate is inaccurate~\citep{malkin2023gflownets}, which is a significant obstacle in large-scale problems where the $\log \rmZ$ estimation poses a significant challenge itself. VarGrad~\citep{richter2020vargrad} replaces the learned $\log \rmZ$ with an estimate computed from a sampled batch of trajectories. Although this avoids explicitly learning $\log \rmZ$, training performance still depends on the accuracy of this estimate.
DB learns a flow function at every state and suffers from credit assignment difficulties over long trajectories~\citep{malkin2022trajectory}.
Policy gradient methods rely less on such globally consistent objects: the objective they need is the expected return of the current policy, which can be estimated directly from rollouts. A learned value function enters only as a baseline, whose errors affect the variance of the gradient estimate and, for GAE with $\lambda < 1$, introduce a bias controlled by $\lambda$, while the objective itself does not depend on it.

Existing work on policy gradients for GFlowNets~\citep{pmlr-v235-niu24c} adopts a reward formulation in which the per-step reward depends on the current forward policy. This formulation has the same objective and on-policy gradient as the entropy-regularized one, but in multi-epoch updates such as PPO the reward changes together with the policy, and the standard practice of reusing advantages computed at the rollout policy silently drops an entropy correction (\cref{app:policy_dependent}). In contrast, our work studies policy gradient algorithms within the theoretical framework of~\cite{tiapkin2024generative}, starting from the soft policy improvement theorem \citep{haarnoja2017reinforcement,chan2022greedification} and culminating in proposing a variant of proximal policy optimization~\cite{schulman2017proximal} as an approximate version of conservative policy iteration~\citep{kakade2002approximately} for training GFlowNets, as well as exploring various methodological aspects. We highlight our main contributions:
\begin{enumerate}[leftmargin=*, itemsep=1pt,topsep=1pt]
    \item We derive standard policy gradient objectives for training discrete samplers with Entropy-Regularized RL under the theoretical framework of~\cite{tiapkin2024generative}. We experimentally study various design choices related to policy gradient training, including variance reduction with value baselines and advantage estimation~\cite{Schulman2015HighDimensionalCC}, showing their crucial role in stable and efficient training of GFlowNets. 
    \item We derive Ent-PPO, a variant of PPO adapted to the soft-RL formulation of GFlowNet sampling. The derivation makes explicit two ingredients that any principled application of PPO to GFlowNet sampling must include, both of which were absent from the original attempt of~\cite{bengio2021flow}: the intermediate per-step rewards $r(s_t, s_{t+1}) = \log \PB(s_t \mid s_{t+1})$ identified by~\cite{tiapkin2024generative}, which make the soft-RL objective equivalent to the sampling problem; and the entropy regularization required by that equivalence, 
    which in our derivation combines with a cross-entropy against the behaviour policy into an analytic KL term of unit weight.
    Without these corrections, PPO targets the wrong objective: with terminal rewards alone and no entropy term, its optimum is $\arg\max_x \cR(x)$ rather than the sampling target $\cR(x)/\rmZ$, consistent with the negative results reported by~\citet{bengio2021flow}. We show that Ent-PPO improves over both naive PPO and the established off-policy GFlowNet objectives~\citep{malkin2022trajectory, bengio2023gflownet, madan2023learning} across a range of benchmarks.   
\end{enumerate}

%\vspace{-0.15cm}
Source code: \href{https://github.com/tgritsaev/ent-ppo}{github.com/tgritsaev/ent-ppo}.

%\vspace{-0.15cm}
\section{Background}\label{sec:background}

\vspace{-2pt}

\subsection{Reinforcement Learning and Policy Gradients}\label{sec:rl_background}

\vspace{-1pt}

This subsection collects the RL background needed for the methodology: the Markov decision process, value functions, entropy-regularized RL, soft policy improvement, and the two policy gradient algorithms we adapt — VPG and PPO.

\textbf{Markov decision process.} 
The basic object of RL is a Markov decision process (MDP), specified by a tuple $\cM = (\cS, \cA, \MK, r, \gamma, \sinit)$, where $\cS$ and $\cA$ are finite state and action spaces, $\MK$ is a Markovian transition kernel, $r:\cS\times\cA\to\mathbb R$ is a bounded reward, $\gamma \in [0,1]$ is a discount factor, and $\sinit \in \cS$ is a 
fixed initial state. Each state $s$ has a feasible action set $\cA_s\subseteq\cA$. A (Markovian) policy assigns to each $s$ a distribution $\pi(\cdot\mid s)\in\Delta(\cA_s)$. In deterministic MDPs — including the GFlowNet setting of \cref{sec:gfn_background} — every action is identified with its successor state, in which case we write $r(s,s')$ and $\pi(s'\mid s)$ interchangeably. For episodic problems 
we adopt the convention of an absorbing terminal state with zero reward, so all returns considered below are finite even when $\gamma = 1$.

\textbf{Value functions.}
The performance of a policy $\pi$ is measured by the (soft) value
\begin{equation}\label{eq:value_def}
    V^\pi_{\ent\alpha}(s) \triangleq 
    \E_\pi\!\left[\sum_{t=0}^\infty \gamma^t \big(r(s_t,a_t) 
    \ent{{} + \alpha \cH(\pi(\cdot\mid s_t))} \big) \,\Big|\, 
    s_0 = s\right]\!,
\end{equation}
where $\alpha \geq 0$ controls entropy regularization 
\citep{neu2017unified, geist2019theory}: \ent{$\alpha > 0$ gives the entropy-regularized case}, $\alpha = 0$ recovers classical RL. 
The associated $Q$-function and advantage:
\begin{align}
    Q^\pi_{\ent\alpha}(s, a) &\triangleq r(s, a) + \gamma
    \E_{s' \sim \MK(s, a)}\!\left[V^\pi_{\ent\alpha}(s')\right]\!, \label{eq:qfunction}\\
    A^\pi_{\ent\alpha}(s, a) &\triangleq Q^\pi_{\ent\alpha}(s, a)
    - V^\pi_{\ent\alpha}(s) \ent{{} - \alpha \log \pi(a \mid s)}, \label{eq:advantage}
\end{align}
where the advantage definition follows~\citep[Definition 5]{chan2022greedification} and is normalized so that 
$\E_{a\sim\pi}[A^\pi_{\ent\alpha}(s,a)] = 0$. Using $\E_{a\sim\pi}
[-\log\pi(a\mid s_t)] = \cH(\pi(\cdot\mid s_t))$, the soft value also 
admits the sample-path form $V^\pi_\alpha(s) = \E_\pi\!\big[
\sum_t \gamma^t (r(s_t, a_t) - \alpha \log\pi(a_t\mid s_t)) \mid 
s_0 = s\big]$, used in the appendix derivations. We write 
$J_{\ent\alpha}(\pi) \triangleq V^\pi_{\ent\alpha}(\sinit)$. A 
\ent{(soft-)}optimal policy $\pistar_{\ent\alpha}$ maximizes 
$V^\pi_{\ent\alpha}(s)$ at every $s$.

\textbf{\ent{(Soft)} policy improvement.}
Given any $\pi_{\text{old}}$, the policy 
% \begin{equation}\label{eq:soft_pi_state}
%     \textstyle \pi_{\text{new}}(\cdot\mid s) \in 
%     {\arg\max}_{\pi(\cdot\mid s)\in\Delta(\cA_s)}
%     {\E}_{a\sim\pi(\cdot\mid s)}\!\left[Q^{\pi_{\text{old}}}_{\ent\alpha}
%     (s,a)\right] \ent{{} + \alpha\,\cH(\pi(\cdot\mid s))}
% \end{equation}
\begin{equation}\label{eq:soft_pi_state}
    \pi_{\text{new}} \in 
    \underset{\pi(\cdot\mid s)\in\Delta(\cA_s)}{\arg\max}
    {\E}_{a\sim\pi(\cdot\mid s)}\!\left[Q^{\pi_{\text{old}}}_{\ent\alpha}
    (s,a)\right] \ent{{} + \alpha\,\cH(\pi(\cdot\mid s))}
\end{equation}
satisfies $V^{\pi_{\text{new}}}_{\ent\alpha}(s) \geq 
V^{\pi_{\text{old}}}_{\ent\alpha}(s)$ for every $s$. At $\alpha = 0$ 
this is the classical greedy improvement theorem 
\citep{kakade2002approximately}; at \ent{$\alpha > 0$} it is its soft 
analogue \citep[Section 5]{chan2022greedification}. Denote the 
right-hand side of~\eqref{eq:soft_pi_state} by 
$\cG^{\pi_{\text{old}}}_{\ent\alpha}(\pi; s)$.

\textbf{Vanilla Policy Gradient.}
Policy gradient methods~\citep{williams1991function,
williams1992simple,sutton1999policy} parameterize the policy 
$\pi_\theta$ and directly optimize $J_\alpha(\pi_\theta)$ by 
stochastic gradient ascent. For an episodic trajectory $\tau = 
(s_0, a_0, \dots, a_{T-1}, s_T)$ of random length $T$, the policy 
gradient theorem yields
\begin{equation}\label{eq:pg_theorem}
    \textstyle 
    \nabla_\theta J_\alpha(\pi_\theta) = 
    \E_{\tau \sim \pi_\theta}\!\left[ 
    \sum_{t=0}^{T-1} \nabla_\theta \log \pi_\theta(a_t \mid s_t) 
    \, A^{\pi_\theta}_\alpha(s_t, a_t) \right]\!,
\end{equation}
where $A^{\pi_\theta}_\alpha$ is the soft advantage 
from~\cref{eq:advantage}. The advantage 
admits multiple estimators with different bias-variance trade-offs, 
which we study in~\cref{sec:variance_reduction}.

\textbf{Proximal Policy Optimization.} PPO~\citep{schulman2017proximal} improves on Vanilla Policy Gradient (VPG) with markedly better stability and sample efficiency. Inspired by the same principles as Trust-Region Policy Optimization~\citep[TRPO]{schulman2015trust}, it introduces two modifications absent in VPG: a clipping mechanism that stabilizes updates, and importance sampling, which enables multiple off-policy updates on freshly sampled rollouts. Together, these yield the following objective:
\begin{align}
    \textstyle \cL_{\mathrm{PPO}}(\theta; \pi_{\text{old}}) \triangleq 
    \E_{\tau \sim \pi_{\text{old}}}\!\left[\sum_{t=0}^{T-1}
    \mathrm{PPOClip}\bigl(\rho_t(\theta),\, A^{\pi_{\text{old}}}_t
    \bigr)\right],\label{eq:ppo_clip}
\end{align}
where $\mathrm{PPOClip}(\rho, A) \triangleq 
\min\!\bigl( \rho A,\; \mathrm{clip}(\rho,1{-}\epsilon,
1{+}\epsilon)\,A \bigr)$, $\rho_t(\theta) \triangleq \pi_\theta(a_t \mid s_t) / \pi_{\text{old}}(a_t \mid s_t)$, and $\epsilon > 0$ is a small constant. We include a detailed discussion of PPO in \cref{app:ppo} for further clarification.

\subsection{Generative Flow Networks as Soft RL}\label{sec:gfn_background}

In this section, we introduce GFlowNets and recall their equivalence with entropy-regularized RL at $\alpha = \gamma = 1$~\citep{tiapkin2024generative,deleu2024discrete}. 
For detailed derivations of the GFlowNet objects, see~\citep{bengio2023gflownet, malkin2022trajectory}.

\textbf{Directed acyclic graphs.} The generation process in GFlowNets is modelled as a sequence of constructive actions: we begin with an "empty" state and add a new component at each step. Formally, this is described by a finite directed acyclic graph (DAG) $\cG = (\cS, \cE)$, where $\cS$ is a state space and $\cE \subseteq \cS \times \cS$ is a set of edges. For each state $s \in \cS$, the set of actions $\cA_s$ consists of all possible next states: $s' \in \cA_s \Leftrightarrow s \to s' \in \cE$. Each action corresponds to attaching a new component to the object represented by $s$, producing a new object $s'$. When $s' \in \cA_s$, we call $s'$ a child of $s$ and $s$ a parent of $s'$. A designated initial state $\sinit \in \cS$ represents the "empty" object and is the unique state with no incoming edges. Every other state is reachable from $\sinit$, and the set of terminal states, those with no outgoing edges, coincides with the space of interest $\cX$.

\textbf{Markovian flows.}
GFlowNets introduce state flows $\cF(s)$ and edge flows $\cF(s\to s')$.
The training goal is to match terminal flows to rewards,
$\cF(x)=\cR(x)$ for $x\in\cX$; a correct flow then satisfies
$\cF(\sinit)=\rmZ$ and induces the target distribution $\cR(x)/\rmZ$.
The associated forward and backward policies are
% \begin{equation*}
%     \PF(s'\mid s)=\frac{\cF(s\to s')}{\cF(s)}, 
%     \quad
%     \PB(s\mid s')=\frac{\cF(s\to s')}{\cF(s')},
% \end{equation*}
$\PF(s'\mid s)=\frac{\cF(s\to s')}{\cF(s)}$ and $\PB(s\mid s')=\frac{\cF(s\to s')}{\cF(s')}$, where $\cF(s)$ denotes the total outgoing flow at nonterminal states and
the total incoming flow at noninitial states. Thus $\PF$ constructs objects from $\sinit$, while $\PB$ decomposes trajectories in reverse. Equivalently, a Markovian flow is determined by either $(\rmZ,\PF)$ or $(\cR,\PB)$.

\textbf{Training a GFlowNet.} 
Given the reward $\cR$, standard GFlowNet objectives train a forward policy $\PF(\cdot\mid\cdot,\theta)$, together with auxiliary quantities such as a state flow $\cF_\theta$ and, optionally, a backward policy $\PB$, by enforcing flow consistency over the DAG. A unifying objective is Subtrajectory Balance (SubTB,~\citealt{madan2023learning}), which applies to partial trajectories $\tau=(s_m,\ldots,s_n)$:
\begin{equation}\label{eq:SubTB_loss}
    \mathcal{L}_{\mathrm{SubTB}}(\tau)
    =
    \left(
    \log
    \frac{
        \cF_\theta(s_m)\prod_{t=m}^{n-1}\PF(s_{t+1}\mid s_t,\theta)
    }{
        \cF_\theta(s_n)\prod_{t=m}^{n-1}\PB(s_t\mid s_{t+1},\theta)
    }
    \right)^2,
\end{equation}
with $\cF_\theta(x)=\cR(x)$ at terminal states. DB is the one-step special case $n=m+1$, while TB is the complete-trajectory special case $s_m=\sinit$, $s_n=x\in\cX$, with $\cF_\theta(\sinit)$ identified with $\rmZ_\theta$. SubTB scheme trains on a weighted combination of subtrajectory losses and, in the limits, recovers average DB and TB.
We use DB, TB, and SubTB as GFlowNet baselines in
\cref{sec:experiments}.

\textbf{Forward policy learning as a soft RL problem.}\label{sec:equivalence} \cite{tiapkin2024generative} show that GFlowNet training admits an equivalent entropy-regularized RL formulation when the backward policy $\PB$ is fixed. On the same DAG $\cG$, define an episodic MDP whose action at state $s_t$ is the choice of a child $s_{t+1} \in \mathcal{A}_{s_t}$. Since each action is identified with its successor, we write $s_{t+1}$ in place of $a_t$, so the policy reads $\pi_\theta(s_{t+1} \mid s_t)$. Identify the forward policy with the agent's policy, $\pi_\theta(s_{t+1} \mid s_t) \equiv \PF(s_{t+1} \mid s_t, \theta)$. The per-step reward is defined as
\begin{equation}\label{eq:sampling_reward}
    r(s_t, s_{t+1}) = \begin{cases}
    \log \PB(s_t \mid s_{t+1}), & s_t \notin \mathcal{X} \cup \{s_f\}, \\
    \log \cR(s_t), & s_t \in \mathcal{X}, \\
    0, & s_t = s_f,
    \end{cases}
\end{equation}
where $s_f$ is an absorbing terminal state appended to the DAG. The absorbing state is a standard technical device that extends each terminated GFlowNet trajectory to a varying-horizon one without altering its return. Crucially, once $\PB$ is fixed, the reward~\eqref{eq:sampling_reward} is itself fixed and does not depend on the forward policy. This is in contrast to~\cite{pmlr-v235-niu24c}, where the reward depends on the forward policy.

At $\gamma = 1$ and $\alpha = 1$, the soft-optimal policy under this formulation induces the target distribution $\cR(x) / \rmZ$ over terminals, so the solution to the entropy-regularized RL objective coincides with the solution to the sampling problem. To show this, let $\PB(\tau) = \frac{1}{\rmZ} \cR(s_T) \prod_{t=0}^{T-1} \PB(s_t \mid s_{t+1})$ be the backward-induced target distribution over complete trajectories. The reverse KL decomposes as~\citep[Proposition 1]{tiapkin2024generative}:
$\mathrm{KL}\,\left(\PF \,\|\, \PB\right) = -V^{\pi_\theta}_{\alpha=1}(\sinit) + \log \rmZ$.
Hence, minimizing the reverse KL is equivalent to maximizing the soft value $V^{\pi_\theta}_{\alpha=1}(\sinit)$. At optimum, $V^{\pistar}_{\alpha=1}(\sinit) = \log \rmZ$, and the soft-optimal policy $\pistar_{\alpha=1}$ induces the terminal distribution $\cP(x) \propto \cR(x)$. The dictionary of the two viewpoints is given in~\cref{tab:gfn_rl_dictionary}.

\begin{table*}[t]
  \centering
  \small
  \caption{Notation map from GFlowNets to entropy-regularized RL. $\alpha = 1$, $\gamma = 1$. $\pistar_{\alpha=1}$ stands for the optimal regularized policy. For detailed derivation, see~\cite{tiapkin2024generative}.}
  \label{tab:gfn_rl_dictionary}
  \begin{adjustbox}{max width=\textwidth}
  \begin{tabular}{@{}lcccccc@{}}
    \toprule
    \textbf{GFlowNets}
    & $\PF(s_{t+1}\mid s_t;\theta)$
    & $\log\PB(s_t\mid s_{t+1})$
    & $\log\cR(s_t)$
    & $\log\rmZ$
    & $\log\cF(s_t)$
    & $\log\cF(s_t\to s_{t+1})$ \\
    % \midrule
    \textbf{Soft RL}
    & $\pi_\theta(s_{t+1}\mid s_t)$
    & $r(s_t,s_{t+1})$
    & $r(s_t,s_f)$
    & $V^{\pistar}_{\alpha=1}(\sinit)$
    & $V^{\pistar}_{\alpha=1}(s_t)$
    & $Q^{\pistar}_{\alpha=1}(s_t,s_{t+1})$ \\
    \bottomrule
  \end{tabular}
  \end{adjustbox}
  \vspace{-0.3cm}
\end{table*}

\section{Related work}

The original GFlowNet paper~\citep{bengio2021flow} briefly experimented with PPO as a forward-policy training method but reported negative results. Two ingredients were missing. First, the intermediate per-step rewards $\log \PB(s_t \mid s_{t+1})$, identified by~\cite{tiapkin2024generative} as necessary to align the soft-RL optimum with the sampling target, were absent --- their formulation provides only the terminal $\log \cR(x)$, reducing the problem to standard RL with a sparse reward. Second, the entropy regularization required by the soft-RL equivalence was not included. The combination of these omissions makes the optimum of their objective $\arg\max_x \cR(x)$ rather than $\cR(x)/\rmZ$, which is consistent with the mode-collapse behaviour they reported. Even with the correct intermediate rewards~\eqref{eq:sampling_reward} and soft advantages, dropping only the KL term of Ent-PPO makes naive PPO degrade increasingly with the number of update epochs (\cref{sec:experiments}).

\cite{malkin2023gflownets} draw a formal connection between GFlowNets and variational inference, showing that the expected on-policy gradient of TB with respect to the parameters of $\PF$ coincides with the REINFORCE gradient of the reverse KL, $\mathrm{KL}(\PF \,\|\, \PB)$. In our soft-RL terminology, this estimator is the simplest policy gradient (\cref{sec:sgp}) with a \emph{scalar} baseline: the trajectory-level coefficient they multiply with $\nabla_\theta \log \PF(\tau)$ equals (up to sign) the soft return $\sum_t g_t$, and the running-average control variate they subtract is shown to follow the same update rule as the learned $\log \rmZ$ parameter in TB --- giving a clean interpretation of $\log \rmZ$ as a constant baseline for the simplest gradient. They do not study reward-to-go, state-dependent baselines, or advantage estimation, which we show in~\cref{sec:variance_reduction} substantially improve performance.

The closest prior work is~\citet{pmlr-v235-niu24c}, which applies VPG and TRPO to GFlowNet training (we include TRPO in our experiments) but does not consider PPO. They adopt a policy-dependent reward, $\log\tfrac{\PF}{\PB}$ for transitions to non-terminal states (their Definition~3.2), whereas we use the stationary reward $\log\PB$ with an explicit entropy term; they also relate their method to soft Q-learning through the latter view~\citep[App.~B.5]{pmlr-v235-niu24c}. The two formulations share the objective, the value functions and the on-policy gradient, and the soft performance-difference lemma (\cref{lem:soft_pdl}) also appears, in the GFlowNet setting, in the proof of~\citet[Lemma~C.2]{pmlr-v235-niu24c}. They differ in multi-epoch updates: in the stationary formulation, the correct off-policy surrogate, including a KL term of unit weight, is the default, whereas with a policy-dependent reward the standard practice of reusing rollout advantages drops this term, and the TRPO of~\citet{pmlr-v235-niu24c} instead bounds it through the trust-region radius. We discuss this in detail in \cref{app:policy_dependent}. Building on this framework, \cite{niu2026evaluating} propose a SubTB-based procedure for training value functions used in policy-based algorithms. Their procedure is complementary to ours: it could be used to train the value function that we use in gradient estimates in~\cref{sec:variance_reduction}. We include it in our variance-reduction experiments in \cref{sec:exp_variance_reduction} alongside standard RL variance reduction techniques~\citep{Schulman2015HighDimensionalCC}.

\vspace{-0.12cm}
\section{Methodology}\label{sec:methodology}
\vspace{-0.12cm}

This section starts from VPG and discusses techniques from the RL and GFlowNets literature that improve policy-gradient estimation~\citep{sutton1999policy,Schulman2015HighDimensionalCC,niu2026evaluating}. We include the derivations of the subsequent variance reduction approaches only for the sake of completeness in~\cref{app:standard_derivations}. These techniques are well-known in RL literature and can be found in~\citep{sutton1999policy, Schulman2015HighDimensionalCC, williams1992simple}. Throughout, we specialize to the GFlowNet regime $\alpha = \gamma = 1$. 
To keep formulas light, we write
$
    \tilde V^\pi := V^\pi_{\alpha=1},
    \tilde Q^\pi := Q^\pi_{\alpha=1},
    \tilde A^\pi := A^\pi_{\alpha=1},
$
and $\tilde J(\pi) := \tilde V^\pi(\sinit)$ for the corresponding soft return.

\subsection{Variance reduction in Vanilla Policy Gradient}
\label{sec:variance_reduction}

VPG samples a batch of trajectories from the current policy, 
estimates the policy gradient via~\cref{eq:pg_theorem}, and takes a 
stochastic gradient ascent step. The full derivation is given 
in~\cref{app:vpg}; the resulting identity takes the standard form
\begin{equation}\label{eq:soft_pg}
    \textstyle \nabla_\theta \tilde J(\pi_\theta) = 
    \E_{\pi_\theta}\!\left[ \sum_{t=0}^{T-1} 
    \nabla_\theta \log \pi_\theta(s_{t+1} \mid s_t)\, \Psi_t \right]\!,
\end{equation}
where different choices of $\Psi_t$ yield different estimators of the gradient with different variance. In implementation, 
$\Psi_t$ is detached from the computation graph; the score-function 
identity guarantees that gradients flowing through $\Psi_t$ vanish in 
expectation. We now describe five estimators in order of conceptual development, each building on the previous: the simplest policy gradient, reward-to-go, reward-to-go with a learned value baseline, Generalized Advantage Estimation (GAE), and GAE with Subtrajectory Evaluation Balance~\citep{niu2026evaluating}. Their empirical ranking is given in~\cref{sec:experiments}.

\textbf{Simplest policy gradient.}\label{sec:sgp}
The simplest estimator uses the full-trajectory soft return, which is the same for every gradient term (i.e., does not depend on $t$):
\begin{equation}\label{eq:spg}
    \textstyle \Psi_t = \sum_{k=0}^{T-1} g_k, \qquad
    g_k = r(s_k, s_{k+1}) - \log \pi_\theta(s_{k+1} \mid s_k),
\end{equation}
where $g_k$ is the one-step soft return. Substituting into 
\cref{eq:soft_pg} gives an unbiased estimator of the policy gradient 
(see \cref{app:vpg}), but with high variance.

\textbf{Reward-to-go.} \label{sec:rtg}
% do not causally depend on the action at step $t$ and
A lower-variance estimator follows from the observation that rewards obtained before step $t$ contribute zero in expectation to the corresponding gradient term 
(the derivation in \cref{app:rtg}). Dropping them gives the \emph{reward-to-go}:
\begin{equation}\label{eq:vr_rtg}
    \textstyle \Psi_t = \hat R_t = \sum_{k=t}^{T-1} g_k.
\end{equation}
The estimator is unbiased and has lower variance than the previous. All subsequent estimators build on reward-to-go, which empirically dominates the simplest form. This finding is consistent with RL literature.

\textbf{Learning the baseline function.}\label{sec:baseline} Subtracting any function $b(s_t)$ that depends only on the state from 
the return preserves unbiasedness, since $\E_{s_{t+1}}[\nabla_\theta \log \pi_\theta(s_{t+1} \mid s_t)\, b(s_t)] = 0$ 
(see \cref{app:baseline}). Such a $b$ is called a \emph{baseline}. The standard choice is the on-policy soft value $V^{\pi_\theta}(s_t)$, 
the expected soft return from $s_t$ under $\pi_\theta$. Subtracting it 
from $\hat R_t$ removes the trajectory-level component shared across 
actions and substantially reduces gradient variance. Since 
$V^{\pi_\theta}$ is unknown, we approximate it with a neural network 
$V_\phi(s_t)$, trained alongside the policy to track the value of the 
current policy. This gives the advantage estimator
\begin{equation}\label{eq:vr_baseline}
    \textstyle \Psi_t = \hat A_t = \hat R_t - V_\phi(s_t).
\end{equation}
We train $V_\phi$ by mean-squared-error regression onto the empirical 
soft return $\hat R_t = \sum_{k=t}^{T-1} g_k$, taking a few gradient 
steps per iteration (see \cref{alg:pg}, line 7). Each iteration runs $E$ \emph{epochs} of stochastic gradient descent over the rollout batch, with each epoch splitting the batch into $S$ mini-batches. $E$ and $S$ are hyperparameters we ablate in~\cref{sec:exp_gae}. We use a separate neural network for $V_\phi$, following standard RL 
practice.

\textbf{Generalized Advantage Estimation.}\label{sec:gae}
GAE~\citep{Schulman2015HighDimensionalCC} forms the advantage as 
an exponentially weighted sum of $k$-step soft TD residuals. 
Specialised to $\gamma = 1$, the soft TD residual is
% \begin{equation}\label{eq:td}
%     \textstyle \delta_t = g_t + V_\phi(s_{t+1}) - V_\phi(s_t),
% \end{equation}
$\delta_t = g_t + V_\phi(s_{t+1}) - V_\phi(s_t),$
with \(V_\phi(s_T):=0\) at the absorbing terminal state. The GAE 
advantage estimator is
\begin{equation}\label{eq:vr_gae}
    \textstyle \hat A_t = \sum_{k=0}^{T-1-t} \lambda^{k}\, \delta_{t+k}, 
    \qquad \lambda \in [0,1].
\end{equation}
The parameter $\lambda$ trades off bias and variance asymmetrically: 
$\lambda = 0$ collapses to the one-step TD residual, which is low-variance 
but biased whenever $V_\phi \neq \tilde{V}^{\pi_\theta}$, while $\lambda = 1$ 
recovers reward-to-go with baseline, unbiased for any $V_\phi$ but 
with full Monte-Carlo variance. Intermediate values interpolate. The 
choice of $\lambda$ has a significant effect on convergence speed in 
our experiments (see \cref{sec:experiments}).

When training with GAE, we replace the Monte-Carlo target for $V_\phi$ 
with the bootstrapped target $\hat V^{\text{tgt}}_t = 
\mathrm{sg}[\hat A_t + V_\phi(s_t)]$, where $\mathrm{sg}$ denotes the 
stop-gradient operator. This keeps the value-function target 
consistent with the bias–variance trade-off that $\lambda$ selects. 
All log-probabilities entering the targets $\hat R_t$ and 
$\hat V^{\text{tgt}}_t$ are evaluated at the rollout-time policy and 
held fixed across update epochs.

\textbf{Subtrajectory Evaluation Balance GAE.}\label{sec:subeb} \citep{niu2026evaluating} replace the bootstrapped Monte-Carlo regression target for $V_\phi$ with a balance-based objective inspired by Sub-TB~\citep{madan2023learning}. They observe that, for a fixed $\pi_\theta$, the true soft value $\tilde V^{\pi_\theta}(s_t)$ is uniquely characterised by the \emph{subtrajectory evaluation balance} (Sub-EB) condition: for all $i < j$,
$\textstyle \E_{P_F(\tau_{i:j})}\left[\log P_F(\tau_{i:j}\mid s_i) + V_\phi(s_i)\right] = \E_{P_F(\tau_{i:j})}\left[\log P_B(\tau_{i:j}\mid s_j) + V_\phi(s_j)\right],$
with the boundary condition $\exp V_\phi(x) := R(x)$ at terminal states. The corresponding training loss is a weighted sum of squared log-ratio residuals over all subtrajectories of a rollout,
\begin{equation}\label{eq:subeb_loss}
\begin{split}
    \textstyle \mathcal{L}_V(\phi) &= \E_{P_F(\tau)}\!\left[\sum_{i<j} w_{j-i}\,\bigl(\delta_V(\tau_{i:j};\phi)\bigr)^2\right], \\ \qquad \delta_V(\tau_{i:j};\phi) &= \log\frac{P_F(\tau_{i:j}\mid s_i)\exp V_\phi(s_i)}{P_B(\tau_{i:j}\mid s_j)\exp V_\phi(s_j)},
\end{split}
\end{equation}
with $\theta$ frozen during the critic update. Compared to the bootstrapped MSE target of \cref{eq:vr_baseline}, the Sub-EB loss has two distinguishing features: residuals are formed at the subtrajectory level rather than per edge, and $V_\phi(s_t)$ receives gradient signal from subtrajectories that both start at and end at $s_t$, rather than only those starting at $s_t$. We include SubEB-GAE as a baseline in \cref{sec:experiments}. In contrast to other variance reduction schemes, this arises from GFlowNet literature, not RL. However, it also may be viewed through RL lens as a optimization of a sum of $\mathrm{TD}(k)$-style estimates, whereas GAE plays a role of a single $\mathrm{TD}(\lambda)$-style estimator~\cite{Schulman2015HighDimensionalCC}.

\subsection{Entropic Proximal Policy Optimization}
\label{sec:ent_ppo}

Standard PPO maximizes expected return with a hand-tuned entropy bonus. Even with the GFlowNet sampling reward of 
\cref{eq:sampling_reward}, this objective ignores the entropy term  required by the soft-RL equivalence (\cref{sec:equivalence}): its 
optimum concentrates on $\arg\max_{x}\cR(x)$ rather than sampling from $\cR/\rmZ$~\citep{bengio2021flow}. We now derive a variant of  PPO that maximizes the soft return directly, which we call \emph{Entropic PPO} (Ent-PPO). The construction follows the same three steps as for standard PPO~(\cref{app:ppo}), with the per-state PI operator replaced by its soft analogue~\eqref{eq:soft_pi_state}. A new feature is that soft PI itself produces a KL term of unit weight against $\pi_{\text{old}}$; this term accounts for the entropy and is distinct from the trust region, which is still provided by clipping.

\textbf{Soft policy improvement.}
At $\alpha = 1$, the soft PI operator~\eqref{eq:soft_pi_state} 
returns at every state $s$ the maximizer of $\cG^{\pi_{\text{old}}}_{\alpha=1}
(\pi; s) = \E_{a\sim\pi(\cdot\mid s)}[\tilde Q^{\pi_{\text{old}}}
(s, a)] + \cH(\pi(\cdot\mid s))$. Aggregating along trajectories 
drawn from $\pi_{\text{old}}$ yields the parametric surrogate, soft analogue of \eqref{eq:pi-surrogate}:
\begin{equation}\label{eq:soft_pi_aggregate}
    \textstyle 
    \cL_{\text{soft}}(\theta; \pi_{\text{old}}) \triangleq 
    \E_{\tau \sim \pi_{\text{old}}}\left[\sum_{t=0}^{T-1}\!
    \cG^{\pi_{\text{old}}}_{\alpha=1}\!(\pi_\theta;\, s_t)\right]\,.
\end{equation}
% Recall the action-as-successor convention: $\tilde Q^{\pi_{\text{old}}}(s_t, s_{t+1})  = r(s_t, s_{t+1}) + \tilde V^{\pi_{\text{old}}}(s_{t+1})$.

\textbf{Unit-weight KL term from the soft advantage.}
Using the soft-advantage identity $\tilde Q^{\pi_{\text{old}}} = 
\tilde A^{\pi_{\text{old}}} + \tilde V^{\pi_{\text{old}}} + 
\log\pi_{\text{old}}$ at $\alpha = 1$ (rearranging 
\eqref{eq:advantage}) and dropping the 
$\theta$-independent baseline $\tilde V^{\pi_{\text{old}}}(s_t)$, 
the per-step integrand becomes
\begin{align}\notag
    \cG^{\pi_{\text{old}}}(\pi_\theta; s_t) \,\stackrel{c}{=}\,
    & \E_{s'\sim\pi_\theta(\cdot\mid s_t)}\!\big[
    \tilde A^{\pi_{\text{old}}}(s_t, s') + 
    \log\pi_{\text{old}}(s'\mid s_t)\big] \\
    & + \cH(\pi_\theta(\cdot\mid s_t))\,,\label{eq:soft_pi_via_A}
\end{align}
where $\stackrel{c}{=}$ denotes equality up to a $\theta$-independent 
constant. The cross-entropy and entropy combine into a single KL,
\begin{align}\label{eq:kl_combination} 
\E_{s'\sim\pi_\theta}\!\big[\log\pi_{\text{old}}(s'\mid s_t)\big] 
    &+ \cH(\pi_\theta(\cdot\mid s_t)) \\
    &= 
    -\mathrm{KL}\!\big(\pi_\theta(\cdot\mid s_t)\,\|\,
    \pi_{\text{old}}(\cdot\mid s_t)\big)\,,\notag
\end{align}
so the integrand reduces to $\E_{\pi_\theta}[\tilde A^{\pi_{\text{old}}}
(s_t, s')] - \mathrm{KL}(\pi_\theta(\cdot\mid s_t)\,\|\,
\pi_{\text{old}}(\cdot\mid s_t))$. The KL is intrinsic to soft PI: it accounts for the changes in the entropy terms as $\pi_\theta$ moves away from $\pi_{\text{old}}$, with weight fixed to one by the soft-RL equivalence (\cref{app:policy_dependent}).

\textbf{Importance-weighted, clipped surrogate.}
We importance-weight the advantage term as 
in~\eqref{eq:unclipped-surrogate}; the KL is computed analytically 
and requires no variance control. This yields the unclipped 
Ent-PPO surrogate
\begin{equation}\label{eq:ent_ppo_unclipped}
    \E_{\tau\sim\pi_{\text{old}}}\!\left[\sum_{t=0}^{T-1}\!
    \rho_t(\theta)\,\tilde A^{\pi_{\text{old}}}_t - 
    \KL\!\big(\pi_\theta(\cdot\mid s_t)\,\|\,
    \pi_{\text{old}}(\cdot\mid s_t)\big)\right]\!,
\end{equation}
with $\tilde A^{\pi_{\text{old}}}_t \triangleq 
\tilde A^{\pi_{\text{old}}}(s_t, s_{t+1})$. The trust-region rationale that motivates clipping in~\eqref{eq:ppo_clip_repeat} applies unchanged: the analytic KL does not keep $\pi_\theta$ close to $\pi_{\text{old}}$, since the per-state maximizer of~\eqref{eq:ent_ppo_unclipped} is the soft-greedy policy $\pi^\sharp_{\text{old}}$, which may be far from $\pi_{\text{old}}$ (\cref{prop:phi_decomposition}, \cref{app:policy_dependent}). We therefore clip $\rho_t$ as in standard PPO, which yields the final Ent-PPO objective:
\begin{equation}\label{eq:ent_ppo}
    \textstyle
    \cL_{\mathrm{Ent\text{-}PPO}}(\theta) \triangleq 
    \E_{\tau\sim\pi_{\text{old}}}\!\left[\sum_{t=0}^{T-1}\!
    \ell^{\mathrm{Ent}}_t(\theta)\right]\!,
\end{equation}
where the per-step loss combines the clipped advantage with the KL term:
$\ell^{\mathrm{Ent}}_t(\theta) \triangleq 
        \mathrm{PPOClip}(\rho_t(\theta),\, 
        \tilde A^{\pi_{\text{old}}}_t)
        - \KL(\pi_\theta(\cdot\mid s_t)\,\|\,
        \pi_{\text{old}}(\cdot\mid s_t)).
$
We also refer to \cref{app:entppo} for a more explicit TRPO-style derivation~\citep{kakade2002approximately,schulman2015trust}.

\textbf{Comparison to standard PPO.}
PPO maximizes
$\E\big[\mathrm{PPOClip}(\rho_t(\theta),\, A_t) + c_2\,\cH(\pi_\theta(\cdot\mid s_t))\big],$
with $A_t$ the standard (non-soft) advantage and $c_2\geq0$ a free exploration hyperparameter. Ent-PPO differs in three coupled ways: the advantage is soft; the entropy coefficient is fixed at $\alpha = 1$ by the soft-RL equivalence rather than left as a free $c_2$; and 
the entropy is paired with a cross-entropy against $\pi_{\text{old}}$, the two combining into a KL term. Decomposing $\KL(\pi_\theta\,\|\,\pi_{\text{old}}) = -\cH(\pi_\theta) - \E_{\pi_\theta}[\log\pi_{\text{old}}]$ makes the structure transparent. The cross-entropy cancels the term $-\log\pi_{\text{old}}(s_{t+1}\mid s_t)$ contained in the soft advantage, i.e.\ the entropy bonus evaluated at the behaviour policy, and the entropy term replaces it by the entropy of the current policy. Without clipping, the Ent-PPO integrand therefore equals $\E_{s'\sim\pi_\theta(\cdot\mid s_t)}[r(s_t,s') + \tilde V^{\pi_{\text{old}}}(s') - \tilde V^{\pi_{\text{old}}}(s_t)] + \cH(\pi_\theta(\cdot\mid s_t))$: an entropy bonus with coefficient fixed to one, the principled analogue of standard PPO's heuristic $c_2\cH$, combined with values that include future entropies. The KL thus corrects the objective rather than restricting the update, and it does not bound the importance ratios. We show empirically in \cref{sec:experiments} that dropping it degrades PPO increasingly with the number of update epochs per batch.

\section{Experiments}\label{sec:experiments}

\begin{figure*}[t]
    \centering
    \includegraphics[width=0.9\linewidth]{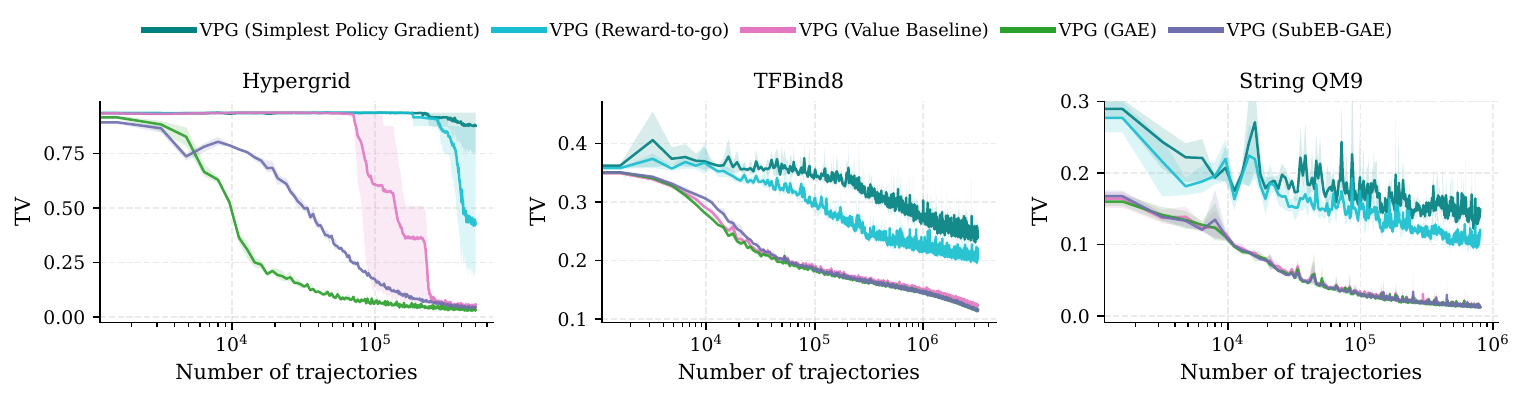}
    \caption{Comparison of five advantage estimators for VPG (simplest policy gradient, reward-to-go, reward-to-go with value baseline, SubEB-GAE, and GAE)  on Hypergrid, TFBind8, and String QM9. Curves show TV distance vs.\ the number of reward evaluations. Lower is better. 
    Lines show the mean over 3 seeds. 
    % Shaded regions show the min--max range.
    }
    \label{fig:var_reduction}
    \vspace{-0.4cm}
\end{figure*}

\subsection{Environments and evaluation}\label{sec:exp_envs}

%We evaluate on three benchmark environments from the GFlowNet literature, all implemented in gfnx.
We use a standard range of environments and metrics in our experiments. We have a set of smaller environments: Hypergrid~\citep{bengio2021flow} and sequence-based environments TFBind8 and String QM9~\citep{shen2023towards}. For these problems, we use the Total Variation (TV) distance between the sampling model and the ground truth distribution. For ablation studies, we use Area Under the TV Curve (AUC) as a marginal metric measuring both the convergence speed and the final TV value. 
As for larger problems, we consider Ising~\citep{ising1925beitrag} and molecular-graph generation problems, sEH~\citep{bengio2021flow} and QM9~\citep{jain2023multi}. Note that graph-based variant of QM9 induces a larger and more complex space than the small sequence-based variant. For larger problems, we use the Evidence Lower Bound (ELBO) and dataset-based Evidence Upper Bound (EUBO) following~\citep{blessing2024beyond}, which measure intra- and inter-mode coverage. 
% For Ising, we additionally evaluate magnetisation and correlation errors. 
See~\cref{app:exp_envs} for a more detailed description.

\subsection{Variance Reduction in VPG}\label{sec:experiments_vpg}

We first study the five advantage estimators of~\cref{sec:variance_reduction} on VPG, identifying the best variance-reduction strategy before turning to PPO. \cref{fig:var_reduction} compares the simplest policy gradient, reward-to-go, reward-to-go with a learned value baseline, SubEB-GAE~\citep{niu2026evaluating}, and GAE~\citep{Schulman2015HighDimensionalCC} on all three environments.

The ordering is consistent with established results in the RL 
literature: convergence improves monotonically as we move from the 
simplest gradient to GAE, and final TV improves correspondingly. The simplest gradient and reward-to-go converge slowly and to substantially worse final solutions on all benchmarks. Adding a learned value baseline produces a large improvement, and GAE further improves convergence speed on Hypergrid and final TV on TFBind8, while matching the value-baseline variant on QM9. Since SubEB-GAE performs slightly worse than GAE, we use GAE as the variance-reduction strategy for the rest experiments.

\label{sec:exp_gae}
\textbf{GAE ablation.} We study the sensitivity of GAE training to three hyperparameters mentioned in methodology: 
the GAE parameter $\lambda$, and two value-network training choices: the number of training epochs per rollout batch and the number of mini-batch splits within each epoch. We perform a full grid search over $\lambda \in \{0.5, 0.6, 0.7, 0.8, 0.9\}$, epochs $\in \{1, 2, 4, 8\}$, and splits $\in \{1, 2, 4, 8\}$ on Hypergrid, 
% (full results in Appendix~\ref{app:gae_grid}), 
identifying $(\lambda, E, S) = (0.7, 4, 2)$ as the best configuration. \cref{fig:GAE_boxpolts} then varies each hyperparameter individually around this optimum. AUC is stable across $\lambda \in [0.5, 0.8]$, with $\lambda = 0.7$ achieving the best median. Higher values ($\lambda = 0.9$) substantially degrade performance, confirming the importance of tuning the bias-variance trade-off. The number of value-network training epochs has the largest effect among the three hyperparameters. A single epoch yields substantially higher AUC, indicating that the value network underfits the rollout data. Performance improves monotonically up to 4 epochs, then degrades for 8. Training the value network with a single mini-batch (splits $= 1$) yields noticeably worse AUC, while splits $\in \{2, 4\}$ perform comparably. We prefer splits $= 2$ for its lower computational cost. Splits $= 8$ degrades performance, likely due to excessive noise in stochastic gradient updates.

\label{sec:exp_variance_reduction}

\begin{figure*}[t]
    \centering
    \includegraphics[width=0.9\linewidth]{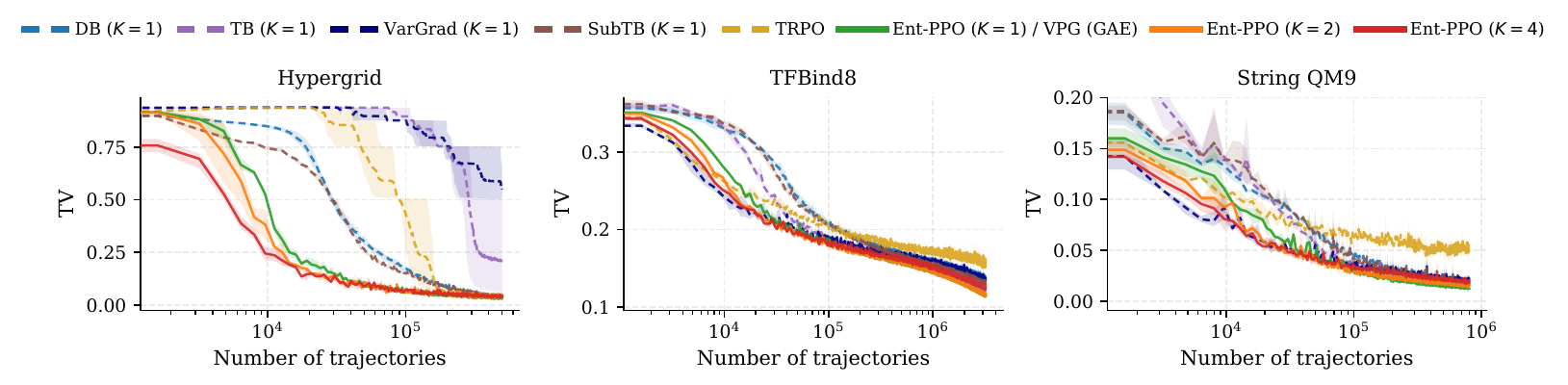}
    \caption{Ent-PPO with $K \in \{1, 2, 4\}$ update epochs per batch, compared against GFlowNet baselines (DB, TB, VarGrad, SubTB, and TRPO) on Hypergrid, TFBind8, and String QM9. Note that Ent-PPO ($K=1$) is equivalent to VPG (GAE). Curves show TV distance vs.\ the number of reward evaluations. Lower is better. 
    % Lines show the mean over 3 seeds. 
    % Shaded regions show the min--max range.
    }
    \vspace{-0.3cm}
    \label{fig:small_exps_EntPPO_vs_baselines}
\end{figure*}

\begin{figure*}[t]
    \centering $\scriptscriptstyle{\mathrm{sEH}}$
    % \centering $\mathrm{sEH}$
    \centering
    \includegraphics[width=1\linewidth]{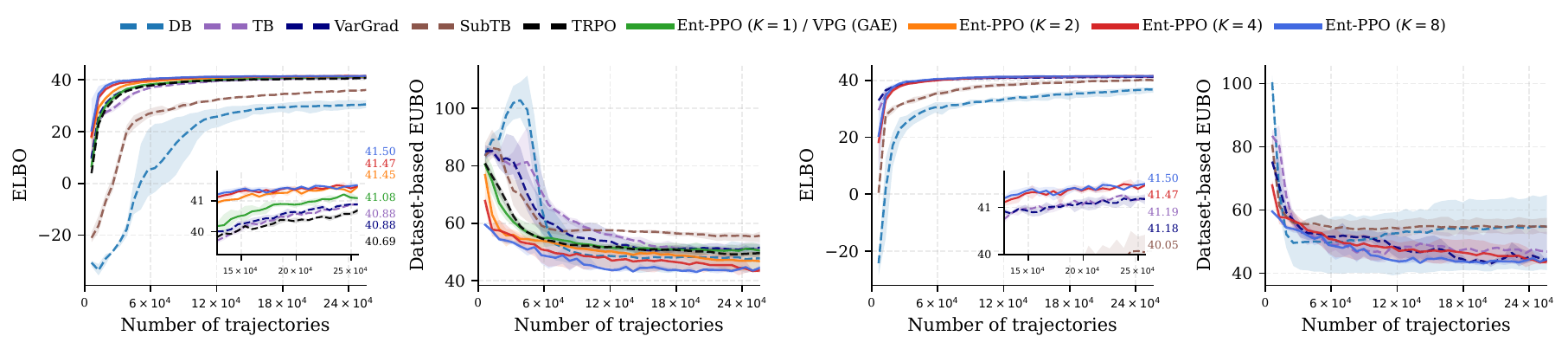}
    \caption{Ent-PPO compared against GFlowNet baselines: DB, TB, VarGrad, SubTB, and TRPO on sEH. Left: baselines use a single update per batch, while Ent-PPO varies $K \in \{1, 2, 4, 8\}$ update epochs. Right: baselines use $K=4$ updates per batch, while Ent-PPO varies $K \in \{4, 8\}$ update epochs. Odd panels show ELBO vs.\ the number of reward evaluations (higher is better). Even panels show Dataset-based EUBO vs.\ the number of reward evaluations (lower is better). Lines show the mean over 3 seeds. 
    % Shaded regions show the min--max range.
    }
    \label{fig:seh_fixedpb_combined_k_panels}
    \vspace{-0.4cm}
\end{figure*}

\subsection{Entropic Proximal Policy Optimization}

\textbf{Ent-PPO outperforms other methods.} We compare Ent-PPO against DB~\citep{bengio2023gflownet}, TB~\citep{malkin2022trajectory}, SubTB~\citep{madan2023learning}, VarGrad~\citep{richter2020vargrad, zhang2023robust}, and TRPO~\citep{pmlr-v235-niu24c}. Since the performance of TRPO in our setup differs from that reported in \citep{pmlr-v235-niu24c}, we additionally investigate this in~\cref{app:trpo_ablation}. 
\cref{fig:small_exps_EntPPO_vs_baselines} shows results on small problems. Even Ent-PPO ($K=1$), which reduces to VPG (GAE), matches or exceeds all baselines. VarGrad is the only competitive baseline, on TFBind8 and String QM9, but it fails on Hypergrid.
\cref{fig:seh_fixedpb_combined_k_panels} presents results on the larger problem, sEH. Ent-PPO ($K=8$) achieves the best ELBO among all methods and a substantially lower dataset-based EUBO than the baselines, indicating better coverage of the target distribution on both metrics. Other Ent-PPO variants also achieve better results than the baselines.
\cref{fig:qm9_fixedpb_combined_k_panels} compares the methods on the largest considered problem, QM9. The performance gap between Ent-PPO and the off-policy baselines is even larger, showing the scalability of our method, as well as that of the policy gradient algorithms.

\textbf{Multiple update epochs improve sample efficiency.} Increasing $K$ yields improvement gains across environments, as shown in \cref{fig:small_exps_EntPPO_vs_baselines}, \cref{fig:seh_fixedpb_combined_k_panels}, and \cref{fig:qm9_fixedpb_combined_k_panels}. We additionally consider the setup where off-policy baselines perform multiple gradient updates per sampled batch of trajectories, similarly to Ent-PPO. 
% This is theoretically viable since the baseline objectives are off-policy~\citep{tiapkin2024generative}.
\cref{fig:seh_fixedpb_combined_k_panels}, \cref{fig:qm9_fixedpb_combined_k_panels}, and \cref{fig:small_exps_EntPPO_4epochs} show that Ent-PPO still retains its advantage in this setup.

\textbf{The KL term is essential.} The KL term distinguishes Ent-PPO from naive PPO and arises directly from the soft-RL derivation (\cref{eq:ent_ppo}). As predicted in \cref{app:policy_dependent}, the two coincide to first order, so for $K = 4$ removing the KL has little effect beyond a slightly slower start on Hypergrid (\cref{fig:EntPPO_vs_PPO}). The gap grows with $K$, and for $K \in \{16, 32\}$ naive PPO plateaus at a higher TV, most visibly on Hypergrid. The entropy correction required by the soft-RL formulation thus matters precisely in the multi-epoch regime that PPO is designed for.

\textbf{Clipping is essential too.} Since the unit-weight KL is not a proximity term, \cref{prop:surrogate_error} predicts that a proper trust region is still needed. \cref{fig:EntPPO_no_clip} confirms this: without clipping, training is comparable for $K = 4$, unstable for $K = 8$, and fails for $K \in \{16, 32\}$. Without clipping, the $K$ inner epochs push $\pi_\theta$ towards $\pi^\sharp_{\text{old}}$, far outside the region where $\tilde\cL_{\text{soft}}$ approximates $\tilde J(\pi_\theta) - \tilde J(\pi_{\text{old}})$, and the importance ratios become unbounded. The KL determines where each update points; clipping limits how far it goes.

\textbf{We include the results on backward policy learning, Ising and QM9 problems, and comparison with discrete diffusion samplers in~\cref{app:extended_discussion_experiments}}.

\section{Conclusion}
\label{sec:conclusion}

We study policy gradient algorithms for sampling from discrete distributions under the GFlowNet framework. Our experiments show that, with appropriate variance reduction, even Vanilla Policy Gradient matches or outperforms established off-policy GFlowNet objectives on most of the considered problems. Building on this, we derive Ent-PPO, a PPO variant adapted to the soft-RL formulation of GFlowNet sampling, and show empirically that it further improves over VPG and matches or outperforms all baselines across the considered problems. The KL term in our derivation, which arises naturally from the entropy regularization required by the soft-RL equivalence, is essential: removing it causes naive PPO to degrade increasingly with the number of update epochs. An interesting direction for future work is adapting policy gradient methods within the RL framework to continuous domains sampling~\citep{lahlou2023theory, sendera2024improved, blessing2024beyond}.

\subsubsection*{Acknowledgements}
This research was supported in part through computational resources of HPC facilities at HSE University~\citep{kostenetskiy2021hpc}. 

\bibliography{bibfile}

\begin{thebibliography}{}

\bibitem[Achiam et~al., 2017]{achiam2017constrained}
Achiam, J., Held, D., Tamar, A., and Abbeel, P. (2017).
\newblock Constrained policy optimization.
\newblock In Precup, D. and Teh, Y.~W., editors, {\em Proceedings of the 34th International Conference on Machine Learning}, volume~70 of {\em Proceedings of Machine Learning Research}, pages 22--31. PMLR.

\bibitem[Barrera et~al., 2016]{barrera2016survey}
Barrera, L.~A., Vedenko, A., Kurland, J.~V., Rogers, J.~M., Gisselbrecht, S.~S., Rossin, E.~J., Woodard, J., Mariani, L., Kock, K.~H., Inukai, S., et~al. (2016).
\newblock Survey of variation in human transcription factors reveals prevalent dna binding changes.
\newblock {\em Science}, 351(6280):1450--1454.

\bibitem[Bengio et~al., 2021]{bengio2021flow}
Bengio, E., Jain, M., Korablyov, M., Precup, D., and Bengio, Y. (2021).
\newblock Flow network based generative models for non-iterative diverse candidate generation.
\newblock {\em Neural Information Processing Systems (NeurIPS)}.

\bibitem[Bengio et~al., 2023]{bengio2023gflownet}
Bengio, Y., Lahlou, S., Deleu, T., Hu, E.~J., Tiwari, M., and Bengio, E. (2023).
\newblock {GFlowNet} foundations.
\newblock {\em Journal of Machine Learning Research}, 24(210):1--55.

\bibitem[Black et~al., 2023]{black2023training}
Black, K., Janner, M., Du, Y., Kostrikov, I., and Levine, S. (2023).
\newblock Training diffusion models with reinforcement learning.
\newblock {\em arXiv preprint arXiv:2305.13301}.

\bibitem[Blessing et~al., 2024]{blessing2024beyond}
Blessing, D., Jia, X., Esslinger, J., Vargas, F., and Neumann, G. (2024).
\newblock Beyond {ELBOs}: A large-scale evaluation of variational methods for sampling.
\newblock {\em International Conference on Machine Learning (ICML)}.

\bibitem[Carter et~al., 2026]{carter2026discrete}
Carter, A., Choi, S., Tamogashev, K., Elvira, V., and Whitammer, E.~S. (2026).
\newblock Discrete diffusion samplers and bridges: Off-policy algorithms and applications in latent spaces.
\newblock {\em arXiv preprint arXiv:2602.05961}.

\bibitem[Cen et~al., 2022]{cen2022fast}
Cen, S., Cheng, C., Chen, Y., Wei, Y., and Chi, Y. (2022).
\newblock Fast global convergence of natural policy gradient methods with entropy regularization.
\newblock {\em Operations Research}, 70(4):2563--2578.

\bibitem[Chan et~al., 2022]{chan2022greedification}
Chan, A., Silva, H., Lim, S., Kozuno, T., Mahmood, A.~R., and White, M. (2022).
\newblock Greedification operators for policy optimization: Investigating forward and reverse kl divergences.
\newblock {\em Journal of Machine Learning Research}, 23(253):1--79.

\bibitem[Deleu et~al., 2024]{deleu2024discrete}
Deleu, T., Nouri, P., Malkin, N., Precup, D., and Bengio, Y. (2024).
\newblock Discrete probabilistic inference as control in multi-path environments.
\newblock {\em Proceedings of the Fortieth Conference on Uncertainty in Artificial Intelligence}, pages 997--1021.

\bibitem[Ding and Lawrence, 2003]{ding2003statistical}
Ding, Y. and Lawrence, C.~E. (2003).
\newblock A statistical sampling algorithm for rna secondary structure prediction.
\newblock {\em Nucleic acids research}, 31(24):7280--7301.

\bibitem[Fan et~al., 2023]{fan2023dpok}
Fan, Y., Watkins, O., Du, Y., Liu, H., Ryu, M., Boutilier, C., Abbeel, P., Ghavamzadeh, M., Lee, K., and Lee, K. (2023).
\newblock Dpok: Reinforcement learning for fine-tuning text-to-image diffusion models.
\newblock {\em Advances in Neural Information Processing Systems}, 36:79858--79885.

\bibitem[Geist et~al., 2019]{geist2019theory}
Geist, M., Scherrer, B., and Pietquin, O. (2019).
\newblock A theory of regularized {Markov} decision processes.
\newblock {\em International Conference on Machine Learning (ICML)}.

\bibitem[Green, 1995]{green1995reversible}
Green, P.~J. (1995).
\newblock Reversible jump markov chain monte carlo computation and bayesian model determination.
\newblock {\em Biometrika}, 82(4):711--732.

\bibitem[Griffiths and Steyvers, 2002]{griffiths2002probabilistic}
Griffiths, T.~L. and Steyvers, M. (2002).
\newblock A probabilistic approach to semantic representation.
\newblock In {\em Proceedings of the Annual Meeting of the Cognitive Science Society}, volume~24.

\bibitem[Gritsaev et~al., 2025]{gritsaev2025optimizing}
Gritsaev, T., Morozov, N., Samsonov, S., and Tiapkin, D. (2025).
\newblock Optimizing backward policies in {GF}lownets via trajectory likelihood maximization.
\newblock {\em International Conference on Learning Representations (ICLR)}.

\bibitem[Guo et~al., 2025]{guo2025deepseek}
Guo, D., Yang, D., Zhang, H., Song, J., Wang, P., Zhu, Q., Xu, R., Zhang, R., Ma, S., Bi, X., et~al. (2025).
\newblock Deepseek-r1: Incentivizing reasoning capability in llms via reinforcement learning.
\newblock {\em arXiv preprint arXiv:2501.12948}.

\bibitem[Haarnoja et~al., 2017]{haarnoja2017reinforcement}
Haarnoja, T., Tang, H., Abbeel, P., and Levine, S. (2017).
\newblock Reinforcement learning with deep energy-based policies.
\newblock In {\em International conference on machine learning}, pages 1352--1361. PMLR.

\bibitem[Haarnoja et~al., 2018]{haarnoja2018soft}
Haarnoja, T., Zhou, A., Abbeel, P., and Levine, S. (2018).
\newblock Soft actor-critic: Off-policy maximum entropy deep reinforcement learning with a stochastic actor.
\newblock In {\em International Conference on Machine Learning}, pages 1861--1870.

\bibitem[Ionides, 2008]{ionides2008truncated}
Ionides, E.~L. (2008).
\newblock Truncated importance sampling.
\newblock {\em Journal of Computational and Graphical Statistics}, 17(2):295--311.

\bibitem[Ising, 1925]{ising1925beitrag}
Ising, E. (1925).
\newblock Beitrag zur theorie des ferromagnetismus.
\newblock {\em Zeitschrift f{\"u}r Physik}, 31(1):253--258.

\bibitem[Jain et~al., 2023]{jain2023multi}
Jain, M., Raparthy, S.~C., Hern{\'a}ndez-Garc{\i}a, A., Rector-Brooks, J., Bengio, Y., Miret, S., and Bengio, E. (2023).
\newblock Multi-objective gflownets.
\newblock In {\em International conference on machine learning}, pages 14631--14653. PMLR.

\bibitem[Kakade and Langford, 2002]{kakade2002approximately}
Kakade, S. and Langford, J. (2002).
\newblock Approximately optimal approximate reinforcement learning.
\newblock In {\em Proceedings of the nineteenth international conference on machine learning}, pages 267--274.

\bibitem[Kostenetskiy et~al., 2021]{kostenetskiy2021hpc}
Kostenetskiy, P., Chulkevich, R., and Kozyrev, V. (2021).
\newblock {HPC} resources of the {Higher School of Economics}.
\newblock {\em Journal of Physics: Conference Series}, 1740:012050.

\bibitem[Lahlou et~al., 2023]{lahlou2023theory}
Lahlou, S., Deleu, T., Lemos, P., Zhang, D., Volokhova, A., Hern{\'a}ndez-Garc{\i}a, A., Ezzine, L.~N., Bengio, Y., and Malkin, N. (2023).
\newblock A theory of continuous generative flow networks.
\newblock {\em International Conference on Machine Learning (ICML)}.

\bibitem[Lan, 2023]{lan2023policy}
Lan, G. (2023).
\newblock Policy mirror descent for reinforcement learning: Linear convergence, new sampling complexity, and generalized problem classes.
\newblock {\em Mathematical Programming}, 198:1059--1106.

\bibitem[Madan et~al., 2023]{madan2023learning}
Madan, K., Rector-Brooks, J., Korablyov, M., Bengio, E., Jain, M., Nica, A.~C., Bosc, T., Bengio, Y., and Malkin, N. (2023).
\newblock Learning {GFlowNets} from partial episodes for improved convergence and stability.
\newblock {\em International Conference on Machine Learning (ICML)}.

\bibitem[Madigan et~al., 1995]{madigan1995bayesian}
Madigan, D., York, J., and Allard, D. (1995).
\newblock Bayesian graphical models for discrete data.
\newblock {\em International statistical review/revue internationale de statistique}, pages 215--232.

\bibitem[Malkin et~al., 2022]{malkin2022trajectory}
Malkin, N., Jain, M., Bengio, E., Sun, C., and Bengio, Y. (2022).
\newblock Trajectory balance: Improved credit assignment in {GFlowNets}.
\newblock {\em Neural Information Processing Systems (NeurIPS)}.

\bibitem[Malkin et~al., 2023]{malkin2023gflownets}
Malkin, N., Lahlou, S., Deleu, T., Ji, X., Hu, E.~J., Everett, K.~E., Zhang, D., and Bengio, Y. (2023).
\newblock {GFlowNets} and variational inference.
\newblock {\em International Conference on Learning Representations (ICLR)}.

\bibitem[Miao et~al., 2019]{miao2019cgmh}
Miao, N., Zhou, H., Mou, L., Yan, R., and Li, L. (2019).
\newblock Cgmh: Constrained sentence generation by metropolis-hastings sampling.
\newblock In {\em Proceedings of the AAAI Conference on Artificial Intelligence}, volume~33, pages 6834--6842.

\bibitem[Munos et~al., 2016]{munos2016safe}
Munos, R., Stepleton, T., Harutyunyan, A., and Bellemare, M.~G. (2016).
\newblock Safe and efficient off-policy reinforcement learning.
\newblock In {\em Advances in Neural Information Processing Systems}, volume~29.

\bibitem[Nachum et~al., 2017]{nachum2017bridging}
Nachum, O., Norouzi, M., Xu, K., and Schuurmans, D. (2017).
\newblock Bridging the gap between value and policy based reinforcement learning.
\newblock {\em Neural Information Processing Systems (NeurIPS)}.

\bibitem[Neu et~al., 2017]{neu2017unified}
Neu, G., Jonsson, A., and G{\'o}mez, V. (2017).
\newblock A unified view of entropy-regularized {Markov} decision processes.
\newblock {\em arXiv preprint arXiv:1705.07798}.

\bibitem[Niu et~al., 2024]{pmlr-v235-niu24c}
Niu, P., Wu, S., Fan, M., and Qian, X. (2024).
\newblock {GF}low{N}et training by policy gradients.
\newblock In Salakhutdinov, R., Kolter, Z., Heller, K., Weller, A., Oliver, N., Scarlett, J., and Berkenkamp, F., editors, {\em Proceedings of the 41st International Conference on Machine Learning}, volume 235 of {\em Proceedings of Machine Learning Research}, pages 38344--38380. PMLR.

\bibitem[Niu et~al., 2026]{niu2026evaluating}
Niu, P., Wu, S., and Qian, X. (2026).
\newblock Evaluating {GF}lownet from partial episodes for stable and flexible policy-based training.
\newblock In {\em The Fourteenth International Conference on Learning Representations}.

\bibitem[Ouyang et~al., 2022]{ouyang2022training}
Ouyang, L., Wu, J., Jiang, X., Almeida, D., Wainwright, C., Mishkin, P., Zhang, C., Agarwal, S., Slama, K., Ray, A., et~al. (2022).
\newblock Training language models to follow instructions with human feedback.
\newblock {\em Advances in neural information processing systems}, 35:27730--27744.

\bibitem[Potts, 1952]{potts1952some}
Potts, R.~B. (1952).
\newblock Some generalized order-disorder transformations.
\newblock In {\em Mathematical proceedings of the cambridge philosophical society}, volume~48, pages 106--109. Cambridge University Press.

\bibitem[Ramakrishnan et~al., 2014]{ramakrishnan2014quantum}
Ramakrishnan, R., Dral, P.~O., Rupp, M., and Von~Lilienfeld, O.~A. (2014).
\newblock Quantum chemistry structures and properties of 134 kilo molecules.
\newblock {\em Scientific data}, 1(1):1--7.

\bibitem[Richter et~al., 2020]{richter2020vargrad}
Richter, L., Boustati, A., N{\"u}sken, N., Ruiz, F., and Akyildiz, O.~D. (2020).
\newblock {VarGrad}: a low-variance gradient estimator for variational inference.
\newblock {\em Neural Information Processing Systems (NeurIPS)}.

\bibitem[Schulman et~al., 2017a]{schulman2017equivalence}
Schulman, J., Chen, X., and Abbeel, P. (2017a).
\newblock Equivalence between policy gradients and soft q-learning.
\newblock {\em arXiv preprint arXiv:1704.06440}.

\bibitem[Schulman et~al., 2015a]{schulman2015trust}
Schulman, J., Levine, S., Abbeel, P., Jordan, M., and Moritz, P. (2015a).
\newblock Trust region policy optimization.
\newblock In {\em International Conference on Machine Learning}, pages 1889--1897. PMLR.

\bibitem[Schulman et~al., 2015b]{Schulman2015HighDimensionalCC}
Schulman, J., Moritz, P., Levine, S., Jordan, M.~I., and Abbeel, P. (2015b).
\newblock High-dimensional continuous control using generalized advantage estimation.
\newblock {\em CoRR}, abs/1506.02438.

\bibitem[Schulman et~al., 2017b]{schulman2017proximal}
Schulman, J., Wolski, F., Dhariwal, P., Radford, A., and Klimov, O. (2017b).
\newblock Proximal policy optimization algorithms.
\newblock {\em arXiv preprint arXiv:1707.06347}.

\bibitem[Sendera et~al., 2024]{sendera2024improved}
Sendera, M., Kim, M., Mittal, S., Lemos, P., Scimeca, L., Rector-Brooks, J., Adam, A., Bengio, Y., and Malkin, N. (2024).
\newblock Improved off-policy training of diffusion samplers.
\newblock {\em Neural Information Processing Systems (NeurIPS)}.

\bibitem[Shani et~al., 2020]{shani2020adaptive}
Shani, L., Efroni, Y., and Mannor, S. (2020).
\newblock Adaptive trust region policy optimization: Global convergence and faster rates for regularized {MDP}s.
\newblock In {\em Proceedings of the AAAI Conference on Artificial Intelligence}, volume~34, pages 5668--5675.

\bibitem[Shen et~al., 2023]{shen2023towards}
Shen, M.~W., Bengio, E., Hajiramezanali, E., Loukas, A., Cho, K., and Biancalani, T. (2023).
\newblock Towards understanding and improving gflownet training.
\newblock In {\em International conference on machine learning}, pages 30956--30975. PMLR.

\bibitem[Sutton et~al., 1999]{sutton1999policy}
Sutton, R.~S., McAllester, D., Singh, S., and Mansour, Y. (1999).
\newblock Policy gradient methods for reinforcement learning with function approximation.
\newblock In {\em Advances in Neural Information Processing Systems}, volume~12.

\bibitem[Tiapkin et~al., 2025]{tiapkin2025gfnx}
Tiapkin, D., Agarkov, A., Morozov, N., Maksimov, I., Tsyganov, A., Gritsaev, T., and Samsonov, S. (2025).
\newblock gfnx: Fast and scalable library for generative flow networks in jax.
\newblock {\em arXiv preprint arXiv:2511.16592}.

\bibitem[Tiapkin et~al., 2024]{tiapkin2024generative}
Tiapkin, D., Morozov, N., Naumov, A., and Vetrov, D.~P. (2024).
\newblock Generative flow networks as entropy-regularized rl.
\newblock In {\em International Conference on Artificial Intelligence and Statistics}, pages 4213--4221. PMLR.

\bibitem[Williams, 1992]{williams1992simple}
Williams, R.~J. (1992).
\newblock Simple statistical gradient-following algorithms for connectionist reinforcement learning.
\newblock {\em Machine learning}, 8:229--256.

\bibitem[Williams and Peng, 1991]{williams1991function}
Williams, R.~J. and Peng, J. (1991).
\newblock Function optimization using connectionist reinforcement learning algorithms.
\newblock {\em Connection Science}, 3(3):241--268.

\bibitem[Wolff, 1989]{wolff1989collective}
Wolff, U. (1989).
\newblock Collective monte carlo updating for spin systems.
\newblock {\em Physical Review Letters}, 62(4):361.

\bibitem[Yang and Rannala, 1997]{yang1997bayesian}
Yang, Z. and Rannala, B. (1997).
\newblock Bayesian phylogenetic inference using dna sequences: a markov chain monte carlo method.
\newblock {\em Molecular biology and evolution}, 14(7):717--724.

\bibitem[Yu et~al., 2025]{yu2025dapo}
Yu, Q., Zhang, Z., Zhu, R., Yuan, Y., Zuo, X., Yue, Y., Dai, W., Fan, T., Liu, G., Liu, L., et~al. (2025).
\newblock Dapo: An open-source llm reinforcement learning system at scale, 2025.
\newblock {\em URL https://arxiv. org/abs/2503.14476}, 1:2.

\bibitem[Zhan et~al., 2023]{zhan2023policy}
Zhan, W., Cen, S., Huang, B., Chen, Y., Lee, J.~D., and Chi, Y. (2023).
\newblock Policy mirror descent for regularized reinforcement learning: A generalized framework with linear convergence.
\newblock {\em SIAM Journal on Optimization}, 33(2):1061--1091.

\bibitem[Zhang et~al., 2023]{zhang2023robust}
Zhang, D., Rainone, C., Peschl, M., and Bondesan, R. (2023).
\newblock Robust scheduling with {GFlowNets}.
\newblock {\em International Conference on Learning Representations (ICLR)}.

\end{thebibliography}

\clearpage
\appendix
\thispagestyle{empty}

% Supplementary material: To improve readability, you must use a single-column format for the supplementary material.
\onecolumn
\aistatstitle{Proximal Policy Optimization for Amortized Discrete Sampling: \\
Supplementary Materials}
\section{Proximal Policy Optimization}\label{app:ppo}

PPO~\citep{schulman2017proximal} is an alternative to Vanilla Policy Gradient (VPG) with markedly better stability and sample efficiency. Like TRPO~\citep{schulman2015trust}, it can be derived 
as a parametric, gradient-based approximation of an underlying policy improvement scheme. We construct it in three steps: (i) the exact policy iteration (PI) operator, (ii) an off-policy, importance-weighted surrogate of its objective, and (iii) a trust-region correction in the spirit of conservative policy iteration~\citep[CPI;][]{kakade2002approximately}.

\textit{Policy iteration as per-state maximization.}
The PI operator returns, at every state $s$,
\[
    \textstyle 
    \pi_{\mathrm{PI}}(\cdot \mid s) \in 
    \arg\max_{\pi(\cdot\mid s)\in\Delta(\cA_s)}\;
    \sum_{a \in \cA_s} \pi(a \mid s)\, Q^{\pi_{\text{old}}}(s,a),
\]
mirroring the classical greedy update. Subtracting the constant 
baseline $V^{\pi_{\text{old}}}(s)$ does not change the maximizer 
but reduces the variance of sample-based estimates; aggregating 
along trajectories drawn from $\pi_{\text{old}}$ yields the 
surrogate
\begin{equation}\label{eq:pi-surrogate}
    \textstyle 
    \cL(\theta; \pi_{\text{old}}) \triangleq 
    \E_{\tau \sim \pi_{\text{old}}}\!\left[ \sum_{t=0}^{T-1}
    \sum_{a' \in \cA_{s_t}} \pi_\theta(a'\mid s_t)\, 
    A^{\pi_{\text{old}}}(s_t,a') \right]\!.
\end{equation}
When $\pi_{\text{old}}$ visits every state, the maximizers of $\cL$ 
coincide with the PI update; more generally, $\cL$ serves as a 
first-order surrogate for $J(\pi_\theta) - J(\pi_{\text{old}})$ 
around $\pi_{\text{old}}$~\citep{kakade2002approximately,
schulman2015trust}, which is what justifies treating it as a 
stand-in for $J$ during off-policy optimization.

\textit{Unclipped surrogate via importance sampling.}
The inner sum over $\cA_{s_t}$ is impractical for large action 
spaces. A single-sample importance estimator yields the equivalent 
form
\begin{equation}\label{eq:unclipped-surrogate}
    \textstyle \cL(\theta; \pi_{\text{old}}) = 
    \E_{\tau \sim \pi_{\text{old}}}\!\left[ \sum_{t=0}^{T-1}
    \rho_t(\theta)\, A^{\pi_{\text{old}}}_t \right]\!,
\end{equation}
where $\rho_t(\theta) \triangleq \pi_\theta(a_t \mid s_t) / 
\pi_{\text{old}}(a_t \mid s_t)$ and $A^{\pi_{\text{old}}}_t 
\triangleq A^{\pi_{\text{old}}}(s_t,a_t)$. This form requires only 
rollouts from $\pi_{\text{old}}$ together with an estimate of 
$A^{\pi_{\text{old}}}$.

\textit{From CPI to a clipped trust region.}
Maximizing~\eqref{eq:unclipped-surrogate} without restriction is 
unsafe: the surrogate is only a good proxy for $J$ in a 
neighborhood of $\pi_{\text{old}}$, and unbounded ratios 
$\rho_t(\theta)$ produce high-variance gradient estimates. CPI 
controls this by mixing the PI target with $\pi_{\text{old}}$; 
TRPO enforces an explicit KL trust region. PPO produces a similar 
effect with a cheaper, asymmetric mechanism inspired by 
\emph{truncated importance sampling}~\citep{ionides2008truncated,
munos2016safe}: the contribution of any transition whose ratio 
strays outside $[1-\epsilon,\,1+\epsilon]$ in the rewarding 
direction is capped,
\begin{align}\notag
    \textstyle \cL_{\mathrm{PPO}}(\theta; \pi_{\text{old}}) &\triangleq 
    \E_{\tau \sim \pi_{\text{old}}}\!\left[\sum_{t=0}^{T-1}
    \mathrm{PPOClip}\bigl(\rho_t(\theta),\, A^{\pi_{\text{old}}}_t
    \bigr)\right]\!, \\
    \mathrm{PPOClip}(\rho, A) &\triangleq 
    \min\!\bigl( \rho A,\; \mathrm{clip}(\rho,1{-}\epsilon,
    1{+}\epsilon)\,A \bigr)\,.\label{eq:ppo_clip_repeat}
\end{align}
The pessimistic $\min$ removes the incentive to push $\rho_t$ 
further from $1$ once the surrogate becomes untrustworthy, while 
leaving the gradient unattenuated in the safe regime.

\newpage

\section{Extended discussion of considered problems, evaluation protocol, and experimental results}\label{app:extended}

\subsection{Environments}\label{app:exp_envs}

\textbf{Hypergrid}~\citep{bengio2021flow} is a synthetic $d$-dimensional 
discrete hypercube of side length $H$. There are $2^d$ regions with high rewards near the corners of the grid, while the states outside have much lower rewards. The reward of a terminal state $s$ is defined as
\begin{equation*}
R(s) = R_0 + R_1 \prod_{i = 1}^d \mathbb{I}\left[0.25 < \left|\frac{s^i}{H-1}-0.5\right|\right] + R_2 \prod_{i = 1}^d \mathbb{I}\left[0.3 < \left|\frac{s^i}{H-1}-0.5\right| < 0.4\right].
\end{equation*}
We use the standard reward parameters $R_0 = 10^{-3}$, $R_1 = 0.5$, $R_2 = 2$ following \citep{bengio2021flow, madan2023learning}.  We use $d = 4$, $H = 20$, giving $|\mathcal{X}| = 160\,000$ terminal states.

\textbf{TFBind8}~\citep{shen2023towards} is a DNA sequence design task: the goal is to generate length-8 sequences over $\{A, C, G, T\}$. The reward is the measured binding activity of the generated sequence to the human transcription factor SIX6~\citep{barrera2016survey}, giving $|\mathcal{X}| = 4^8 = 65\,536$ terminal states out of $87\,381$ states in total.

\textbf{String QM9} is a small-molecule generation environment. We use the prepend/append formulation~\citep{shen2023towards} with 11 building blocks and trajectories of length up to 5, giving $|\mathcal{X}| = 11^5 = 161\,051$ terminal states out of $177\,156$ in total. The reward is predicted by a proxy network trained on the QM9 dataset~\cite{ramakrishnan2014quantum} to estimate the HOMO-LUMO gap, a molecular property relevant for chemical reactivity.

\textbf{Ising model}~\citep{ising1925beitrag} is a model of interacting binary spins. We consider a ferromagnetic $16\times16$ lattice with periodic boundary conditions, unit nearest-neighbor coupling, and zero external field, giving $|\mathcal{X}|=2^{256}$ configurations. The target distribution is proportional to $R(s)=\exp\bigl(\beta\sum_{\langle i,j\rangle}s_i s_j\bigr)$, where $s_i\in\{-1,+1\}$, each neighboring pair is counted once, and $\beta\in\{0.4407,0.6\}$ is the inverse temperature. A larger $\beta$ leads to a sharper distribution. Our GFlowNet fills one eight-spin token per step, jointly selecting its value and an unfilled position, in any order over 32 steps. Figure~\ref{fig:ising_reference_samples} shows the reference samples for $\beta=0.4407$ and $\beta=0.6$, respectively.

\textbf{sEH}~\citep{bengio2021flow} is a fragment-based molecular-graph generation task using a junction-tree representation. Molecules are generated as trees of molecular fragments by sequentially adding fragments from the 72-fragment vocabulary of \citep{bengio2021flow} and specifying chemically valid attachment points between adjacent fragments. This induces a large combinatorial state space: the original formulation estimates up to $10^{16}$ states, with roughly $100$--$2000$ available actions per state depending on the number of available stems. The environment includes a stop action and masks invalid attachments. We use a maximum of 9 fragments. 
% Since the same final molecular graph can be reached through multiple fragment-addition and attachment-order sequences, the resulting state space is naturally a DAG rather than a tree. 
Rewards are given by the pretrained MPNN proxy from \citep{bengio2021flow}, which predicts sEH binding affinity for the generated molecule.

\textbf{QM9}~\citep{jain2023multi} is a molecular-graph generation task based on the QM9 dataset~\citep{ramakrishnan2014quantum}. Molecules are constructed atom-by-atom and bond-by-bond with up to 9 heavy atoms from the alphabet $\{\mathrm{C}, \mathrm{N}, \mathrm{F}, \mathrm{O}\}$. At each step the model can add an atom, add a bond between existing atoms, set node or bond attributes such as charge, explicit hydrogens, chirality, or bond order, or terminate the trajectory. The induced MDP is combinatorially large: a trajectory can contain up to roughly 128 construction steps, including atom additions, bond additions, attribute assignments, and termination, and the number of candidate actions in a 9-heavy-atom state is on the order of a few hundred before chemical-validity masks are applied. The reward is the same as for String QM9.
% given by a pretrained MXMNet~\citep{zhang2020molecular} proxy trained on QM9 to estimate the HOMO--LUMO gap, a molecular property relevant for chemical reactivity.

\begin{figure}[H]
    \centering
    \begin{minipage}[c]{0.5\linewidth}
        \includegraphics[width=\linewidth]{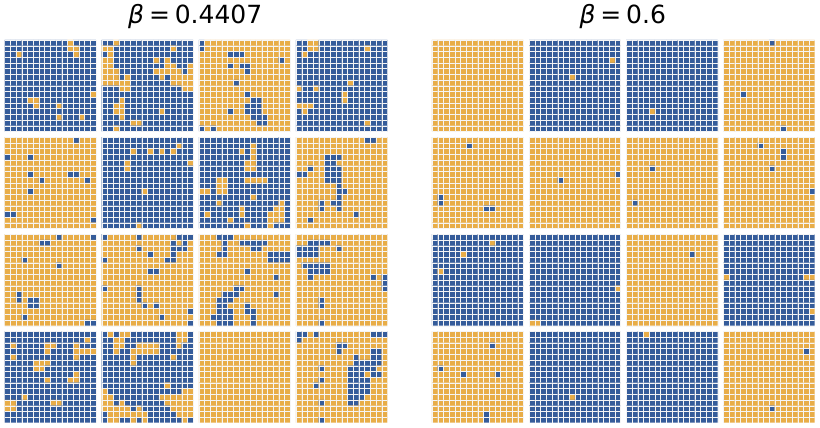}
    \end{minipage}\hfill
    \begin{minipage}[c]{0.46\linewidth}
    \caption{Ising model reference configurations on a $16\times16$ periodic lattice: a $4\times4$ grid of samples each at $\beta=0.4407$ (left) and $\beta=0.6$ (right), from the long-run Swendsen--Wang pools used for evaluation. Blue and yellow denote spins $-1$ and $+1$, respectively.}
    \label{fig:ising_reference_samples}
    \end{minipage}
\end{figure}

\subsection{Evaluation} 

We use 5 seeds for experiments on QM9 (\cref{fig:qm9_fixedpb_combined_k_panels}) and GAE sensitivity study (\cref{fig:GAE_boxpolts}), and other experiments are conducted over 3 random seeds. Figures report the mean, with shaded regions showing the min--max range. Hyperparameters for all methods, including the off-policy baselines DB, TB, VarGrad, and SubTB and on-policy TRPO, are tuned, and full ranges and selected values are reported in~\cref{app:hyperparameters}.

\textbf{Small problems.} For small experiments (Hypergrid, TFBind8, and String QM9), the exact target distribution $\cR(x)/\rmZ$ is tractable. We evaluate sampling quality using the exact total variation (TV) distance between the ground truth and the distribution induced by the forward policy. For ablations we summarise each training run by the area under the TV curve (AUC). AUC captures both convergence speed and final TV in a single scalar. 

\textbf{Molecular-graph generation.} For large experiments, the state space is too large for exact TV distance evaluation. Following~\cite{blessing2024beyond}, we instead report the Evidence Lower Bound (ELBO) and Evidence Upper Bound (EUBO) on $\log \rmZ$, defined as:
\begin{align*} 
    \mathrm{ELBO} &= \E_{\pi_\theta} \left[ \log \frac{\cR(s_T) \cdot
    \prod_{t=0}^{T-1} \PB( s_{t} \mid s_{t+1} )}{
    \prod_{t=0}^{T-1} \pi_\theta( s_{t+1} \mid s_t ) } \right]
    \leq \log \rmZ \\
    &\leq
     \E_{s_T \sim \cR(x)/\rmZ,\, \tau \sim \PB( \cdot \mid s_T)} \left[ \log \frac{\cR(s_T) \cdot \prod_{t=0}^{T-1} \PB( s_{t} \mid s_{t+1} )}{
    \prod_{t=0}^{T-1} \pi_\theta( s_{t+1} \mid s_t ) } \right]
    = \mathrm{EUBO}
\end{align*}
ELBO measures intra-mode coverage, as trajectories are sampled from $\pi_\theta$. EUBO measures mode coverage by sampling terminal states from the target distribution $\cR(x)/\rmZ$. For molecular tasks, exact target samples are unavailable. We instead use a fixed evaluation dataset $\mathcal{D}$ and reweight its molecules by reward, giving the importance-weighted proxy
\begin{align*}
    \mathrm{Dataset\text{-}based\text{ }EUBO}
    &= \E_{s_T\sim\mathrm{Unif}(\mathcal{D}),\,\tau\sim\PB(\cdot\mid s_T)}
    \left[
    \frac{\cR(s_T)}{\frac{1}{|\mathcal{D}|}\sum_{x\in\mathcal{D}}\cR(x)}
    \log\frac{\cR(s_T)\cdot\prod_{t=0}^{T-1}\PB(s_t\mid s_{t+1})}
    {\prod_{t=0}^{T-1}\pi_\theta(s_{t+1}\mid s_t)}
    \right].
\end{align*}
We estimate the expectation using backward trajectories from the evaluation molecules, with reward weights normalized over the full dataset. It measures coverage of molecules in $\mathcal{D}$, but is not guaranteed to upper-bound the $\log\rmZ$. For sEH, we use the molecular test set of~\citep{bengio2021flow}. We retain the 893 of 950 molecules that can be converted to our fragment-based graph representation. For QM9, we use the held-out 10\% split of the QM9 dataset~\citep{ramakrishnan2014quantum}.

These metrics are preferable to the mode-count and top-$k$ reward statistics, which measure optimization quality rather than distributional fidelity. A policy that concentrates mass on a few high-reward modes can score well on top-$k$ metrics while badly mismatching the target $\cR(x)/\rmZ$. The ideal ELBO and EUBO instead directly bound $\log \rmZ$. Crucially, the two bounds are complementary: ELBO is insensitive to uncovered modes (since $\pi_\theta$ never samples them), while EUBO is insensitive to over-represented ones (since sampling is from the target). Reporting both provides a two-sided diagnostic of distributional accuracy.

\textbf{Ising model.} We report ELBO and EUBO using the definitions above, together with magnetization and two-point correlation errors following~\citet{carter2026discrete}. For each inverse temperature, we generate a fixed pool of 16,384 configurations generated by the Swendsen–Wang algorithm as in~\citep{carter2026discrete}. These approximate target samples are used to estimate EUBO and the reference correlations.

Magnetization error measures whether the sampler reproduces the symmetry between positive and negative spins. Let $m^{\mathrm{row}}(k)$ and $m^{\mathrm{col}}(k)$ be the mean spins averaged over samples and sites in row or column $k$. Since the target magnetization is zero in the absence of an external field, we compute
\[
    e_{\mathrm{mag}}=\frac{1}{2L}\sum_{k=1}^{L}
    \left(\left|m^{\mathrm{row}}(k)\right|+\left|m^{\mathrm{col}}(k)\right|\right).
\]
Correlation error measures how well spatial dependencies between spins are reproduced. We first compute the connected correlation $C(i,j)=\mathbb{E}[s_i s_j]-\mathbb{E}[s_i]\mathbb{E}[s_j]$ and average it over all site pairs separated by periodic distance $r$ in each lattice direction, obtaining $C^{\mathrm{row}}(r)$ and $C^{\mathrm{col}}(r)$. We report
\[
    e_{\mathrm{corr}}=\frac{1}{2L}\sum_{r=0}^{L-1}
    \left(\left|C^{\mathrm{row}}(r)-C^{\mathrm{row}}_{\mathrm{ref}}(r)\right|
    +\left|C^{\mathrm{col}}(r)-C^{\mathrm{col}}_{\mathrm{ref}}(r)\right|\right),
\]
where the reference correlations are estimated from the Swendsen--Wang samples and $L=16$. Lower values are better for both errors.

\subsection{Extended discussion of experimental results}
\label{app:extended_discussion_experiments}

\paragraph{SubEB-GAE.} We aim to reproduce the algorithm presented in the original work~\citep{niu2026evaluating}. The paper does not discuss mini-batch splits or multiple epochs for value learning, hence we set $E$ to 1 and $S$ to 1 for these experiments. However, we find these techniques useful for GAE, which shows the potential of SubEB-GAE.

\paragraph{Trajectory Likelihood Maximization enables Ent-PPO backward policy learning.} Since $\PB$ defines the reward~\eqref{eq:sampling_reward}, Ent-PPO keeps it fixed within each policy-improvement step and, in contrast to the off-policy baselines, does not learn it through its own objective~\citep{tiapkin2024generative}. We therefore learn $\PB$ with a separate loss, Trajectory Likelihood Maximization~\citep[TLM]{gritsaev2025optimizing}, defined as
\begin{equation}\label{eq:tlm}
    \mathrm{TLM} = - \E_{\tau \sim \PF} \left[ \log \prod_{t=0}^{T-1} \PB(s_t \mid s_{t+1}, \phi) \right]
    \qquad
    \nabla_{\phi} \mathrm{TLM} = \nabla_{\phi} \KL (\PF \| \PB),
\end{equation}
was proposed as a simple yet effective loss that enables backward policy learning for the RL-based algorithms of~\citep{tiapkin2024generative}. With the stabilization tricks of~\citep{gritsaev2025optimizing}, it also outperforms naive backward policy learning within a single off-policy objective, and can be used as backward learning strategy with DB, TB, and SubTB. 

\cref{fig:EntPPO_TLM_vs_baselines} presents results for off-policy baselines and Ent-PPO with learned $\PB$, where the off-policy baselines train the backward policy jointly and Ent-PPO uses TLM. Ent-PPO converges faster than the off-policy baselines.
% Note that we suggest a slight modification of TLM presented in the next paragraph.
For Ent-PPO, we consider two backward-policy learning schemes, and both outperform Ent-PPO with fixed $\PB$ (\cref{fig:fixed_vs_learned_pb}, right). The first applies a single TLM update per iteration. The second performs $K$ TLM updates per iteration on a freshly sampled rollout, equalizing the number of gradient steps taken by the forward and backward policies. The second scheme is less theoretically grounded, since TLM is motivated as a stochastic estimate of $\nabla_{\phi} \KL (\PF \| \PB)$, but achieves better results in practice.

\begin{figure}[H]
    \centering
    \begin{minipage}[c]{0.5\linewidth}
        \includegraphics[width=\linewidth]{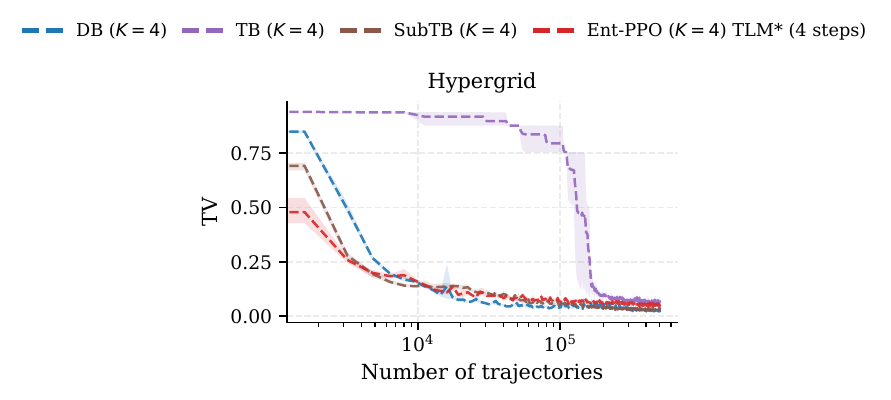}
    \end{minipage}\hfill
    \begin{minipage}[c]{0.46\linewidth}
    \caption{Ent-PPO with TLM-based backward policy learning compared against off-policy baselines (DB, TB, SubTB) that learn backward policy jointly within a single objective on Hypergrid. $K=4$ for all methods. Ent-PPO converges faster than the off-policy baselines. Curves show TV distance vs.\ the number of reward evaluations. Lower is better. Lines show the mean over 3 seeds. Shaded regions show the min--max range.}
    \label{fig:EntPPO_TLM_vs_baselines}
    \end{minipage}
\end{figure}

\newpage
\paragraph{Backward policy learning can improve performance.} \cref{fig:fixed_vs_learned_pb} compares algorithms with fixed $\PB$ against the same algorithms with learned $\PB$. Learning the backward policy speeds up convergence, a well-known result from~\citep{malkin2022trajectory, gritsaev2025optimizing}. We do not include similar experiments for TFBind8, since this is an autoregressive environment with a fixed backward policy. For String QM9, we do not observe a difference between off-policy methods with learned and fixed backward policies. The same results are found for Ent-PPO. We hypothesize this happens due to the limited expressivity of the backward policy in this task: it can only choose from which side, left or right, to erase a symbol.

\begin{figure}[H]
    \centering
    \includegraphics[width=1\linewidth]{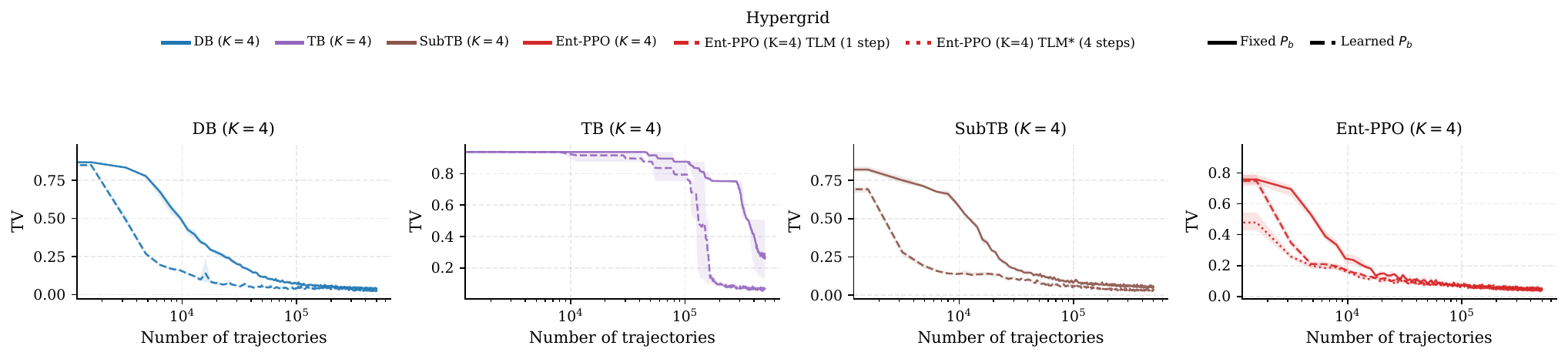}
    \caption{Algorithms with fixed uniform backward policy compared against the same algorithms with learned backward policy on Hypergrid. Off-policy baselines learn $\PB$ jointly with their main objective. Ent-PPO uses TLM, and the right panel shows both Ent-PPO backward-policy learning schemes discussed in \cref{app:extended_discussion_experiments}. Learned backward policy consistently speeds up convergence. Curves show TV distance vs.\ the number of reward evaluations. Lower is better. Lines show the mean over 3 seeds. Shaded regions show the min--max range.}
    \label{fig:fixed_vs_learned_pb}
\end{figure}

\paragraph{QM9.} \cref{fig:qm9_fixedpb_combined_k_panels} compares Ent-PPO against GFlowNet baselines on QM9, the largest problem considered. DB and SubTB perform poorly, as reflected in their ELBO values. While VarGrad and TB match Ent-PPO in ELBO, their EUBO values suggest mode collapse. TRPO and Ent-PPO achieve the best results on this challenging problem. Considering both ELBO and EUBO, the performance gap between the best Ent-PPO configuration and the best off-policy baseline is substantially larger here than on the other problems: Hypergrid, the small string problems, the Ising model, and sEH. We attribute this to the larger scale of the problem and hypothesize that policy gradient algorithms are better suited to particularly complex problems than off-policy methods.

\begin{figure}[H]
    \centering $\scriptscriptstyle{\mathrm{QM9}}$
    % \centering $\mathrm{sEH}$
    \centering
    \includegraphics[width=1\linewidth]{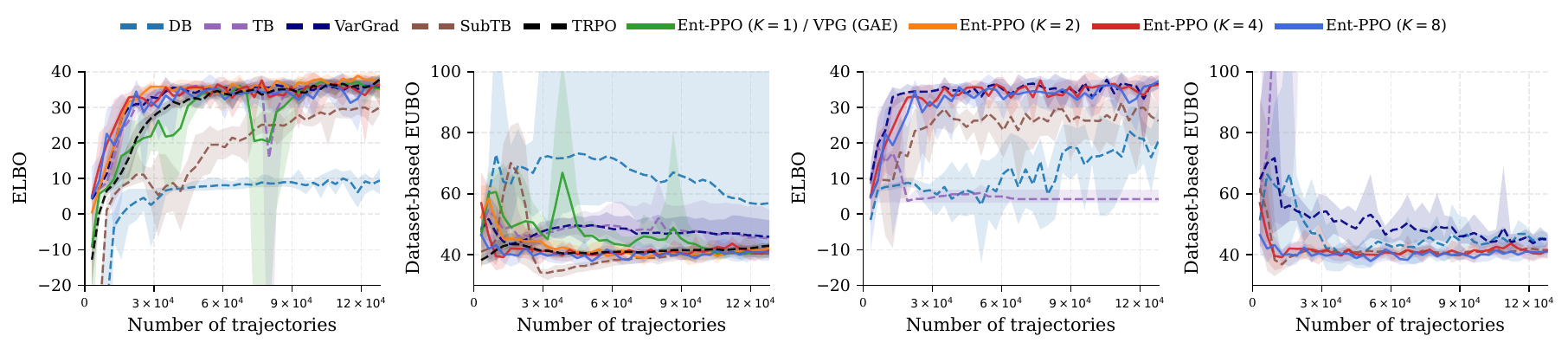}
    \caption{Ent-PPO compared against other GFlowNet baselines: DB, TB, VarGrad, SubTB, and TRPO on QM9. Left: baselines use a single update per batch, while Ent-PPO varies $K \in \{1, 2, 4, 8\}$ update epochs. Right: baselines use $K=4$ updates per batch, while Ent-PPO varies $K \in \{4, 8\}$ update epochs. Odd panels show ELBO vs.\ the number of reward evaluations (higher is better). Even panels show Dataset-based EUBO vs.\ the number of reward evaluations (lower is better). Lines show the mean over 5 seeds. Shaded regions show the min--max range.}
    \label{fig:qm9_fixedpb_combined_k_panels}
\end{figure}

\paragraph{Ising model.} \cref{tab:ising_results} reports results for three sampling paradigms (MCMC, discrete diffusion, and GFlowNets) and various methods within these paradigms. \cref{fig:ising_reference_samples} visualises samples generated by GFlowNets methods.

First, the best reported results for discrete diffusion samplers and GFlowNets are numerically close. However, the training budgets are not matched across studies, and the experimental setups also differ. These results therefore provide context rather than a controlled comparison: they do not establish relative performance at a common training budget, relative sample efficiency, and neural network architecture.

Second, among the GFlowNet approaches, Ent-PPO and TB achieve the tightest ELBO--EUBO gaps at both inverse temperatures, while DB and SubTB perform poorly. VarGrad and TRPO achieve similar performance at $\beta=0.6$, but have larger bound gaps and correlation errors at $\beta=0.4407$.

\begin{table}[H]
    \centering
    \caption{\textbf{Result of Ising sampling on a $16\times16$ periodic lattice}. The results for MCMC and discrete diffusion are taken from Table~1 of \citet{carter2026discrete}; see the work for more details. \textbf{Training budgets are not matched across studies, so these results do not establish relative performance at a common training budget or relative sample efficiency.} Bold and underlining mark the best and second-best distinct means, respectively, for each metric.}
    \label{tab:ising_results}
    \small
    \setlength{\tabcolsep}{3pt}
    \begin{tabular}{p{3.3cm}lrrrr}
        \toprule
        Paradigm & Training objective & ELBO $\uparrow$ & EUBO $\downarrow$ & Corr. error $\downarrow$ & Mag. error $\downarrow$ \\
        \midrule
        \multicolumn{6}{c}{$\beta=0.4407$} \\
        \midrule
        \raggedright MCMC (Metropolis Hastings) & --- & --- & --- & $0.140\,{\scriptstyle\pm 0.010}$ & $0.040\,{\scriptstyle\pm 0.010}$  \\
        \noalign{\vskip 2pt}\hdashline\noalign{\vskip 2pt}
        \multirow[t]{7}{3.3cm}{\raggedright Discrete diffusion samplers} & WDCE (MDNS) & $237.99\,{\scriptstyle\pm 0.49}$ & $245.66\,{\scriptstyle\pm 8.30}$ & $0.230\,{\scriptstyle\pm 0.260}$ & $0.320\,{\scriptstyle\pm 0.350}$  \\
         & LV (on-policy) & $237.89\,{\scriptstyle\pm 0.16}$ & $242.32\,{\scriptstyle\pm 4.20}$ & $0.430\,{\scriptstyle\pm 0.100}$ & $0.670\,{\scriptstyle\pm 0.100}$  \\
         & LV + Buffer & $238.21\,{\scriptstyle\pm 0.24}$ & $239.00\,{\scriptstyle\pm 0.18}$ & $0.070\,{\scriptstyle\pm 0.020}$ & $0.030\,{\scriptstyle\pm 0.020}$  \\
         & LV + Buffer + MCMC & $238.28\,{\scriptstyle\pm 0.04}$ & $238.97\,{\scriptstyle\pm 0.03}$ & $0.060\,{\scriptstyle\pm 0.010}$ & $0.120\,{\scriptstyle\pm 0.070}$  \\
         & TB (on-policy) & $237.64\,{\scriptstyle\pm 0.23}$ & $249.26\,{\scriptstyle\pm 9.65}$ & $0.550\,{\scriptstyle\pm 0.020}$ & $0.780\,{\scriptstyle\pm 0.020}$  \\
         & TB + Buffer & $\underline{238.43}\,{\scriptstyle\pm 0.02}$ & $238.85\,{\scriptstyle\pm 0.02}$ & $0.030\,{\scriptstyle\pm 0.010}$ & $0.080\,{\scriptstyle\pm 0.050}$  \\
         & TB + Buffer + MCMC & $\mathbf{238.48}\,{\scriptstyle\pm 0.01}$ & $\mathbf{238.80}\,{\scriptstyle\pm 0.01}$ & $\underline{0.020}\,{\scriptstyle\pm 0.010}$ & $0.040\,{\scriptstyle\pm 0.020}$  \\
        \noalign{\vskip 2pt}\hdashline\noalign{\vskip 2pt}
        \multirow[t]{7}{3.3cm}{\raggedright GFlowNets (this work)} & DB & $234.91\,{\scriptstyle\pm 0.01}$ & $336.37\,{\scriptstyle\pm 25.27}$ & $0.597\,{\scriptstyle\pm 0.000}$ & $0.908\,{\scriptstyle\pm 0.001}$  \\
         & TB & $238.40\,{\scriptstyle\pm 0.06}$ & $238.95\,{\scriptstyle\pm 0.08}$ & $0.032\,{\scriptstyle\pm 0.008}$ & $\underline{0.025}\,{\scriptstyle\pm 0.008}$  \\
         & VarGrad & $238.22\,{\scriptstyle\pm 0.12}$ & $239.28\,{\scriptstyle\pm 0.26}$ & $0.052\,{\scriptstyle\pm 0.016}$ & $\mathbf{0.014}\,{\scriptstyle\pm 0.007}$  \\
         & SubTB & $234.89\,{\scriptstyle\pm 0.00}$ & $384.65\,{\scriptstyle\pm 6.22}$ & $0.597\,{\scriptstyle\pm 0.000}$ & $0.910\,{\scriptstyle\pm 0.003}$  \\
         & TRPO & $237.26\,{\scriptstyle\pm 0.15}$ & $243.95\,{\scriptstyle\pm 0.99}$ & $0.145\,{\scriptstyle\pm 0.009}$ & $0.036\,{\scriptstyle\pm 0.014}$  \\
         & Ent-PPO ($K=1$) & $238.36\,{\scriptstyle\pm 0.01}$ & $239.00\,{\scriptstyle\pm 0.03}$ & $0.030\,{\scriptstyle\pm 0.009}$ & $0.029\,{\scriptstyle\pm 0.017}$  \\
         & Ent-PPO ($K=2$) & $\mathbf{238.48}\,{\scriptstyle\pm 0.02}$ & $\underline{238.83}\,{\scriptstyle\pm 0.03}$ & $\mathbf{0.014}\,{\scriptstyle\pm 0.007}$ & $0.052\,{\scriptstyle\pm 0.031}$  \\
        \midrule
        \multicolumn{6}{c}{$\beta=0.6$} \\
        \midrule
        \raggedright MCMC (Metropolis Hastings) & --- & --- & --- & $0.220\,{\scriptstyle\pm 0.020}$ & $0.090\,{\scriptstyle\pm 0.060}$  \\
        \noalign{\vskip 2pt}\hdashline\noalign{\vskip 2pt}
        \multirow[t]{7}{3.3cm}{\raggedright Discrete diffusion samplers} & WDCE (MDNS) & $310.18\,{\scriptstyle\pm 0.33}$ & $341.82\,{\scriptstyle\pm 38.37}$ & $0.380\,{\scriptstyle\pm 0.460}$ & $0.410\,{\scriptstyle\pm 0.460}$  \\
         & LV (on-policy) & $309.77\,{\scriptstyle\pm 0.04}$ & $422.53\,{\scriptstyle\pm 94.60}$ & $0.950\,{\scriptstyle\pm 0.000}$ & $0.970\,{\scriptstyle\pm 0.000}$  \\
         & LV + Buffer & $308.55\,{\scriptstyle\pm 1.36}$ & $311.05\,{\scriptstyle\pm 0.36}$ & $0.270\,{\scriptstyle\pm 0.250}$ & $0.400\,{\scriptstyle\pm 0.300}$  \\
         & LV + Buffer + MCMC & $310.18\,{\scriptstyle\pm 0.24}$ & $310.65\,{\scriptstyle\pm 0.08}$ & $0.040\,{\scriptstyle\pm 0.050}$ & $0.140\,{\scriptstyle\pm 0.130}$  \\
         & TB (on-policy) & $309.79\,{\scriptstyle\pm 0.00}$ & $377.92\,{\scriptstyle\pm 14.19}$ & $0.950\,{\scriptstyle\pm 0.000}$ & $0.970\,{\scriptstyle\pm 0.000}$  \\
         & TB + Buffer & $310.42\,{\scriptstyle\pm 0.03}$ & $310.56\,{\scriptstyle\pm 0.01}$ & $\mathbf{0.000}\,{\scriptstyle\pm 0.000}$ & $0.040\,{\scriptstyle\pm 0.020}$  \\
         & TB + Buffer + MCMC & $310.43\,{\scriptstyle\pm 0.01}$ & $310.55\,{\scriptstyle\pm 0.01}$ & $\mathbf{0.000}\,{\scriptstyle\pm 0.000}$ & $0.020\,{\scriptstyle\pm 0.010}$  \\
        \noalign{\vskip 2pt}\hdashline\noalign{\vskip 2pt}
        \multirow[t]{7}{3.3cm}{\raggedright GFlowNets (this work)} & DB & $309.53\,{\scriptstyle\pm 0.00}$ & $525.82\,{\scriptstyle\pm 6.51}$ & $0.949\,{\scriptstyle\pm 0.000}$ & $0.980\,{\scriptstyle\pm 0.000}$  \\
         & TB & $\mathbf{310.47}\,{\scriptstyle\pm 0.00}$ & $\underline{310.51}\,{\scriptstyle\pm 0.00}$ & $\underline{0.001}\,{\scriptstyle\pm 0.001}$ & $\mathbf{0.006}\,{\scriptstyle\pm 0.004}$  \\
         & VarGrad & $310.44\,{\scriptstyle\pm 0.00}$ & $310.59\,{\scriptstyle\pm 0.01}$ & $0.003\,{\scriptstyle\pm 0.000}$ & $0.019\,{\scriptstyle\pm 0.014}$  \\
         & SubTB & $309.53\,{\scriptstyle\pm 0.00}$ & $537.09\,{\scriptstyle\pm 28.12}$ & $0.949\,{\scriptstyle\pm 0.000}$ & $0.980\,{\scriptstyle\pm 0.000}$  \\
         & TRPO & $310.39\,{\scriptstyle\pm 0.02}$ & $310.71\,{\scriptstyle\pm 0.04}$ & $0.004\,{\scriptstyle\pm 0.002}$ & $0.025\,{\scriptstyle\pm 0.016}$  \\
         & Ent-PPO ($K=1$) & $\underline{310.46}\,{\scriptstyle\pm 0.01}$ & $310.55\,{\scriptstyle\pm 0.04}$ & $0.002\,{\scriptstyle\pm 0.001}$ & $0.021\,{\scriptstyle\pm 0.014}$  \\
         & Ent-PPO ($K=2$) & $\mathbf{310.47}\,{\scriptstyle\pm 0.00}$ & $\mathbf{310.50}\,{\scriptstyle\pm 0.00}$ & $\underline{0.001}\,{\scriptstyle\pm 0.001}$ & $\underline{0.013}\,{\scriptstyle\pm 0.005}$  \\
        \bottomrule
    \end{tabular}
\end{table}

\begin{figure}[H]
    \centering
    \includegraphics[width=\linewidth]{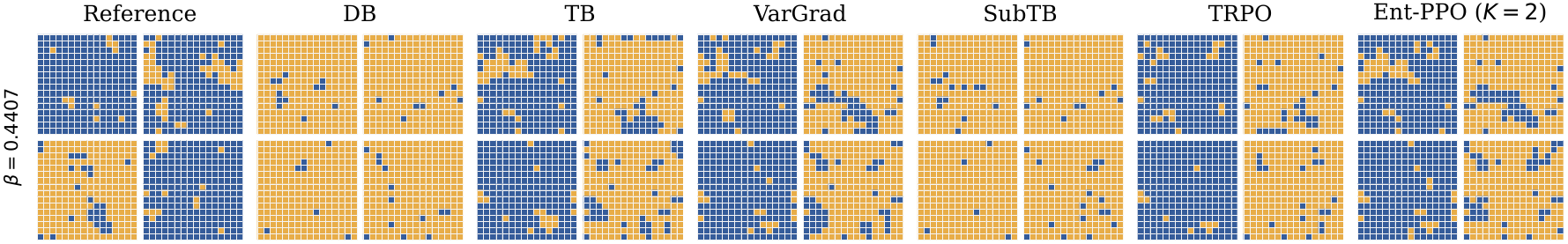}
    \caption{Ising samples at $\beta=0.4407$. Each panel contains a $2\times2$ grid of configurations. Columns show the Swendsen--Wang reference, DB, TB, VarGrad, SubTB, TRPO, and Ent-PPO ($K=2$). Model samples use the final seed-0 checkpoints of the runs reported in~\cref{tab:ising_results}. Blue and yellow denote spins $-1$ and $+1$, respectively.}
    \label{fig:ising_method_samples}
\end{figure}

\newpage
\subsection{Discussion of why TRPO performance is not consistent with the results in~\citep{pmlr-v235-niu24c}}
\label{app:trpo_ablation}

In this section, we discuss why TRPO performance in our experiments deviates from that reported in~\citep{pmlr-v235-niu24c}. We find the key differences and discuss their effects.

\paragraph{Advantage estimation and value training.} First, the variance reduction approach we consider for Ent-PPO differs from that used in~\citep{pmlr-v235-niu24c}. We consider both approaches, ours and that of~\citet{pmlr-v235-niu24c}, in Figure~\ref{fig:trpo_gae_ablation} and report the best-performing variant in the main section. \emph{TRPO Niu et al.} reproduces the algorithm of \citet{pmlr-v235-niu24c}: we use the hyperparameters from their official implementation, in particular $\lambda = 0.99$, and the value function is trained by regression onto the sum of the per-step rewards. \emph{TRPO ours} uses exactly the same variance reduction as our Ent-PPO: GAE with $\lambda = 0.7$, and the value network is trained as described in Section~\ref{sec:variance_reduction}, with $E = 4$ epochs and $S = 2$ mini-batch splits per rollout batch. The two versions differ only in advantage estimation and value training, and both are compared with Ent-PPO ($K = 4$). As can be seen from Figure~\ref{fig:trpo_gae_ablation}, \emph{TRPO ours} outperforms \emph{TRPO Niu et al.}

Second, Ent-PPO outperforms both variants of TRPO, and using the same variance reduction approach allows a direct comparison of the algorithms' convergence speeds.

\begin{figure}[H]
    \centering
    \includegraphics[width=1\linewidth]{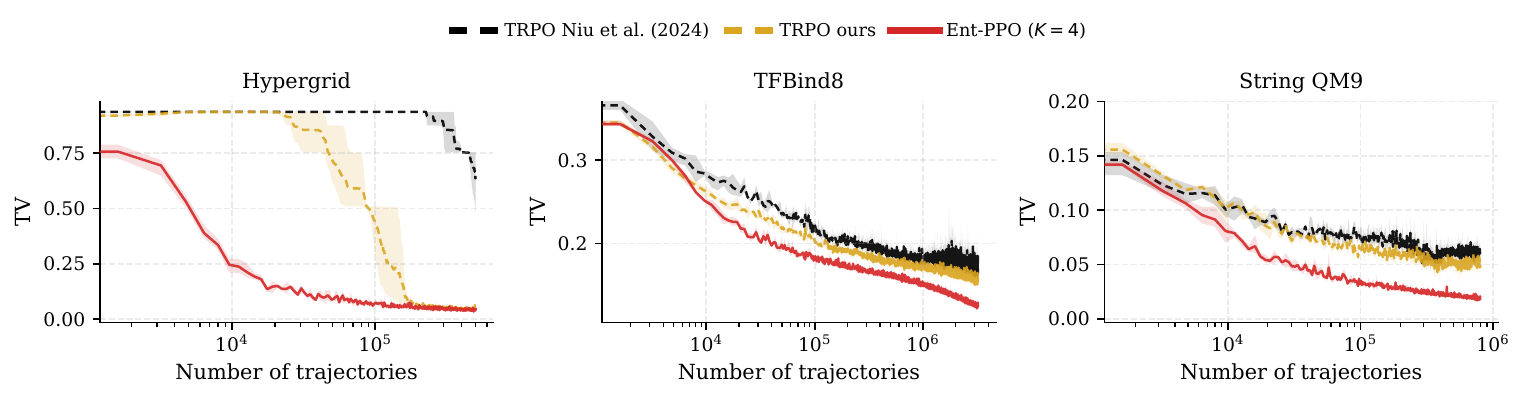}
    \caption{Comparison of two versions of TRPO and Ent-PPO ($K = 4$) on Hypergrid, TFBind8, and String QM9. \emph{TRPO Niu et al.} uses the hyperparameters of \citet{pmlr-v235-niu24c}; \emph{TRPO ours} uses GAE with $\lambda = 0.7$ and the value training of Ent-PPO ($E = 4$, $S = 2$). Curves show TV distance vs.\ the number of reward evaluations. Lower is better. Lines show the mean over 3 seeds. Shaded regions show the min--max range.}
    \label{fig:trpo_gae_ablation}
\end{figure}

However, the advantage estimation above is not the only difference between our setup and that used by~\citet{pmlr-v235-niu24c}. Hence, we ablate all identified differences on Hypergrid, listed in order of importance: the reward \(R_0\), policy-network depth, rollout batch size, and grid size. Figure~\ref{fig:grid_ablation_ppo_trpo} presents these ablations: the top row shows the results with the reward as in~\citep{pmlr-v235-niu24c}, and the bottom row shows the results with the usual reward definition, as in our work and the prior literature~\citep{bengio2021flow, madan2023learning}.

\textbf{The effect of the reward definition $R_0$.}
The hypergrid reward is
\begin{equation*}
R(s) = R_0 + R_1 \prod_{i = 1}^D \mathbb{I}\left[0.25 < \left|\frac{s^i}{H-1}-0.5\right|\right] + R_2 \prod_{i = 1}^D \mathbb{I}\left[0.3 < \left|\frac{s^i}{H-1}-0.5\right| < 0.4\right],
\end{equation*}
with $R_1 = 0.5$ and $R_2 = 2$. \citet{pmlr-v235-niu24c} use $R_0 = 10^{-2}$, whereas our main experiments use $R_0 = 10^{-3}$, following \citep{bengio2021flow, madan2023learning}. Thus, in \citep{pmlr-v235-niu24c} the reward in the low-reward region between the modes is $10\times$ larger than in the standard setup, which leads to less distinct modes and easier exploration:
Figure~\ref{fig:grid_reward_r0_comparison} shows the ground truth log-probability for Hypergrids with $R_0=10^{-3}$ and $R_0=10^{-2}$, respectively.
Figure~\ref{fig:grid_ablation_ppo_trpo} (top) shows more similar convergence speeds between Ent-PPO and TRPO than the bottom row. This effect is related to that observed in Figure~\ref{fig:var_reduction}: less effective algorithms do not effectively discover modes, and it is not seen when $R_0=10^{-2}$ and intermediate rewards are high since modes are easy to find.

\begin{figure}[H]
    \centering
    \includegraphics[width=1\linewidth]{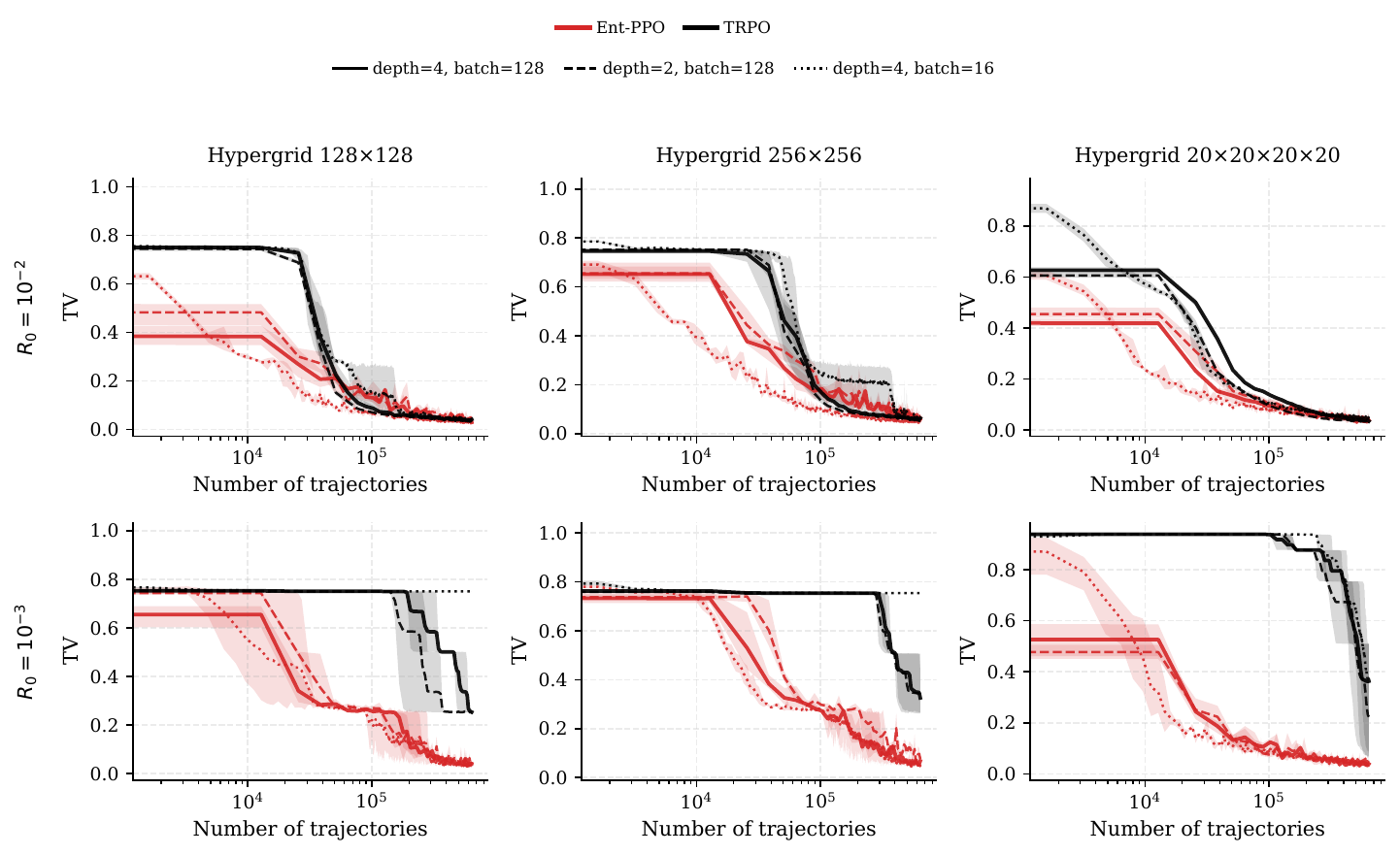}
    \caption{Comparison of Ent-PPO and TRPO on hypergrids of different sizes ($128{\times}128$, $256{\times}256$, $20{\times}20{\times}20{\times}20$) with the baseline reward $R_0 = 10^{-2}$ (top row) and $R_0 = 10^{-3}$ (bottom row). Line styles correspond to different network depth and batch size configurations: solid (depth $=4$, batch size $=128$), dashed (depth $=2$, batch size $=128$), dotted (depth $=4$, batch size $=16$). With the standard $R_0=10^{-3}$, TRPO converges more slowly or does not reach comparable TV within the training budget, while with the larger $R_0 = 10^{-2}$ it converges substantially faster. Curves show TV distance vs.\ the number of reward evaluations. Lower is better. Lines show the mean over 3 seeds. Shaded regions show the min--max range.}
    \label{fig:grid_ablation_ppo_trpo}
\end{figure}

\begin{figure}[H]
    \centering
    \includegraphics[width=1\linewidth]{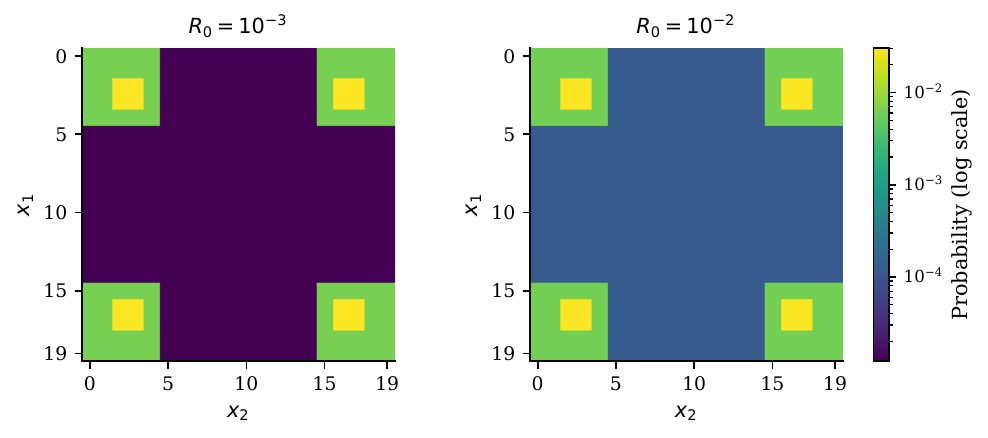}
    \caption{The ground-truth probability in the log scale for the Hypergrid with $D=2$ and $H=20$. \textit{Left}: $R_0=10^{-3}$. \textit{Right}: $R_0=10^{-2}$.}
    \label{fig:grid_reward_r0_comparison}
\end{figure}

\textbf{The differences in the network depth, batch size, and grid size are not critical.}  \citet{pmlr-v235-niu24c} use policy networks with $4$ hidden layers, while we use $2$. \citet{pmlr-v235-niu24c} use a batch of $128$ trajectories per iteration, while we use $16$. \citet{pmlr-v235-niu24c} consider larger grids ($256^2$, $128^2$, $64^3$, $32^4$), while our main experiments use a $4$-dimensional grid of height $20$. We ablate all these choices in Figure~\ref{fig:grid_ablation_ppo_trpo} (solid vs.\ dotted lines) on three Hypergrid boxes,  $128{\times}128$, $256{\times}256$, and $20{\times}20{\times}20{\times}20$.

To keep the comparison fair, all runs use the same budget of trajectories: $40\,000$ iterations at batch size $16$ and $5\,000$ iterations at batch size $128$ ($6.4 \cdot 10^{5}$ trajectories in both cases). All curves are plotted against the number of reward evaluations.
 
Figure~\ref{fig:grid_ablation_ppo_trpo} indicates that the baseline reward \(R_0\) is the dominant factor among the setup differences considered above. Crucially, Ent-PPO performs better than TRPO in all scenarios. Note that we show TRPO with the better variance reduction approach.

\newpage
\section{Additional figures}\label{app:figures}

\begin{figure}[H]
    \centering
    \includegraphics[width=1\linewidth]{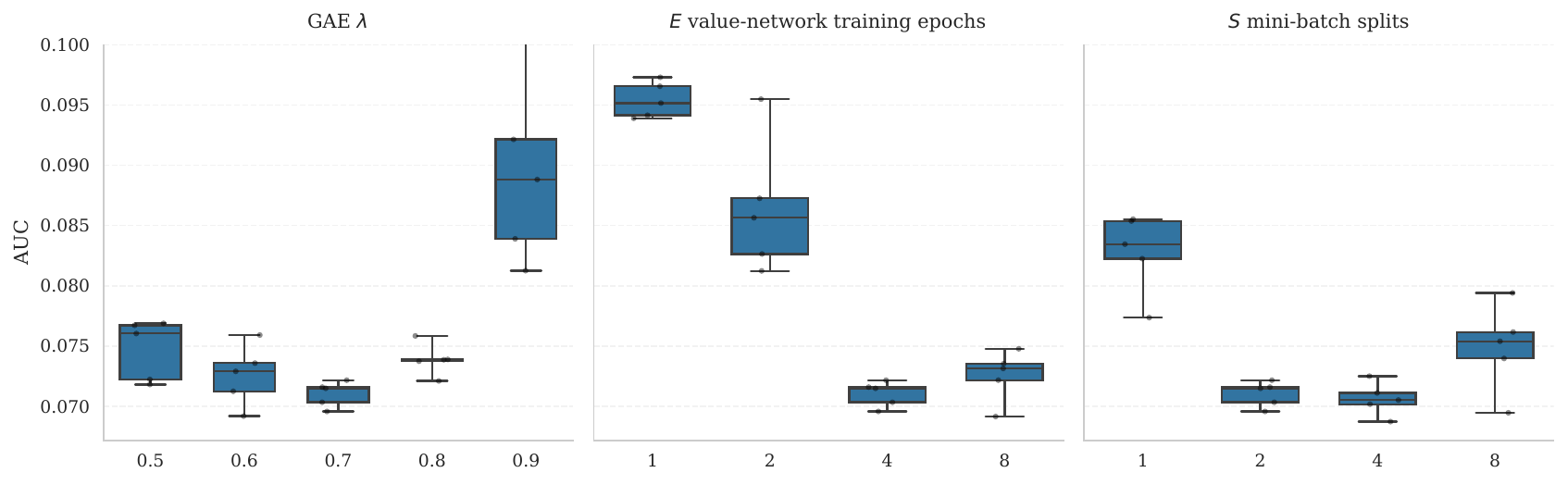}
    \caption{Sensitivity of GAE training to three hyperparameters, evaluated on Hypergrid: the GAE parameter $\lambda$, the number of value-network training epochs per batch, and the number of mini-batch splits within each epoch. For each panel, the labeled hyperparameter varies while the other two are fixed at their best values, $(\lambda, E, S) = (0.7, 4, 2)$. Lower AUC (area under the TV curve) is better. Box plots show the median, Q1, and Q3 across 5 seeds. Whiskers extend to the min and max.}
    \label{fig:GAE_boxpolts}
\end{figure}

\begin{figure}[H]
    \centering
    \includegraphics[width=1\linewidth]{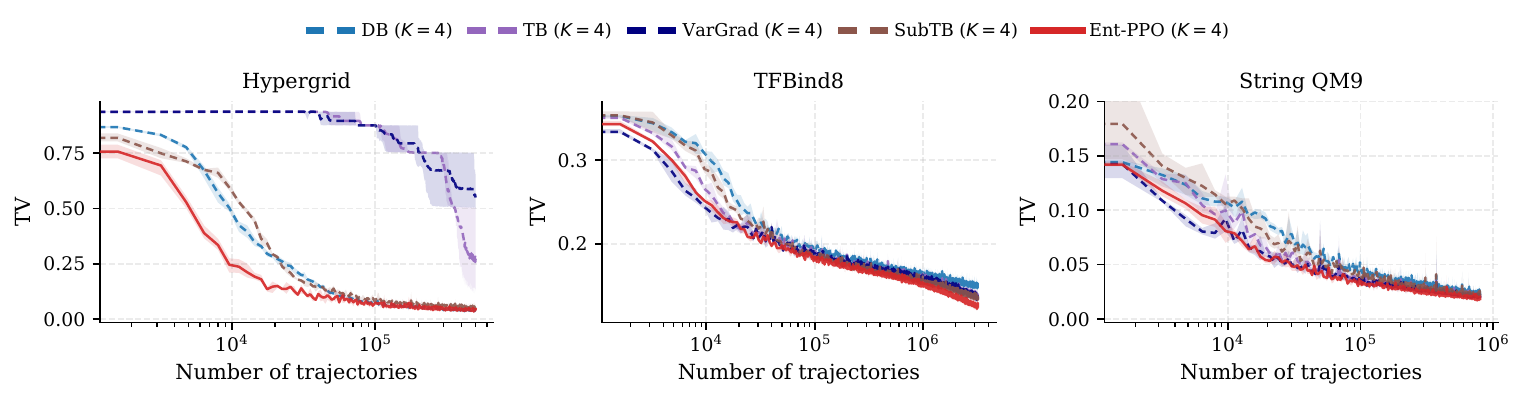}
    \caption{Ent-PPO compared against off-policy baselines (DB, TB, SubTB) and VarGrad when all methods perform four gradient updates per batch. Ent-PPO retains its advantage on all three environments. Lines show the mean over 3 seeds. Shaded regions show the min--max range.}
    \label{fig:small_exps_EntPPO_4epochs}
\end{figure}

\begin{figure}[H]
    \centering
    \includegraphics[width=1\linewidth]{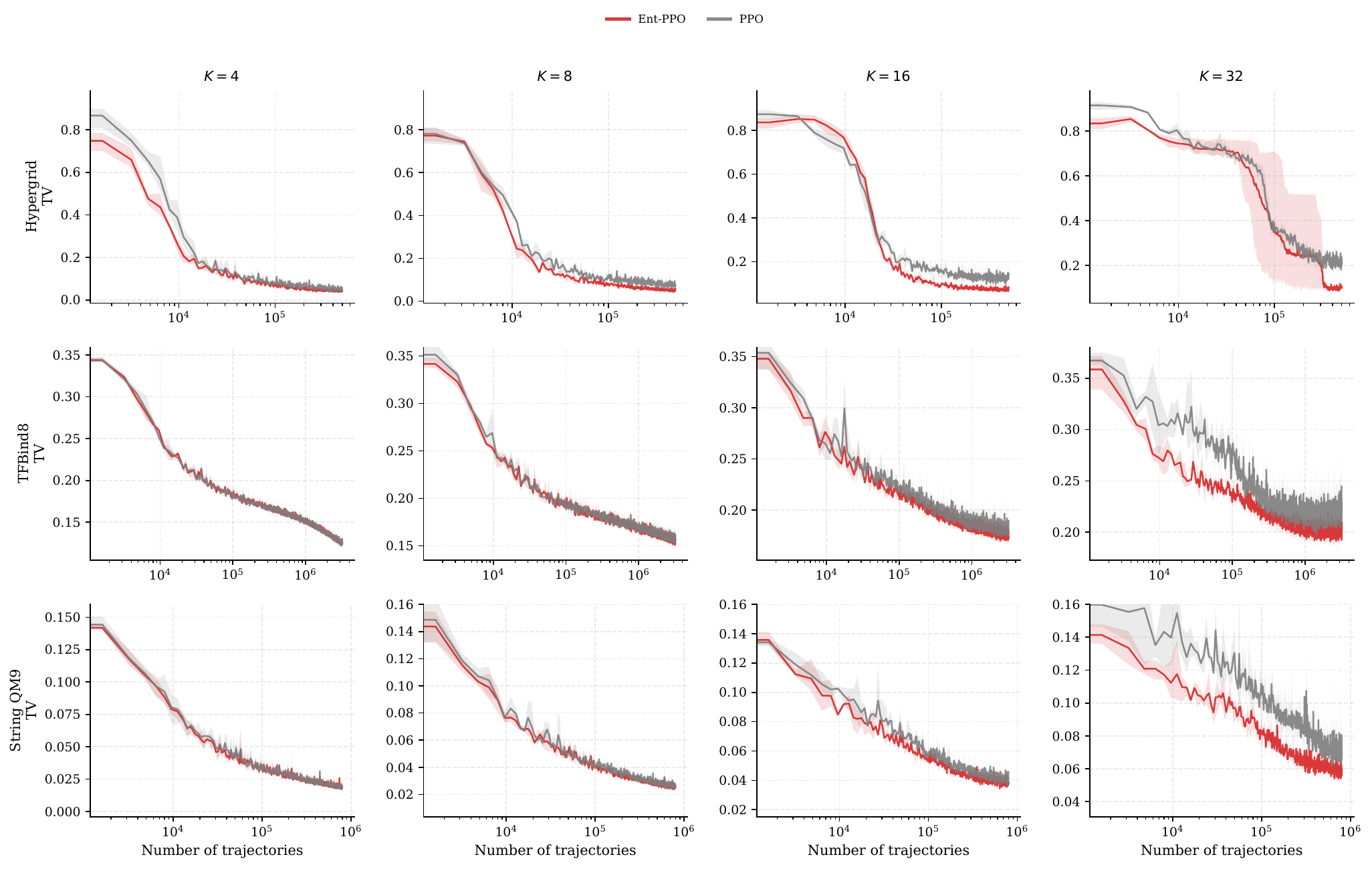}
    \caption{Effect of removing the KL term from Ent-PPO (\cref{eq:ent_ppo}) on Hypergrid, TFBind8, and String QM9, with $K \in \{4, 8, 16, 32\}$ update epochs per batch. Without the KL term, which reduces Ent-PPO to naive PPO, the difference is small for $K = 4$ and grows with $K$: for $K \in \{16, 32\}$ naive PPO plateaus at a higher TV. Lines show the mean over 3 seeds. Shaded regions show the min--max range.}
    \label{fig:EntPPO_vs_PPO}
\end{figure} 

\begin{figure}[H]
    \centering
    \includegraphics[width=1\linewidth]{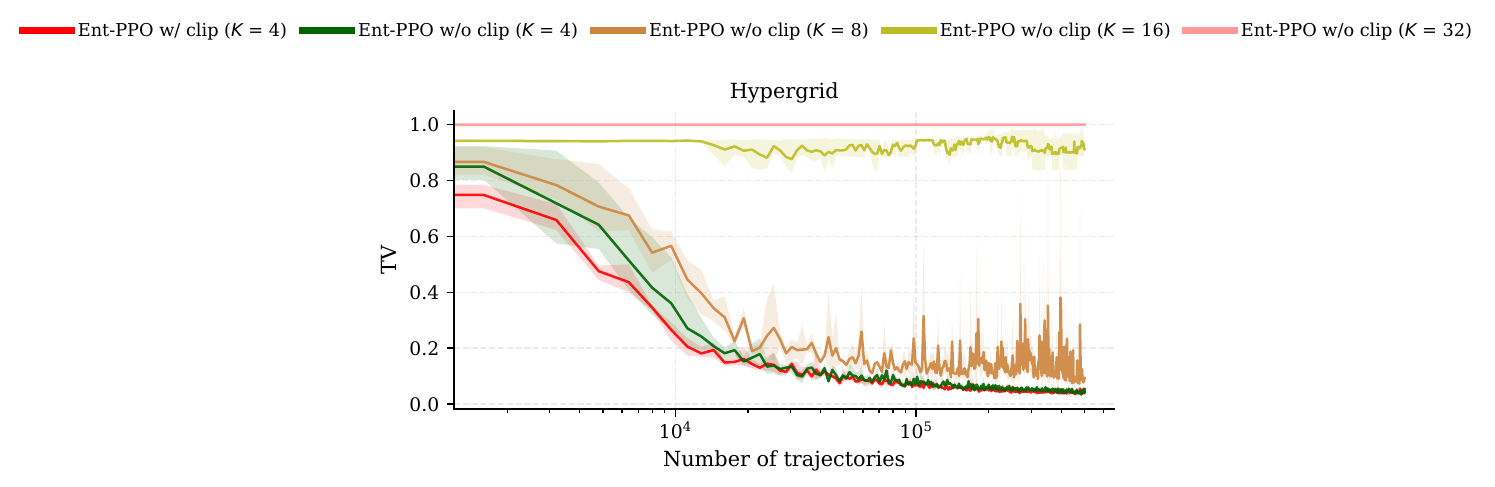}
    \caption{Effect of removing the clipping from Ent-PPO (\cref{eq:ent_ppo}) on Hypergrid with $K \in \{4, 8, 16, 32\}$ update epochs per batch. Without clipping, training is comparable to Ent-PPO for $K=4$, unstable for $K=8$ and fails for $K \in \{16,32\}$. Lines show the mean over 3 seeds. Shaded regions show the min--max range.}
    \label{fig:EntPPO_no_clip}
\end{figure}

\newpage

\section{Pseudocode}\label{app:pseudocode}

\begin{algorithm}[H]
    \caption{Vanilla Policy Gradient}\label{alg:pg}
    \begin{algorithmic}[1]
        \REQUIRE initial policy parameters $\theta_0$, initial 
        value function parameters $\phi_0$, backward policy $\PB$, 
        reward $\cR$
        \FOR{$k = 0, 1, 2, \ldots$}
            \STATE Collect a set of trajectories 
            $\mathcal{D}_k = \{\tau_i\}$ by running 
            $\pi_{\theta_k}$.
            \STATE Assign per-step rewards $r(s_t, s_{t+1})$ along 
            each trajectory using $\PB$ and $\cR$ as 
            in~\cref{eq:sampling_reward}.
            \STATE Compute advantage estimates $\hat A_t$ using 
            one of the methods from~\cref{sec:variance_reduction}, 
            and corresponding value-fit targets $y_t$ (the 
            Monte-Carlo soft return $\hat R_t = \sum_{k=t}^{T-1} 
            g_k$ for the simplest, reward-to-go, and baseline 
            estimators; the bootstrapped target $\mathrm{sg}[\hat 
            A_t + \tilde V_\phi(s_t)]$ for GAE).
            \STATE Estimate the policy gradient:
            \[
                \hat{g}_k = \frac{1}{|\mathcal{D}_k|} 
                \sum_{\tau \in \mathcal{D}_k} \sum_{t=0}^{T-1} 
                \nabla_{\theta_k} \log \pi_{\theta_k}(s_{t+1} 
                \mid s_t) \cdot \hat A_t .
            \]
            \STATE Update $\theta_{k+1}$ by stochastic gradient 
            ascent on $\hat g_k$ using Adam.
            \STATE Fit the value function by mean-squared-error 
            regression:
            \[
                \phi_{k+1} = \arg\min_\phi 
                \frac{1}{|\mathcal{D}_k| \, T} 
                \sum_{\tau \in \mathcal{D}_k} \sum_{t=0}^{T-1} 
                \big( \tilde V_\phi(s_t) - y_t \big)^2 ,
            \]
            via several steps of stochastic gradient descent using 
            Adam.
        \ENDFOR
    \end{algorithmic}
\end{algorithm}

\begin{algorithm}[H]
    \caption{Entropic Proximal Policy Optimisation}\label{alg:ppo}
    \begin{algorithmic}[1]
        \REQUIRE initial policy parameters $\theta_0$, initial 
        value function parameters $\phi_0$, backward policy $\PB$, 
        reward $\cR$, clip $\epsilon$, GAE parameter $\lambda$, 
        number of update epochs $K$
        \FOR{$k = 0, 1, 2, \ldots$}
            \STATE Collect a set of trajectories 
            $\mathcal{D}_k = \{\tau_i\}$ by running 
            $\pi_{\theta_k}$.
            \STATE Assign per-step rewards $r(s_t, s_{t+1})$ along 
            each trajectory using $\PB$ and $\cR$ as 
            in~\cref{eq:sampling_reward}.
            \STATE Compute soft TD residuals $\delta_t = g_t + 
            \tilde V_{\phi_k}(s_{t+1}) - \tilde V_{\phi_k}(s_t)$, 
            GAE advantages $\hat A_t = \sum_{j=0}^{T-1-t} 
            \lambda^j \delta_{t+j}$, and value-fit targets 
            $y_t = \mathrm{sg}[\hat A_t + \tilde V_{\phi_k}(s_t)]$. 
            These quantities are held fixed across the $K$ inner 
            epochs. Set $\theta \gets \theta_k$, 
            $\phi \gets \phi_k$.
            \FOR{$e = 1, \ldots, K$}
                \STATE Estimate the Ent-PPO gradient 
                (\cref{eq:ent_ppo}):
                \[
                    \hat{g}_{k,e} = \nabla_\theta 
                    \frac{1}{|\mathcal{D}_k|} 
                    \sum_{\tau \in \mathcal{D}_k} \sum_{t=0}^{T-1} 
                    \Big[ \mathrm{PPOClip}\big( \rho_t(\theta), 
                    \, \hat A_t \big) - \KL\!\big( 
                    \pi_\theta(\cdot \mid s_t) \,\|\, 
                    \pi_{\theta_k}(\cdot \mid s_t) \big) \Big] ,
                \]
                where $\rho_t(\theta) = \pi_\theta(s_{t+1} \mid 
                s_t) / \pi_{\theta_k}(s_{t+1} \mid s_t)$ and 
                $\mathrm{PPOClip}$ is defined 
                in~\cref{eq:ppo_clip}; $\pi_{\theta_k}$ plays the 
                role of $\pi_{\text{old}}$ from the main text.
                \STATE Update $\theta$ by stochastic gradient 
                ascent on $\hat g_{k,e}$ using Adam.
                \STATE Fit the value function by 
                mean-squared-error regression:
                \[
                    \phi \gets \arg\min_\phi 
                    \frac{1}{|\mathcal{D}_k| \, T} 
                    \sum_{\tau \in \mathcal{D}_k} \sum_{t=0}^{T-1} 
                    \big( \tilde V_\phi(s_t) - y_t \big)^2 ,
                \]
                via several steps of stochastic gradient descent 
                using Adam.
            \ENDFOR
            \STATE Set $\theta_{k+1} \gets \theta$, 
            $\phi_{k+1} \gets \phi$.
        \ENDFOR
    \end{algorithmic}
\end{algorithm}

\newpage
\section{Standard Policy Gradient Derivations}\label{app:standard_derivations}

The derivations below specialize standard policy-gradient results to the soft-RL formulation of GFlowNet sampling 
(\cref{sec:gfn_background}) at $\alpha = \gamma = 1$. We include these derivations for the sake of completeness. These methods can be found in~\citep{williams1992simple, sutton1999policy, Schulman2015HighDimensionalCC}. We use the abbreviations $\tilde V^\pi$, $\tilde Q^\pi$, $\tilde A^\pi$ from \cref{sec:variance_reduction} and write $\tilde J(\pi_\theta) = 
\tilde V^{\pi_\theta}(\sinit)$. Both algorithms in 
\cref{app:pseudocode} take the GFN backward policy $\PB$ and 
reward $\cR$ as input; line 3 of each algorithm constructs the 
soft-MDP reward via~\eqref{eq:sampling_reward}, after which the 
forward policy is trained as in a standard MDP.

We denote by $\cT$ the set of trajectories of length $T$ starting 
at $\sinit$, where $T$ is the maximum trajectory length and shorter 
trajectories are padded by repeated visits to the absorbing state 
$s_f$. Padding contributes zero reward, so all sums  effectively truncate at the first absorbing transition. The rewards $r(s_t, s_{t+1})$ are independent of $\theta$ since $\PB$ is fixed.

\subsection{Simplest policy gradient}\label{app:vpg}

For brevity, write the per-step soft return
\[
    g_t = r(s_t, s_{t+1}) - \log\pi_\theta(s_{t+1}\mid s_t)
\]
as in~\cref{sec:variance_reduction}, and the trajectory return
\[
    G_\theta(\tau) = \sum_{t=0}^{T-1} g_t,
\]
so that $\tilde J(\pi_\theta) = \E_{\tau\sim\pi_\theta}
[G_\theta(\tau)]$ by~\cref{eq:value_def} at 
$\alpha = \gamma = 1$. Trajectories are sampled with probability 
$P_\theta(\tau) = \prod_{t=0}^{T-1} \pi_\theta(s_{t+1}\mid s_t)$, 
hence
\[
    \nabla_\theta \log P_\theta(\tau) = \sum_{t=0}^{T-1} 
    \nabla_\theta \log\pi_\theta(s_{t+1}\mid s_t).
\]

Both $P_\theta$ and $G_\theta$ depend on $\theta$, so we apply the 
product rule:
\begin{align*}
\nabla_\theta \tilde J(\pi_\theta)
&= \sum_{\tau\in\cT} \nabla_\theta P_\theta(\tau)\cdot 
G_\theta(\tau) + \sum_{\tau\in\cT} P_\theta(\tau)\cdot 
\nabla_\theta G_\theta(\tau) \\
&= \E_{\pi_\theta}\!\left[\sum_{t=0}^{T-1} 
\nabla_\theta \log\pi_\theta(s_{t+1}\mid s_t) \cdot G_\theta(\tau)
\right] - \E_{\pi_\theta}\!\left[\sum_{t=0}^{T-1} 
\nabla_\theta \log\pi_\theta(s_{t+1}\mid s_t)\right]\!,
\end{align*}
where the first equality uses $\nabla_\theta P_\theta = P_\theta\, 
\nabla_\theta\log P_\theta$, and the second uses 
$\nabla_\theta G_\theta = -\sum_{t=0}^{T-1} \nabla_\theta 
\log\pi_\theta(s_{t+1}\mid s_t)$ (the rewards $r(s_t, s_{t+1})$ are 
$\theta$-independent because $\PB$ is fixed).

The second expectation vanishes by the score-function identity. 
Conditioning on $s_t$,
\[
    \E_{s_{t+1}\sim\pi_\theta(\cdot\mid s_t)}\!\big[\nabla_\theta 
    \log\pi_\theta(s_{t+1}\mid s_t) \,\big|\, s_t\big] = 
    \sum_{s'} \nabla_\theta \pi_\theta(s'\mid s_t) = \nabla_\theta 
    \sum_{s'} \pi_\theta(s'\mid s_t) = \nabla_\theta 1 = 0,
\]
and the tower rule extends this to the unconditional expectation. 
We obtain the simplest policy gradient:
\[
    \nabla_\theta \tilde J(\pi_\theta) = \E_{\pi_\theta}\!\left[
    \sum_{t=0}^{T-1} \nabla_\theta \log\pi_\theta(s_{t+1}\mid s_t) 
    \cdot G_\theta(\tau)\right]\!.
\]

\subsection{Reward-to-go policy gradient}\label{app:rtg}

We show that the trajectory return $G_\theta(\tau) = 
\sum_{k=0}^{T-1} g_k$ in the simplest policy gradient can be 
replaced by the \emph{reward-to-go} $\hat R_t = \sum_{k=t}^{T-1} 
g_k$, since transitions occurring before step $t$ contribute zero 
in expectation to the gradient term at step $t$.

Split the trajectory return into past and reward-to-go:
\begin{align*}
\E_{\pi_\theta}\!\left[\sum_{t=0}^{T-1} \nabla_\theta 
\log\pi_\theta(s_{t+1}\mid s_t) \cdot G_\theta(\tau)\right]
&= \E_{\pi_\theta}\!\left[\sum_{t=0}^{T-1} \nabla_\theta 
\log\pi_\theta(s_{t+1}\mid s_t) \sum_{k=t}^{T-1} g_k\right] \\
&\quad + \E_{\pi_\theta}\!\left[\sum_{t=0}^{T-1} \nabla_\theta 
\log\pi_\theta(s_{t+1}\mid s_t) \sum_{k=0}^{t-1} g_k\right]\!.
\end{align*}

The second (past) term vanishes. The sum $\sum_{k=0}^{t-1} g_k$ 
depends only on $s_0, \ldots, s_t$, hence is $\sigma(s_0, \ldots, 
s_t)$-measurable. By the tower rule,
\begin{align*}
\E_{\pi_\theta}\!\left[\sum_{t=0}^{T-1} \nabla_\theta 
\log\pi_\theta(s_{t+1}\mid s_t) \sum_{k=0}^{t-1} g_k\right]
&= \sum_{t=0}^{T-1} \E_{\pi_\theta}\!\left[\E\!\left[\nabla_\theta 
\log\pi_\theta(s_{t+1}\mid s_t) \sum_{k=0}^{t-1} g_k \,\Big|\, 
s_0, \ldots, s_t\right]\right] \\
&= \sum_{t=0}^{T-1} \E_{\pi_\theta}\!\left[\left(\sum_{k=0}^{t-1} 
g_k\right) \cdot \E\!\left[\nabla_\theta \log\pi_\theta(s_{t+1}
\mid s_t) \,\Big|\, s_t\right]\right] \\
&= 0,
\end{align*}
where the last equality uses the score-function identity 
$\E[\nabla_\theta \log\pi_\theta(s_{t+1}\mid s_t)\mid s_t] = 0$ 
from~\cref{app:vpg}. The $t = 0$ term vanishes trivially since 
the inner sum is empty. We obtain the reward-to-go policy gradient:
\[
    \nabla_\theta \tilde J(\pi_\theta) = \E_{\pi_\theta}\!\left[
    \sum_{t=0}^{T-1} \nabla_\theta \log\pi_\theta(s_{t+1}\mid s_t) 
    \cdot \hat R_t\right]\!.
\]

\subsection{Baseline}\label{app:baseline}

The reward-to-go gradient can be further refined by subtracting a 
state-dependent \emph{baseline} $b(s_t)$ from $\hat R_t$. For any 
function $b: \cS \to \mathbb{R}$, the same score-function identity 
used in~\cref{app:vpg} gives
\[
    \E_{\pi_\theta}\!\left[\nabla_\theta 
    \log\pi_\theta(s_{t+1}\mid s_t) \cdot b(s_t)\right] = 
    \E_{\pi_\theta}\!\left[b(s_t) \cdot 
    \underbrace{\E\big[\nabla_\theta 
    \log\pi_\theta(s_{t+1}\mid s_t) \,\big|\, s_t\big]}_{=0}\right] 
    = 0,
\]
since $b(s_t)$ is $\sigma(s_0, \ldots, s_t)$-measurable. 
Subtracting $b(s_t)$ therefore preserves unbiasedness:
\[
    \nabla_\theta \tilde J(\pi_\theta) = \E_{\pi_\theta}\!\left[
    \sum_{t=0}^{T-1} \nabla_\theta \log\pi_\theta(s_{t+1}\mid s_t) 
    \cdot \big(\hat R_t - b(s_t)\big)\right]\!.
\]

The standard choice is $b(s_t) = \tilde V^{\pi_\theta}(s_t)$, 
which removes the trajectory-level component shared across actions 
at $s_t$ and substantially reduces the variance of the sample 
estimate. Since $\tilde V^{\pi_\theta}$ is unknown, it is replaced 
in practice with a learned approximation $\tilde V_\phi(s_t)$ 
(see~\cref{sec:baseline}).

\subsection{Advantage policy gradient}\label{app:advantage_pg}

We now express the reward-to-go gradient in terms of the soft 
$Q$-function and the soft advantage from 
\eqref{eq:qfunction} and \eqref{eq:advantage}.

Starting from the reward-to-go form derived in~\cref{app:rtg}: the 
score $\nabla_\theta \log\pi_\theta(s_{t+1}\mid s_t)$ is 
$\sigma(s_t, s_{t+1})$-measurable, so by the tower rule and the 
Markov property the conditioning collapses to $(s_t, s_{t+1})$:
\begin{align*}
\nabla_\theta \tilde J(\pi_\theta)
&= \E_{\pi_\theta}\!\left[\sum_{t=0}^{T-1} \nabla_\theta 
\log\pi_\theta(s_{t+1}\mid s_t) \cdot \hat R_t\right] \\
&= \sum_{t=0}^{T-1} \E_{\pi_\theta}\!\left[\nabla_\theta 
\log\pi_\theta(s_{t+1}\mid s_t) \cdot \E\!\left[\hat R_t \,\big|\, 
s_t, s_{t+1}\right]\right]\!.
\end{align*}

The inner expectation evaluates to
\begin{align*}
\E\!\left[\hat R_t \,\big|\, s_t, s_{t+1}\right]
&= g_t + \E\!\left[\sum_{k=t+1}^{T-1} g_k \,\Big|\, s_{t+1}\right] \\
&= \big(r(s_t, s_{t+1}) - \log\pi_\theta(s_{t+1}\mid s_t)\big) + 
\tilde V^{\pi_\theta}(s_{t+1}) \\
&= \tilde Q^{\pi_\theta}(s_t, s_{t+1}) - 
\log\pi_\theta(s_{t+1}\mid s_t),
\end{align*}
where the second equality uses the sample-path soft-value identity 
\eqref{eq:value_def} at $\alpha = \gamma = 1$ 
to identify $\E[\sum_{k=t+1}^{T-1} g_k \mid s_{t+1}] = 
\tilde V^{\pi_\theta}(s_{t+1})$, and the third uses the soft 
Bellman relation $\tilde Q^{\pi_\theta}(s, s') = r(s, s') + 
\tilde V^{\pi_\theta}(s')$ (under deterministic transitions and the 
action-as-successor convention, the entropy is paid at the 
successor state via $\tilde V^{\pi_\theta}(s')$, with no explicit 
term in $\tilde Q^{\pi_\theta}$).

Substituting back:
\[
    \nabla_\theta \tilde J(\pi_\theta) = \E_{\pi_\theta}\!\left[
    \sum_{t=0}^{T-1} \nabla_\theta \log\pi_\theta(s_{t+1}\mid s_t) 
    \cdot \big(\tilde Q^{\pi_\theta}(s_t, s_{t+1}) - 
    \log\pi_\theta(s_{t+1}\mid s_t)\big)\right]\!.
\]

Subtracting the baseline $b(s_t) = \tilde V^{\pi_\theta}(s_t)$ as 
in~\cref{app:baseline} and recognizing the soft advantage
\[
    \tilde A^{\pi_\theta}(s, s') = \tilde Q^{\pi_\theta}(s, s') - 
    \tilde V^{\pi_\theta}(s) - \log\pi_\theta(s'\mid s)
\]
from~\cref{eq:advantage} at $\alpha = 1$ 
yields:
\[
    \nabla_\theta \tilde J(\pi_\theta) = \E_{\pi_\theta}\!\left[
    \sum_{t=0}^{T-1} \nabla_\theta \log\pi_\theta(s_{t+1}\mid s_t) 
    \cdot \tilde A^{\pi_\theta}(s_t, s_{t+1})\right]\!,
\]
which is the soft-advantage form of \cref{eq:pg_theorem} 
specialized to $\alpha = \gamma = 1$.

\newpage
\section{Theory of Entropic PPO}\label{app:entppo}

In this section, we re-derive the Ent-PPO algorithm from the same perspective as PPO was derived in the seminal paper \citep{schulman2017proximal}. 
We write $\tilde\cL_{\text{soft}}(\theta;\pi_{\text{old}}) \triangleq \cL_{\text{soft}}(\theta;\pi_{\text{old}}) - \E_{\tau\sim\pi_{\text{old}}}\big[\sum_{t=0}^{T-1}\tilde V^{\pi_{\text{old}}}(s_t)\big]$ for the surrogate~\eqref{eq:soft_pi_aggregate} with the $\theta$-independent baseline removed; by~\eqref{eq:soft_pi_via_A} and~\eqref{eq:kl_combination}, it coincides with the unclipped Ent-PPO surrogate~\eqref{eq:ent_ppo_unclipped}.

We start from the soft version of the performance-difference lemma \citep{kakade2002approximately,neu2017unified,geist2019theory} and then characterize how the soft surrogate objective $\tilde{\cL}_{\text{soft}}(\theta;\pi_{\text{old}})$ appears as a natural first-order approximation to it. Afterwards, we characterize the error between $\tilde{\cL}_{\text{soft}}(\theta;\pi_{\text{old}})$ and $\tilde{J}(\pi_\theta) - \tilde{J}(\pi_{\text{old}})$ as a function of the distance between $\pi_\theta$ and $\pi_{\text{old}}$ to justify the theoretical necessity of the trust-region component.

Throughout this section we use the notation of \cref{sec:rl_background,sec:methodology} specialized to the GFlowNet regime $\alpha = \gamma = 1$ with the action-as-successor convention, and we write $d^\pi_t \in \Delta(\cS)$ for the marginal distribution of $s_t$ under trajectories sampled from $\pi$ starting at $\sinit$.

\paragraph{Soft performance-difference lemma.}
The classical performance-difference lemma of \citet{kakade2002approximately} extends to the entropy-regularized regime; equivalent statements appear in \citet{neu2017unified} and \citet{geist2019theory}, and in the GFlowNet setting it also appears in the proof of Lemma~C.2 of \citet{pmlr-v235-niu24c}. The form below is specialized to our finite-horizon, deterministic-transition setting at $\alpha = \gamma = 1$.

% \todoDT{Define $\phi$-functions just here}

\begin{lemma}[Soft performance-difference lemma]\label{lem:soft_pdl}
For any two policies $\pi$ and $\pi_{\text{old}}$,
\begin{equation}\label{eq:soft_pdl}
    \tilde J(\pi) - \tilde J(\pi_{\text{old}})
    = \E_{\tau\sim\pi}\!\left[\sum_{t=0}^{T-1}\!\Big(
        \tilde A^{\pi_{\text{old}}}(s_t,s_{t+1})
        - \KL\!\big(\pi(\cdot\mid s_t)\,\|\,\pi_{\text{old}}(\cdot\mid s_t)\big)
    \Big)\right]\!.
\end{equation}
\end{lemma}
The result is well-known in the soft RL literature \citep{neu2017unified,geist2019theory}, and we provide a proof for the sake of completeness.
\begin{proof}
From \eqref{eq:value_def} at $\alpha = \gamma = 1$, together with the identity $\E_{a\sim\pi}[-\log\pi(a\mid s)] = \cH(\pi(\cdot\mid s))$, the soft value admits the sample-path form
\begin{equation*}
    \tilde V^\pi(\sinit) = \E_{\tau\sim\pi}\!\left[\sum_{t=0}^{T-1}\!\big(r(s_t,s_{t+1}) - \log\pi(s_{t+1}\mid s_t)\big)\right]\!.
\end{equation*}
Pathwise, the telescoping identity
\(
    \sum_{t=0}^{T-1}\!\big(\tilde V^{\pi_{\text{old}}}(s_{t+1}) - \tilde V^{\pi_{\text{old}}}(s_t)\big)
    = \tilde V^{\pi_{\text{old}}}(s_T) - \tilde V^{\pi_{\text{old}}}(s_0)
    = -\tilde V^{\pi_{\text{old}}}(\sinit)
\)
holds because $s_0 = \sinit$ and $\tilde V^{\pi_{\text{old}}}(s_T) = 0$ at the absorbing state. Substituting this into the previous display,
\begin{align*}
    \tilde V^\pi(\sinit) - \tilde V^{\pi_{\text{old}}}(\sinit)
    = \E_{\tau\sim\pi}\!\left[\sum_{t=0}^{T-1}\!\Big(
        r(s_t,s_{t+1}) - \log\pi(s_{t+1}\mid s_t)
        + \tilde V^{\pi_{\text{old}}}(s_{t+1}) - \tilde V^{\pi_{\text{old}}}(s_t)
    \Big)\right]\!.
\end{align*}
Identifying $\tilde Q^{\pi_{\text{old}}}(s_t,s_{t+1}) = r(s_t,s_{t+1}) + \tilde V^{\pi_{\text{old}}}(s_{t+1})$ from \eqref{eq:qfunction} (deterministic transitions, $\gamma = 1$),
\begin{equation*}
    \tilde J(\pi) - \tilde J(\pi_{\text{old}})
    = \E_{\tau\sim\pi}\!\left[\sum_{t=0}^{T-1}\!\Big(
        \tilde Q^{\pi_{\text{old}}}(s_t,s_{t+1}) - \tilde V^{\pi_{\text{old}}}(s_t) - \log\pi(s_{t+1}\mid s_t)
    \Big)\right]\!.
\end{equation*}
The advantage definition in \eqref{eq:advantage} at $\alpha = 1$ rearranges to $\tilde Q^{\pi_{\text{old}}}(s,s') - \tilde V^{\pi_{\text{old}}}(s) = \tilde A^{\pi_{\text{old}}}(s,s') + \log\pi_{\text{old}}(s'\mid s)$, hence
\begin{equation}\label{eq:soft_pdl_intermediate}
    \tilde J(\pi) - \tilde J(\pi_{\text{old}})
    = \E_{\tau\sim\pi}\!\left[\sum_{t=0}^{T-1}\!\Big(
        \tilde A^{\pi_{\text{old}}}(s_t,s_{t+1}) + \log\pi_{\text{old}}(s_{t+1}\mid s_t) - \log\pi(s_{t+1}\mid s_t)
    \Big)\right]\!.
\end{equation}
Conditional on $s_t$, the inner expectation of the last two log-probabilities under $s_{t+1}\sim\pi(\cdot\mid s_t)$ equals $-\KL(\pi(\cdot\mid s_t)\,\|\,\pi_{\text{old}}(\cdot\mid s_t))$, which is $\sigma(s_t)$-measurable; the tower rule then yields \eqref{eq:soft_pdl}.
\end{proof}

\paragraph{Approximation by soft surrogate objective.}
The right-hand side of \eqref{eq:soft_pdl} is impractical to optimise directly: at the start of an Ent-PPO iteration we have rollouts from $\pi_{\text{old}}$, not from $\pi_\theta$. Following the strategy of CPI \citep{kakade2002approximately} and TRPO \citep{schulman2015trust}, we replace the on-policy distribution under $\pi$ in \eqref{eq:soft_pdl} by the analogous distribution under $\pi_{\text{old}}$. 
The next lemma shows that the soft surrogate $\cL_{\text{soft}}$ defined in \eqref{eq:soft_pi_aggregate} is exactly this substitution.

% DT{The following lemma should be rewritten, I switched the notation to advantage-dependent}
\begin{lemma}[Surrogate as state-distribution swap]\label{lem:surrogate_eq}
Define the per-state functional
\begin{equation}\label{eq:phi_def}
    \phi(\pi_\theta;s) \triangleq \E_{s'\sim\pi_\theta(\cdot\mid s)}\!\big[\tilde A^{\pi_{\text{old}}}(s,s')\big] - \KL\!\big(\pi_\theta(\cdot\mid s)\,\|\,\pi_{\text{old}}(\cdot\mid s)\big).
\end{equation}
Then, the following identity holds
\begin{align}\label{eq:decompose_tilde}
    \tilde{\cL}_{\text{soft}}(\theta;\pi_{\text{old}}) 
    &= \sum_{t=0}^{T-1}\E_{s_t\sim d^{\pi_{\text{old}}}_t}\!\big[\phi(\pi_\theta;s_t)\big]\,,\qquad 
    \tilde J(\pi_\theta) - \tilde J(\pi_{\text{old}})
    = \sum_{t=0}^{T-1}\E_{s_t\sim d^{\pi_\theta}_t}\!\big[\phi(\pi_\theta;s_t)\big].
\end{align}
\end{lemma}

\begin{proof}
Follows from applying the tower rule to \eqref{eq:ent_ppo_unclipped} and \eqref{eq:soft_pdl} with the inner expectation taken over $s_{t+1}\sim\pi_\theta(\cdot\mid s_t)$, the integrand becomes exactly $\phi(\pi_\theta;s_t)$ (the KL term in \eqref{eq:soft_pdl} is already $\sigma(s_t)$-measurable).
\end{proof}

\eqref{eq:decompose_tilde} make the surrogate-vs-true relationship transparent: both quantities are sums over time of the \emph{same} per-state functional $\phi(\pi_\theta;\cdot)$, evaluated under different state distributions ($d^{\pi_{\text{old}}}_t$ for the surrogate, $d^{\pi_\theta}_t$ for the truth). They coincide at $\theta = \theta_{\text{old}}$, where both vanish, and agree to first order in $\theta$ at $\theta_{\text{old}}$ --- the same property that justifies CPI, TRPO, and PPO in the unregularized case.

\paragraph{Policy improvement guarantees.}
We now bound the surrogate-vs-true error and use it to derive a CPI-style monotone improvement guarantee for Ent-PPO. The key observation is that the per-state functional $\phi$ admits a clean decomposition into two non-negative KL terms.

Let $\pi^\sharp_{\text{old}}$ denote the soft greedy of $\pi_{\text{old}}$, i.e.\ the unique maximizer in \eqref{eq:soft_pi_state} at $\alpha = 1$ when applied to $\pi_{\text{old}}$:
\begin{equation}\label{eq:soft_greedy_def}
    \pi^\sharp_{\text{old}}(s'\mid s) \,=\, \frac{\exp\!\big(\tilde Q^{\pi_{\text{old}}}(s,s')\big)}{\sum_{u\in\cA_s}\exp\!\big(\tilde Q^{\pi_{\text{old}}}(s,u)\big)}\,, \qquad s\in\cS,\; s'\in\cA_s.
\end{equation}
% \todoDT{Is it policy mirror descent?}

\begin{proposition}[Two-term decomposition of $\phi$]\label{prop:phi_decomposition}
For every $s\in\cS$,
\begin{equation}\label{eq:phi_decomposition}
    \phi(\pi_\theta;s) \,=\, \KL\!\big(\pi_{\text{old}}(\cdot\mid s)\,\|\,\pi^\sharp_{\text{old}}(\cdot\mid s)\big) \,-\, \KL\!\big(\pi_\theta(\cdot\mid s)\,\|\,\pi^\sharp_{\text{old}}(\cdot\mid s)\big).
\end{equation}
\end{proposition}

\begin{proof}
Set $\Lambda(s) \triangleq \log\sum_{u\in\cA_s}\exp(\tilde Q^{\pi_{\text{old}}}(s,u))$, so that \eqref{eq:soft_greedy_def} reads $\log\pi^\sharp_{\text{old}}(s'\mid s) = \tilde Q^{\pi_{\text{old}}}(s,s') - \Lambda(s)$. For any policy $\pi$,
\begin{align}
    \E_{s'\sim\pi}\!\big[\tilde Q^{\pi_{\text{old}}}(s,s')\big] + \cH(\pi(\cdot\mid s))
    &= \E_{s'\sim\pi}\!\big[\log\pi^\sharp_{\text{old}}(s'\mid s) + \Lambda(s)\big] - \E_{s'\sim\pi}\!\big[\log\pi(s'\mid s)\big] \notag\\
    &= \Lambda(s) - \KL\!\big(\pi(\cdot\mid s)\,\|\,\pi^\sharp_{\text{old}}(\cdot\mid s)\big).
    \label{eq:LSE_KL_identity}
\end{align}
Applied with $\pi = \pi_{\text{old}}$, the left-hand side of \eqref{eq:LSE_KL_identity} equals $\tilde V^{\pi_{\text{old}}}(s)$ by the soft Bellman equation at $\alpha = 1$, so
$    \Lambda(s) = \tilde V^{\pi_{\text{old}}}(s) + \KL\!\big(\pi_{\text{old}}(\cdot\mid s)\,\|\,\pi^\sharp_{\text{old}}(\cdot\mid s)\big).$
Substituting back into \eqref{eq:LSE_KL_identity} at $\pi = \pi_\theta$ and rearranging,
\begin{equation*}
    \E_{s'\sim\pi_\theta}\!\big[\tilde Q^{\pi_{\text{old}}}(s,s')\big] + \cH(\pi_\theta(\cdot\mid s)) - \tilde V^{\pi_{\text{old}}}(s)
    = \KL(\pi_{\text{old}}\,\|\,\pi^\sharp_{\text{old}}) - \KL(\pi_\theta\,\|\,\pi^\sharp_{\text{old}}).
\end{equation*}
The left-hand side equals $\phi(\pi_\theta;s)$ by the calculation in the proof of \cref{lem:surrogate_eq}.
\end{proof}

The decomposition \eqref{eq:phi_decomposition} has a crisp interpretation. The first term, $\KL(\pi_{\text{old}}\,\|\,\pi^\sharp_{\text{old}})$, is independent of $\theta$ and quantifies how far $\pi_{\text{old}}$ is from being soft-optimal --- it vanishes if and only if $\pi_{\text{old}}$ is a fixed point of soft policy improvement. The second term penalises deviation of $\pi_\theta$ from the soft-greedy target $\pi^\sharp_{\text{old}}$, attaining its global minimum (zero) at $\pi_\theta = \pi^\sharp_{\text{old}}$, which is the exact soft-policy-iteration step. Combining Lemma~\ref{lem:surrogate_eq} and Proposition~\ref{prop:phi_decomposition},
\begin{equation}\label{eq:Lsoft_KLs}
    \tilde{\cL}_{\text{soft}}(\theta;\pi_{\text{old}})
    \,=\, \cE(\pi_{\text{old}}) \,-\, \sum_{t=0}^{T-1}\E_{s_t\sim d^{\pi_{\text{old}}}_t}\!\big[\KL\!\big(\pi_\theta(\cdot\mid s_t)\,\|\,\pi^\sharp_{\text{old}}(\cdot\mid s_t)\big)\big],
\end{equation}
where
\begin{equation}\label{eq:E_def}
    \cE(\pi_{\text{old}})
    \,\triangleq\, \sum_{t=0}^{T-1}\E_{s_t\sim d^{\pi_{\text{old}}}_t}\!\big[\KL\!\big(\pi_{\text{old}}(\cdot\mid s_t)\,\|\,\pi^\sharp_{\text{old}}(\cdot\mid s_t)\big)\big] \,\ge\, 0
\end{equation}
is the soft Bellman residual of $\pi_{\text{old}}$ averaged along its own occupancy. The supremum of \eqref{eq:Lsoft_KLs} over $\theta$ equals $\cE(\pi_{\text{old}})$, attained at $\pi_\theta = \pi^\sharp_{\text{old}}$, which yields the following lemma.

\begin{lemma}[Soft fixed point]\label{lem:soft_fixed_point}
$\max_{\theta} \tilde{\cL}_{\text{soft}}(\theta;\pi_{\text{old}}) = \cE(\pi_{\text{old}}) = 0$ if and only if $\pi_{\text{old}}(\cdot\mid s) = \pi^\sharp_{\text{old}}(\cdot\mid s)$ for every $s$ reachable under $\pi_{\text{old}}$. In that case $\pi_{\text{old}}$ is the soft-optimal policy and induces the target sampling distribution $\cR/\rmZ$ over $\cX$.
\end{lemma}

\begin{proof}
Each summand in \eqref{eq:E_def} is non-negative, so $\cE(\pi_{\text{old}}) = 0$ iff $\KL(\pi_{\text{old}}(\cdot\mid s)\,\|\,\pi^\sharp_{\text{old}}(\cdot\mid s)) = 0$ for every $s$ with $d^{\pi_{\text{old}}}_t(s) > 0$ for some $t$, which is equivalent to $\pi_{\text{old}} = \pi^\sharp_{\text{old}}$ on those states. By \eqref{eq:soft_greedy_def} this is the soft Bellman optimality equation at $\alpha=1$, whose unique solution on reachable states is $\pi^\star_{\alpha=1}$ (\citealt{geist2019theory}, Theorem~1). The equivalence in \cref{sec:equivalence} (see also \citealt[Proposition 1]{tiapkin2024generative}) then identifies the induced terminal distribution with $\cR/\rmZ$.
\end{proof}

We are now in position to bound the surrogate-vs-true error.

\begin{proposition}[Surrogate error bound]\label{prop:surrogate_error}
Let
\begin{equation*}
    \bar\delta(\theta) \triangleq \max_{0\le t \le T-1}\E_{s\sim d^{\pi_{\text{old}}}_t}\!\Big[\sqrt{\KL\!\big(\pi_\theta(\cdot\mid s)\,\|\,\pi_{\text{old}}(\cdot\mid s)\big)}\Big],
    \quad
    M(\theta) \triangleq \big\|\KL(\pi_{\text{old}}\,\|\,\pi^\sharp_{\text{old}})\big\|_\infty + \big\|\KL(\pi_\theta\,\|\,\pi^\sharp_{\text{old}})\big\|_\infty,
\end{equation*}
where the suprema are over $\cS$. Then $M(\theta)<\infty$ and
\begin{equation}\label{eq:surrogate_error_bound}
    \big|\tilde J(\pi_\theta) - \tilde J(\pi_{\text{old}}) - \tilde\cL_{\text{soft}}(\theta;\pi_{\text{old}})\big| \le \tfrac{T(T-1)}{2\sqrt2}\, M(\theta)\,\bar\delta(\theta).
\end{equation}
\end{proposition}

\begin{proof}
Subtracting the two identities of~\eqref{eq:decompose_tilde},
\begin{equation}\label{eq:error_as_dual}
    \tilde J(\pi_\theta) - \tilde J(\pi_{\text{old}}) - \tilde\cL_{\text{soft}}(\theta;\pi_{\text{old}})
    = \sum_{t=0}^{T-1}\big\langle d^{\pi_\theta}_t - d^{\pi_{\text{old}}}_t,\, \phi(\pi_\theta;\cdot)\big\rangle.
\end{equation}
Since $d^{\pi_\theta}_t$ and $d^{\pi_{\text{old}}}_t$ are probability distributions, each summand is unchanged when a constant is subtracted from $\phi(\pi_\theta;\cdot)$; subtracting the midpoint of its range gives $|\langle d^{\pi_\theta}_t - d^{\pi_{\text{old}}}_t, \phi(\pi_\theta;\cdot)\rangle| \le \mathrm{osc}(\phi(\pi_\theta;\cdot))\,\mathrm{TV}(d^{\pi_\theta}_t, d^{\pi_{\text{old}}}_t)$, with $\mathrm{TV} = \tfrac12\|\cdot\|_1$ and $\mathrm{osc}(f) = \sup f - \inf f$. By~\eqref{eq:phi_decomposition}, $-\|\KL(\pi_\theta\,\|\,\pi^\sharp_{\text{old}})\|_\infty \le \phi(\pi_\theta;s) \le \|\KL(\pi_{\text{old}}\,\|\,\pi^\sharp_{\text{old}})\|_\infty$, hence $\mathrm{osc}(\phi(\pi_\theta;\cdot)) \le M(\theta)$.

Next, $d^\pi_{t+1}(s') = \sum_s d^\pi_t(s)\,\pi(s'\mid s)$, so
\begin{equation*}
    d^{\pi_\theta}_{t+1} - d^{\pi_{\text{old}}}_{t+1} = \sum_s \big(d^{\pi_\theta}_t(s) - d^{\pi_{\text{old}}}_t(s)\big)\pi_\theta(\cdot\mid s) + \sum_s d^{\pi_{\text{old}}}_t(s)\big(\pi_\theta(\cdot\mid s) - \pi_{\text{old}}(\cdot\mid s)\big).
\end{equation*}
Since Markov kernels do not increase total variation, $\mathrm{TV}(d^{\pi_\theta}_{t+1}, d^{\pi_{\text{old}}}_{t+1}) \le \mathrm{TV}(d^{\pi_\theta}_t, d^{\pi_{\text{old}}}_t) + \E_{s\sim d^{\pi_{\text{old}}}_t}[\mathrm{TV}(\pi_\theta(\cdot\mid s), \pi_{\text{old}}(\cdot\mid s))]$, and since $d^{\pi_\theta}_0 = d^{\pi_{\text{old}}}_0 = \delta_{\sinit}$, induction gives the finite-horizon analogue of~\citet[Lemma~3]{achiam2017constrained},
\begin{equation*}
    \mathrm{TV}\big(d^{\pi_\theta}_t, d^{\pi_{\text{old}}}_t\big) \le \sum_{t'=0}^{t-1}\E_{s\sim d^{\pi_{\text{old}}}_{t'}}\big[\mathrm{TV}\big(\pi_\theta(\cdot\mid s), \pi_{\text{old}}(\cdot\mid s)\big)\big].
\end{equation*}
By Pinsker's inequality each expectation on the right is at most $\bar\delta(\theta)/\sqrt2$, so $\sum_{t=0}^{T-1}\mathrm{TV}(d^{\pi_\theta}_t, d^{\pi_{\text{old}}}_t) \le \tfrac{T(T-1)}{2\sqrt2}\,\bar\delta(\theta)$. Combining the two bounds yields~\eqref{eq:surrogate_error_bound}. Finiteness of $M(\theta)$ follows from $\pi^\sharp_{\text{old}}$ having full support on every $\cA_s$.
\end{proof}

\paragraph{Monotone improvement.}
Combining \eqref{eq:Lsoft_KLs} with \eqref{eq:surrogate_error_bound} yields the central inequality
\begin{equation}\label{eq:improvement_lower_bound}
    \tilde J(\pi_\theta) - \tilde J(\pi_{\text{old}})
    \,\ge\, \underbrace{\cE(\pi_{\text{old}})
    \,-\, \sum_{t=0}^{T-1}\E_{d^{\pi_{\text{old}}}_t}\!\big[\KL(\pi_\theta\,\|\,\pi^\sharp_{\text{old}})\big]}_{=\,\tilde\cL_{\text{soft}}(\theta;\pi_{\text{old}})\ \text{by \eqref{eq:Lsoft_KLs}}}
    \,-\, \tfrac{T(T-1)}{2\sqrt 2}\, M(\theta)\, \bar\delta(\theta).
\end{equation}
The two negative terms call for fundamentally different mechanisms.

The first expression is the soft surrogate itself; \cref{prop:phi_decomposition} rewrites it as a difference of two non-negative KLs that makes the soft-policy-iteration target $\pi^\sharp_{\text{old}}$ visible. Its supremum over $\theta$ is $\cE(\pi_{\text{old}})\ge 0$, attained at $\pi_\theta = \pi^\sharp_{\text{old}}$, and it equals zero at $\pi_\theta = \pi_{\text{old}}$; \cref{lem:soft_fixed_point} shows it furthermore vanishes uniformly only at the soft-optimal policy. Crucially, the analytic KL term $\KL(\pi_\theta\,\|\,\pi_{\text{old}})$ in~\eqref{eq:ent_ppo} is intrinsic to this term: it is the entropy correction that emerges from soft policy improvement via~\eqref{eq:kl_combination}. It makes the maximizer of $\tilde\cL_{\text{soft}}(\theta;\pi_{\text{old}})$ well defined, namely the soft-greedy policy $\pi^\sharp_{\text{old}}$ rather than a vertex of the simplex (\cref{prop:beta_family}), but it does not keep this maximizer close to $\pi_{\text{old}}$.

The second negative term tells a different story. It is controlled by $\bar\delta(\theta)$, the expected $\KL(\pi_\theta\,\|\,\pi_{\text{old}})$ along $d^{\pi_{\text{old}}}$, which introduces a necessity for an trust-region mechanism, and \emph{importance-ratio clipping plays exactly this role}. This explains, at the level of the bound, why removing either component is harmful --- the KL term (\cref{fig:EntPPO_vs_PPO}) and the clipping (\cref{fig:EntPPO_no_clip}) --- and why the effect of each is most pronounced with multiple update epochs per batch.

\paragraph{Adaptive trust region.}
The relative importance of the two mechanisms is not constant along training. As $\pi_{\text{old}}$ approaches the soft-optimal policy, $\cE(\pi_{\text{old}}) \to 0$ and $\pi^\sharp_{\text{old}} \to \pi_{\text{old}}$ by \cref{lem:soft_fixed_point}, so the deviation $\KL(\pi_\theta\,\|\,\pi_{\text{old}})$ that the inner-loop maximizer would naturally produce shrinks. Both factors of the error term $M(\theta)\bar\delta(\theta)$ in~\eqref{eq:improvement_lower_bound} then go to zero, and the additional control provided by clipping becomes correspondingly less consequential. The mechanism is most needed early in training, when $\cE(\pi_{\text{old}})$ is large and the surrogate maximizer pushes $\pi_\theta$ aggressively towards $\pi^\sharp_{\text{old}}$.

We note that this behaviour is a property of entropy regularization itself. Since $\log\pistar_{\alpha=1} = \tilde Q^{\pistar} - \tilde V^{\pistar}$, the soft advantage of the optimal policy vanishes identically, $\tilde A^{\pistar} \equiv 0$, and so does $M(\theta)$ once $\pi_\theta$ is also close to $\pi_{\text{old}}$. The same holds for the bound of \citet[Thm.~3.6]{pmlr-v235-niu24c}: their constant $\epsilon_F$, the maximal expected advantage of the old policy under the new one, also vanishes at $\pistar_{\alpha=1}$. In unregularized RL, by contrast, the advantages of suboptimal actions at the optimal policy stay bounded away from zero, so the analogous constant, e.g.\ $\max_{s,a}|A^{\pi}(s,a)|$ in~\citet{schulman2015trust}, does not vanish and the error term remains comparably significant throughout training. What differs between our analysis and that of \citet{pmlr-v235-niu24c} is the treatment of the KL term of the performance-difference identity~\eqref{eq:soft_pdl}. They bound it by the trust-region radius $\zeta_F$. Here it is part of the surrogate $\tilde\cL_{\text{soft}}$ and is optimized exactly, so the only remaining error term is the state-distribution mismatch $M(\theta)\bar\delta(\theta)$ (see also \cref{app:policy_dependent}).

\paragraph{Scope of the guarantee.}
We view \eqref{eq:improvement_lower_bound} as a CPI-style motivation for the algorithmic structure of Ent-PPO rather than a convergence theorem: it is a per-step bound, the constants $M(\theta)$ and $\bar\delta(\theta)$ depend on $\pi_{\text{old}}$, and the inner-loop optimizer only approximately maximizes the surrogate.

\newpage

\newpage
\section{Stationary versus policy-dependent rewards}\label{app:policy_dependent}

Throughout this section the backward policy $\PB$ is fixed, $\alpha = \gamma = 1$, and we use the notation of \cref{sec:methodology,app:entppo}. We compare two ways of writing the same sampling problem as an RL problem. The \emph{stationary} formulation used in this paper~\citep{tiapkin2024generative,deleu2024discrete} combines the policy-independent reward $r(s,s')$ of~\eqref{eq:sampling_reward} with an explicit entropy bonus. The formulation of~\citet{pmlr-v235-niu24c} absorbs the entropy into the \emph{policy-dependent} reward
\begin{equation}\label{eq:policy_dependent_reward}
    r_\pi(s,s') \triangleq r(s,s') - \log\pi(s'\mid s)
\end{equation}
and uses no explicit regulariser; their Definition~3.2 states $-r_\pi$ as a cost, up to the additive constant $\log\rmZ$ in the terminal transition.

\paragraph{What does not depend on the choice.}
By the sample-path form of~\eqref{eq:value_def} and by~\eqref{eq:advantage} at $\alpha=\gamma=1$, for every policy $\pi$,
\begin{equation}\label{eq:pd_value_advantage}
    \tilde V^\pi(s) = \E_{\pi}\Big[\sum_{t=0}^{T-1} r_\pi(s_t,s_{t+1}) \,\Big|\, s_0 = s\Big],
    \qquad
    \tilde A^\pi(s,s') = r_\pi(s,s') + \tilde V^\pi(s') - \tilde V^\pi(s),
\end{equation}
that is, $\tilde V^\pi$ and $\tilde A^\pi$ are the \emph{standard}, unregularized value and advantage of $\pi$ in the MDP with reward $r_\pi$, and they coincide, up to sign and the constant $\log\rmZ$, with $V_F$ and $A_F$ of~\citet{pmlr-v235-niu24c}. The two formulations therefore share (i)~the objective $\tilde J$ and its optimum $\pistar_{\alpha=1}$; (ii)~the on-policy gradient~\eqref{eq:soft_pg}, and with it every algorithm that only differentiates at the data-collecting policy, including VPG with any estimator of \cref{sec:variance_reduction} and the expected on-policy TB gradient~\citep{malkin2023gflownets}; and (iii)~policy evaluation: for a fixed $\pi$ the reward $r_\pi$ is fixed, so fitting $\tilde V^\pi$ by Monte-Carlo regression, GAE, or Sub-EB~\citep{niu2026evaluating} is the same problem in both views. The non-stationarity induced by learning $\PB$~\citep{gritsaev2025optimizing} affects both formulations alike and is orthogonal to this discussion. In particular, nothing below concerns DB, TB, SubTB, Sub-EB, or soft actor-critic methods.

\paragraph{Where the choice matters: multi-epoch surrogates.}
The two formulations diverge only for updates that evaluate the objective at a policy $\pi_\theta$ other than the policy $\pi_{\text{old}}$ that collected the data. The derivations of CPI, TRPO and PPO~\citep{kakade2002approximately,schulman2015trust,schulman2017proximal} hold the reward fixed while the policy moves, which, as~\citet{pmlr-v235-niu24c} also point out, is not the case for a policy-dependent reward. The size of the effect is given by \cref{lem:soft_pdl} (in the GFlowNet setting, the same identity appears in the proof of~\citet[Lemma~C.2]{pmlr-v235-niu24c}). Indeed, by~\eqref{eq:soft_pdl_intermediate}, the KL term in~\eqref{eq:soft_pdl} is exactly the expected drift of the policy-dependent reward,
\begin{equation}\label{eq:reward_drift}
    \E_{s'\sim\pi(\cdot\mid s)}\big[r_{\pi}(s,s') - r_{\pi_{\text{old}}}(s,s')\big] = -\KL\big(\pi(\cdot\mid s)\,\|\,\pi_{\text{old}}(\cdot\mid s)\big).
\end{equation}
An off-policy surrogate built from data of $\pi_{\text{old}}$ must therefore decide at which policy $r_\pi$ is evaluated. Freezing it at $\pi_{\text{old}}$ is the default of standard PPO practice, where advantages are computed once per rollout and reused for all $K$ epochs; re-evaluating it at $\pi_\theta$ is the alternative:
\begin{align}
    \cL_{\text{frozen}}(\theta;\pi_{\text{old}}) &\triangleq \E_{\tau\sim\pi_{\text{old}}}\Big[\sum_{t=0}^{T-1}\rho_t(\theta)\big(r_{\pi_{\text{old}}}(s_t,s_{t+1}) + \tilde V^{\pi_{\text{old}}}(s_{t+1}) - \tilde V^{\pi_{\text{old}}}(s_t)\big)\Big], \label{eq:L_frozen}\\
    \cL_{\text{reeval}}(\theta;\pi_{\text{old}}) &\triangleq \E_{\tau\sim\pi_{\text{old}}}\Big[\sum_{t=0}^{T-1}\rho_t(\theta)\big(r_{\pi_\theta}(s_t,s_{t+1}) + \tilde V^{\pi_{\text{old}}}(s_{t+1}) - \tilde V^{\pi_{\text{old}}}(s_t)\big)\Big]. \label{eq:L_reeval}
\end{align}

\begin{proposition}[Frozen and re-evaluated surrogates]\label{prop:frozen_vs_reeval}
For every $\theta$,
\begin{equation*}
    \cL_{\text{reeval}}(\theta;\pi_{\text{old}}) = \tilde\cL_{\text{soft}}(\theta;\pi_{\text{old}}),
    \qquad
    \cL_{\text{frozen}}(\theta;\pi_{\text{old}}) = \tilde\cL_{\text{soft}}(\theta;\pi_{\text{old}}) + \sum_{t=0}^{T-1}\E_{s_t\sim d^{\pi_{\text{old}}}_t}\big[\KL\big(\pi_\theta(\cdot\mid s_t)\,\|\,\pi_{\text{old}}(\cdot\mid s_t)\big)\big].
\end{equation*}
In particular, $\cL_{\text{frozen}}(\theta;\pi_{\text{old}}) = \E_{\tau\sim\pi_{\text{old}}}[\sum_t\rho_t(\theta)\,\tilde A^{\pi_{\text{old}}}_t]$ is the surrogate of naive PPO, $\cL_{\text{reeval}}$ coincides with the unclipped Ent-PPO surrogate~\eqref{eq:ent_ppo_unclipped} with the KL estimated by importance sampling instead of computed analytically, and the two surrogates have the same value and gradient at $\theta = \theta_{\text{old}}$.
\end{proposition}

\begin{proof}
Conditionally on $s_t = s$, importance weighting gives $\E_{s'\sim\pi_{\text{old}}}[\rho\, f(s')] = \E_{s'\sim\pi_\theta}[f(s')]$ for any $f$. With $f = r_{\pi_{\text{old}}}(s,\cdot) + \tilde V^{\pi_{\text{old}}}(\cdot) - \tilde V^{\pi_{\text{old}}}(s) = \tilde A^{\pi_{\text{old}}}(s,\cdot)$ by~\eqref{eq:pd_value_advantage}, the per-state integrand of~\eqref{eq:L_frozen} is $\E_{s'\sim\pi_\theta}[\tilde A^{\pi_{\text{old}}}(s,s')]$. Replacing $r_{\pi_{\text{old}}}$ by $r_{\pi_\theta}$ adds $-\KL(\pi_\theta(\cdot\mid s)\,\|\,\pi_{\text{old}}(\cdot\mid s))$ by~\eqref{eq:reward_drift}, so the per-state integrand of~\eqref{eq:L_reeval} is $\phi(\pi_\theta;s)$ from~\eqref{eq:phi_def}. Summing over $t$ under $d^{\pi_{\text{old}}}_t$ and applying \cref{lem:surrogate_eq} gives both identities. Finally, the KL term and its gradient vanish at $\theta = \theta_{\text{old}}$.
\end{proof}

Thus Ent-PPO can be expressed in either formulation; the methodological difference is which surrogate is the default. In the stationary formulation the reward is fixed by construction and the entropy is by definition evaluated at the current policy, so the derivation of \cref{sec:ent_ppo} produces $\tilde\cL_{\text{soft}}$ and nothing else. In the policy-dependent formulation, obtaining it requires deliberately departing from the standard recipe of reusing rollout advantages. The naive PPO baseline of \cref{fig:EntPPO_vs_PPO} corresponds to $\cL_{\text{frozen}}$ with clipping, i.e.\ standard PPO applied to the reward of~\citet{pmlr-v235-niu24c}.
% TODO(verify): naive PPO in fig:EntPPO_vs_PPO uses g_t = r - log pi_old (computed once per batch) and no KL term.
The omitted term is second order in $\pi_\theta - \pi_{\text{old}}$. It is therefore harmless for single-step methods and for TRPO with a small radius, where~\citet[Thm.~3.6]{pmlr-v235-niu24c} bound it by the trust-region radius instead of optimising it, but it matters increasingly as the number of epochs $K$ grows.

\paragraph{Consequence 1: the improvement target and the weight of the KL.}
Both surrogates belong to the family obtained by penalising $\cL_{\text{frozen}}$ with $\beta\,\KL(\pi_\theta\,\|\,\pi_{\text{old}})$: $\beta = 0$ gives $\cL_{\text{frozen}}$ and $\beta = 1$ gives $\tilde\cL_{\text{soft}}$. The per-state maximisers of this family are explicit.

\begin{proposition}[Per-state maximisers of KL-penalised surrogates]\label{prop:beta_family}
Fix $s\in\cS$, assume $\pi_{\text{old}}(\cdot\mid s)$ has full support on $\cA_s$, and for $\beta \ge 0$ and $q \in \Delta(\cA_s)$ let
\begin{equation*}
    \psi_\beta(q;s) \triangleq \E_{s'\sim q}\big[\tilde A^{\pi_{\text{old}}}(s,s')\big] - \beta\,\KL\big(q\,\|\,\pi_{\text{old}}(\cdot\mid s)\big),
\end{equation*}
so that $\psi_0(\pi_\theta;s)$ and $\psi_1(\pi_\theta;s) = \phi(\pi_\theta;s)$ are the per-state integrands of $\cL_{\text{frozen}}$ and $\tilde\cL_{\text{soft}}$. For $\beta > 0$, the unique maximiser of $\psi_\beta(\cdot\,;s)$ is
\begin{equation}\label{eq:beta_maximiser}
    \pi_\beta(s'\mid s) \;\propto\; \pi_{\text{old}}(s'\mid s)\exp\!\big(\tilde A^{\pi_{\text{old}}}(s,s')/\beta\big)
    \;\propto\; \pi_{\text{old}}(s'\mid s)^{1-1/\beta}\,\pi^\sharp_{\text{old}}(s'\mid s)^{1/\beta},
\end{equation}
with $\pi^\sharp_{\text{old}}$ from~\eqref{eq:soft_greedy_def}. For $\beta = 0$, the maximisers are exactly the distributions supported on $\arg\max_{s'}\tilde A^{\pi_{\text{old}}}(s,s')$, i.e.\ on the greedy actions of the frozen reward $r_{\pi_{\text{old}}}$, and $\pi_\beta$ concentrates on this set as $\beta \to 0$.
\end{proposition}

\begin{proof}
Write $\tilde A^{\pi_{\text{old}}} = \tilde Q^{\pi_{\text{old}}} - \tilde V^{\pi_{\text{old}}} - \log\pi_{\text{old}}$ and expand the KL:
\begin{equation*}
    \psi_\beta(q;s) = \E_{s'\sim q}\big[\tilde Q^{\pi_{\text{old}}}(s,s') + (\beta-1)\log\pi_{\text{old}}(s'\mid s)\big] + \beta\,\cH(q) - \tilde V^{\pi_{\text{old}}}(s).
\end{equation*}
For $\beta>0$, the Gibbs variational principle gives the unique maximiser $q \propto \exp\big(\tilde Q^{\pi_{\text{old}}}(s,\cdot)/\beta + (1 - 1/\beta)\log\pi_{\text{old}}(\cdot\mid s)\big)$, which is the second form in~\eqref{eq:beta_maximiser} because $\pi^\sharp_{\text{old}}(\cdot\mid s) \propto \exp\tilde Q^{\pi_{\text{old}}}(s,\cdot)$. Since $\tilde Q^{\pi_{\text{old}}}(s,\cdot) - \log\pi_{\text{old}}(\cdot\mid s)$ and $\tilde A^{\pi_{\text{old}}}(s,\cdot)$ differ by the constant $\tilde V^{\pi_{\text{old}}}(s)$, this equals the first form, which concentrates on $\arg\max_{s'}\tilde A^{\pi_{\text{old}}}(s,s')$ as $\beta\to 0$. For $\beta = 0$, $\psi_0(\cdot\,;s)$ is linear in $q$ and is maximised exactly on the face of $\Delta(\cA_s)$ spanned by the maximising actions.
\end{proof}

At $\beta = 1$, \eqref{eq:beta_maximiser} is exact soft policy iteration, $\pi_1 = \pi^\sharp_{\text{old}}$, in agreement with \cref{prop:phi_decomposition}. By the soft policy improvement theorem (\cref{eq:soft_pi_state}), this step improves $\tilde V$ at every state without any constraint, and on a DAG with maximal trajectory length $T$ exact soft policy iteration reaches $\pistar_{\alpha=1}$ in at most $T$ iterations.\footnote{By induction on the length of the longest path from $s$ to $s_f$: once $\pi_k = \pistar_{\alpha=1}$ on all descendants of $s$, we have $\tilde Q^{\pi_k}(s,\cdot) = \tilde Q^{\pistar}(s,\cdot)$, hence $\pi_{k+1}(\cdot\mid s) = \pistar_{\alpha=1}(\cdot\mid s)$.} For $\beta > 1$, \eqref{eq:beta_maximiser} interpolates geometrically between $\pi_{\text{old}}$ and $\pi^\sharp_{\text{old}}$; this is the entropy-regularized natural policy gradient / policy mirror descent update~\citep{cen2022fast,zhan2023policy,lan2023policy}. In the discounted tabular setting it converges linearly, and under the identification $1/\beta = \eta\tau/(1-\gamma)$ with the step size $\eta$ and temperature $\tau$ of~\citet{cen2022fast}, the range $\beta\ge 1$ is exactly their admissible range $\eta \le (1-\gamma)/\tau$, with soft policy iteration at its boundary. For $\beta\in(0,1)$, \eqref{eq:beta_maximiser} extrapolates beyond $\pi^\sharp_{\text{old}}$ away from $\pi_{\text{old}}$. For $\beta \to 0$, i.e.\ for the frozen surrogate, it tends to the deterministic greedy policy of the frozen reward, a step that can decrease $\tilde J$, as the following example shows.

\paragraph{Example.}
Consider a one-step DAG: $\sinit$ has children $\cX$ with $|\cX| = N$, each $x\in\cX$ has the single child $s_f$, and $1\le\cR(x)\le 1+\kappa$. Then $r(\sinit,x) = \log\PB(\sinit\mid x) = 0$ and $\tilde J(\pi) = \E_{x\sim\pi(\cdot\mid\sinit)}[\log\cR(x)] + \cH(\pi(\cdot\mid\sinit))$. For uniform $\pi_{\text{old}}$ we have $\tilde J(\pi_{\text{old}}) \ge \log N$ and $\tilde A^{\pi_{\text{old}}}(\sinit,x) = \log\cR(x) + \mathrm{const}$. The maximiser of $\cL_{\text{frozen}}$ is the deterministic policy $\delta_{x^\star}$, $x^\star\in\arg\max_x\cR(x)$, with $\tilde J(\delta_{x^\star}) = \log\cR(x^\star) \le \log(1+\kappa)$: for small $\kappa$, the unconstrained frozen step loses about $\log N$. The maximiser of $\tilde\cL_{\text{soft}}$ is $\pi^\sharp_{\text{old}}(x\mid\sinit) \propto \exp\tilde Q^{\pi_{\text{old}}}(\sinit,x) = \cR(x)$, the target distribution, reached in a single step with $\tilde J(\pi^\sharp_{\text{old}}) = \log\rmZ$.

The stationary formulation thus singles out $\beta = 1$, with the KL in the direction $\KL(\pi_\theta\,\|\,\pi_{\text{old}})$, as the largest step covered by the theory above. The policy-dependent formulation provides no such identification. In the TRPO of~\citet{pmlr-v235-niu24c}, the KL enters only as a constraint, in the direction $\KL(\pi_{\text{old}}\,\|\,\pi_\theta)$ and with a tuned radius $\zeta_F$. The surrogate it constrains, $\cL_{\text{frozen}}$, is linear in the new policy, so it is the constraint alone that keeps the step away from a vertex of the simplex. We stress that the unit-weight KL in Ent-PPO is not a proximity term: its maximiser $\pi^\sharp_{\text{old}}$ may be far from $\pi_{\text{old}}$, which is why ratio clipping remains necessary (\cref{prop:surrogate_error,fig:EntPPO_no_clip}).

\paragraph{Consequence 2: behaviour near the optimum.}
By \cref{lem:soft_fixed_point} and~\eqref{eq:soft_greedy_def}, $\log\pistar_{\alpha=1}(s'\mid s) = \tilde Q^{\pistar}(s,s') - \tilde V^{\pistar}(s)$, so the soft advantage of the optimal policy vanishes identically, $\tilde A^{\pistar} \equiv 0$. This is a property of the soft advantage and holds in both formulations. For instance, the constant $\epsilon_F$ in~\citet[Thm.~3.6]{pmlr-v235-niu24c} also vanishes at $\pistar_{\alpha=1}$, as does $M(\theta)$ in \cref{prop:surrogate_error} when, in addition, $\pi_\theta \to \pi_{\text{old}}$.

The two surrogates nevertheless behave differently. At $\pi_{\text{old}} = \pistar_{\alpha=1}$, \cref{prop:frozen_vs_reeval} gives $\cL_{\text{frozen}}(\theta;\pistar_{\alpha=1}) = 0$ for every $\theta$: the expected frozen surrogate is flat and exerts no restoring force, so over $K$ epochs the policy fits the estimation noise in $\hat A_t$ and drifts up to the clipping range in every iteration. In contrast, $\tilde\cL_{\text{soft}}(\theta;\pistar_{\alpha=1}) = -\sum_t\E_{d^{\pistar}_t}[\KL(\pi_\theta\,\|\,\pistar_{\alpha=1})]$ is strictly concave in $\pi_\theta$ with maximiser $\pi_{\text{old}}$, which makes the optimum a stable fixed point of the inner loop. Since the omitted term is second order in $\pi_\theta - \pi_{\text{old}}$, this predicts a small gap between Ent-PPO and naive PPO for small $K$ that grows with $K$ and is most visible late in training, the pattern observed in \cref{fig:EntPPO_vs_PPO}.

\paragraph{Transfer of theory.}
More generally, the theory of regularized MDPs is formulated for a fixed reward plus a policy regulariser. This includes soft Bellman operators, soft policy improvement, regularized performance-difference lemmas, and the convergence of regularized natural policy gradient, policy mirror descent and adaptive TRPO~\citep{neu2017unified,geist2019theory,shani2020adaptive,cen2022fast,zhan2023policy,lan2023policy}. Under the stationary formulation these results apply to forward-policy training with fixed $\PB$ without modification. With a policy-dependent reward each of them has to be re-derived, and the re-derivation amounts to moving $-\log\pi$ back from the reward into a regulariser. Whenever a policy-dependent reward decomposes into a stationary reward plus a policy-only regulariser, as $r_\pi$ in~\eqref{eq:policy_dependent_reward} does, we therefore find the latter representation preferable. It is also the representation in which soft actor-critic~\citep{haarnoja2018soft} is formulated, and the one used by~\citet[App.~B.5]{pmlr-v235-niu24c} to relate their method to soft Q-learning.

\paragraph{Scope.}
We do not claim that the formulation of~\citet{pmlr-v235-niu24c} is incorrect or that their methods are affected: single-step methods coincide in both formulations, and their TRPO keeps $\cL_{\text{frozen}}$ valid through its trust region. Our claim is narrower: for multi-epoch updates, the stationary formulation
(a)~makes the correct surrogate $\tilde\cL_{\text{soft}}$ the default (\cref{prop:frozen_vs_reeval}),
(b)~fixes the weight and direction of the KL term (\cref{prop:beta_family}), and
(c)~exposes the soft-greedy target $\pi^\sharp_{\text{old}}$ together with the associated theory of regularized MDPs, which is used throughout \cref{app:entppo}.

\newpage
\section{Hyperparameters} \label{app:hyperparameters}

All hyperparameters are summarised in~\cref{tab:hyperparams}.

\subsection{Small problems}

In this section, all methods are implemented in the gfnx library~\citep{tiapkin2025gfnx} and trained
with the Adam optimizer  with default momentum parameters.
The backward policy $\PB$ is fixed to be uniform over parents in experiments where it is not learned. All experiments are conducted on CPU.

\paragraph{Batch size and training budget.}
All methods use a batch size of 16 trajectories per iteration. Hypergrid is trained for 31\,250 iterations ($5 \times 10^5$ total trajectories);
TFBind8 for 200\,000 iterations ($3.2 \times 10^6$ total trajectories);
String QM9 for 50\,000 iterations ($8 \times 10^5$ total trajectories).
All experiments are repeated over 3 random seeds.

\paragraph{Network architectures.}
All policy and value networks share the same hidden size of 256 and depth of 2.
In all cases, the value baseline network mirrors the policy encoder architecture with a scalar output head; the two networks do not share parameters.

\paragraph{Policy learning rate.}
We tune the policy learning rate on the off-policy baselines (DB, TB, SubTB) over the grid $\{10^{-3},\, 3\times10^{-4},\, 10^{-4}\}$ using 3 seeds, selecting the value that minimizes AUC. This yields $10^{-3}$ for Hypergrid and $3\times10^{-4}$ for TFBind8 and String QM9. All methods --- VPG variants, Ent-PPO, DB, TB, and SubTB---use these same policy learning rates on the respective environments.

\paragraph{TB.} $\log Z_\theta$ is trained with a separate Adam optimizer with learning rate $10^{-1}$.

\paragraph{SubTB.} We set SubTB-$\lambda$ to $0.9$ following~\cite{madan2023learning}.

\paragraph{Ent-PPO.} The PPO clipping parameter is set to $\varepsilon = 0.2$.

\paragraph{Value baseline.}
For VPG (Value Baseline), we tune the value network learning rate from $\{\text{lr}_\pi,\, \text{lr}_\pi / 3\}$, the number of value-network training epochs per rollout batch from $\{1, 2\}$, and the number of mini-batch splits within each epoch from $\{1, 2\}$ (see \cref{sec:baseline} for the definition of these quantities), using 3 seeds on each environment.
The selected configuration is: value learning rate $= \text{lr}_\pi$, 2 epochs, 1 split.

\paragraph{GAE.}
Based on the grid search described in \cref{sec:exp_gae}, we tune the GAE parameter $\lambda \in \{0.5, 0.6, 0.7, 0.8, 0.9\}$, value-network training epochs $E \in \{1, 2, 4, 8\}$, and mini-batch splits $S \in \{1, 2, 4, 8\}$ for VPG on Hypergrid.
The selected configuration is $(\lambda, E, S) = (0.7, 4, 2)$ with value learning rate $\text{lr}_\pi / 3$. This configuration is applied for Ent-PPO and TRPO across the small problems.

\subsection{Ising model}

We implement Ising experiments in gfnx~\citep{tiapkin2025gfnx} and run each seed on a single GPU. For each inverse temperature $\beta\in\{0.4407,0.6\}$, we train for 50\,000 iterations with batches of 256 trajectories and three random seeds. The backward policy is fixed to be uniform over parents. Following~\citep{carter2026discrete}, we anneal the inverse temperature over the first 40\,000 iterations and use the target temperature for the remaining training iterations. This allows to better explore and avoid mode collapse. Policy networks use three Transformer layers with four attention heads and embedding size 64. We use Adam at learning rate $3\times10^{-4}$, and global gradient-norm clipping at 1. Final metrics are averaged over five evaluation rounds of 16\,384 samples.

\subsection{Molecular-graph generation}

For molecular-graph experiments, all methods are implemented in the gflownet-recursionpharma library\footnote{ \href{https://github.com/recursionpharma/gflownet}{\texttt{https://github.com/recursionpharma/gflownet}}} and trained with the Adam optimizer using default momentum parameters. The backward policy $\PB$ is fixed to be uniform over parents in all experiments. All experiments are conducted on NVIDIA GPU A100 with 40GB memory. Each individual run completed in under 48 hours of wall-clock time. For stable metrics evaluation, we instantiate a target network with Exponential Moving Average (EMA) and $\tau = 0.05$, which is used for ELBO and EUBO evaluation. For off-policy baselines, we compared epsilon-greedy exploration with on-policy sampling and found that it slightly decreased performance. Hence, off-policy baselines sample training trajectories using the current policy. For both sEH and QM9, we use a fixed inverse temperature $\beta=16$, so the target distribution is proportional to $R(x)^{\beta}$, where $R(x)$ is the untempered reward.

\paragraph{Batch size and training budget.}
sEH uses a batch size of $256$ trajectories per iteration and QM9 uses $128$. All experiments are run for 1,000 iterations. We use 3 random seeds for sEH and 5 random seeds for QM9.

\paragraph{Network architectures.}
All policy and value networks share the same hidden size of $128$ and depth of $4$. The value baseline network mirrors the policy encoder architecture with a scalar output head; the two networks do not share parameters.

\paragraph{Policy learning rate.}
All experiments use the same standard learning rate $10^{-4}$.

\paragraph{TB.} $\log Z_\theta$ is trained with a separate Adam optimizer at learning rate $10^{-3}$.

\paragraph{SubTB.} SubTB-$\lambda = 1$.

\paragraph{Ent-PPO.} Interestingly, hyperparameters found on small problems transfer well to the larger molecular-graph environments and require little additional tuning. The PPO clipping parameter is set to $\varepsilon = 0.2$. The selected configuration is $(\lambda, E, S) = (0.7, 4, 8)$ with value learning rate $\text{lr}_\pi / 3$, applied to both sEH and QM9. The values of $\lambda$ and $E$ match those selected on the small problems; only the number of mini-batch splits $S$ increases (from $2$ to $8$), reflecting the larger per-iteration batch size. We verified this configuration by varying one hyperparameter at a time as in the previous section and observed no further improvement.

\begin{table}[ht]
\centering
\caption{\textbf{Hyperparameters used in our experiments}.}
\label{tab:hyperparams}
\footnotesize
\setlength{\tabcolsep}{3pt}
\begin{tabular}{lcccccc}
\toprule
 & \multicolumn{3}{c}{\textbf{Small problems}} & & \multicolumn{2}{c}{\textbf{Molecular graphs}} \\
\cmidrule(lr){2-4}\cmidrule(lr){6-7}
\textbf{Hyperparameter} & \textbf{Hypergrid} & \textbf{TFBind8} & \textbf{String QM9} & \textbf{Ising} & \textbf{sEH} & \textbf{QM9} \\
\midrule
\multicolumn{7}{l}{\textit{Shared settings (policy learning rate excludes the TRPO policy update)}} \\
Optimizer & \multicolumn{3}{c}{Adam} & Adam & \multicolumn{2}{c}{Adam} \\
Batch size (trajectories) & \multicolumn{3}{c}{16} & 256 & 256 & 128 \\
Hidden size & \multicolumn{3}{c}{256} & 64 & \multicolumn{2}{c}{128} \\
Network depth & \multicolumn{3}{c}{2} & 3 & \multicolumn{2}{c}{4} \\
Policy learning rate & $10^{-3}$ & $3\times10^{-4}$ & $3\times10^{-4}$ & $3\times10^{-4}$ & \multicolumn{2}{c}{$10^{-4}$} \\
\midrule
\multicolumn{7}{l}{\textit{TB}} \\
$\log Z$ learning rate & \multicolumn{3}{c}{$10^{-1}$} & $5\times10^{-2}$ & \multicolumn{2}{c}{$10^{-3}$} \\
\midrule
\multicolumn{7}{l}{\textit{SubTB}} \\
$\lambda$ & \multicolumn{3}{c}{0.9} & 0.9 & \multicolumn{2}{c}{1} \\
\midrule
\multicolumn{7}{l}{\textit{Ent-PPO only}} \\
Clip $\varepsilon$ & \multicolumn{6}{c}{0.2} \\
\midrule
\multicolumn{7}{l}{\textit{VPG (GAE) / Ent-PPO}} \\
Value learning rate & \multicolumn{6}{c}{$\text{lr}_\pi / 3$} \\
Value training epochs $E$ & \multicolumn{6}{c}{4} \\
GAE $\lambda$ & \multicolumn{3}{c}{0.7} & 0.95 & \multicolumn{2}{c}{0.7} \\
Mini-batch splits $S$ & \multicolumn{3}{c}{2} & 8 & \multicolumn{2}{c}{8} \\
\midrule
\multicolumn{7}{l}{\textit{VPG (Value Baseline)}} \\
Value learning rate & \multicolumn{3}{c}{$\text{lr}_\pi$} & --- & --- & --- \\
Value training epochs $E$ & \multicolumn{3}{c}{2} & --- & --- & --- \\
Mini-batch splits $S$ & \multicolumn{3}{c}{1} & --- & --- & --- \\
\bottomrule
\end{tabular}
\end{table}

\end{document}